\documentclass[12pt]{article}

\usepackage[english]{babel}
\usepackage[a4paper,top=3cm,bottom=3cm,left=3cm,right=3cm,marginparwidth=2cm]{geometry}
\usepackage{amsmath,amssymb,amsthm,amsfonts,dsfont}
\usepackage{mathrsfs,mathtools,thmtools}
\usepackage{graphicx,xcolor}
\usepackage{ragged2e,fancyhdr}
\usepackage{multirow}
\usepackage[style=authoryear,citestyle=authoryear,backend=bibtex8,bibstyle=authoryear,giveninits=true,maxbibnames=99,isbn=false,doi=false,url=false,hyperref=auto,eprint=true,natbib]{biblatex}
\usepackage{csquotes}
\usepackage{bm,bbm}
\usepackage{caption,subcaption}
\usepackage{algorithm,algpseudocode}
\usepackage{soul}

\usepackage{amsmath,amssymb,amsthm,amsfonts,dsfont}
\usepackage{mathrsfs,mathtools,thmtools}
\usepackage{bm}

\usepackage{xspace}
\usepackage{glossaries-extra}
\newcommand{\myacronym}[4] 
{
	\newglossaryentry{#1}
	{
		type=\acronymtype,
		name={#2},
		description={#3},
		text={#2},
		first={#3 (#2)},
		firstplural={#4 (#2s)},
		short={#2}
	}
	\expandafter\newcommand\csname #1\endcsname{\gls{#1}\xspace} 
	\expandafter\newcommand\csname #1s\endcsname{\glspl{#1}\xspace} 
}

\usepackage[dvipsnames]{xcolor}
\usepackage[colorinlistoftodos,textwidth=2cm]{todonotes}

\usepackage[colorlinks=true, allcolors=red]{hyperref}
\usepackage[capitalize, noabbrev]{cleveref} 

\usepackage{crossreftools}
\DeclareMathOperator{\Var}{Var}
\DeclareMathOperator{\Tr}{tr}
\DeclareMathOperator{\sign}{sign}
\renewcommand{\Pr}{\mathbb{P}}
\newcommand{\Ex}{\mathbb{E}}
\DeclareMathOperator*{\argmin}{arg\,min}

\DeclareMathOperator{\spn}{span}

\newcommand{\R}{\mathbb{R}}
\newcommand{\N}{\mathbb{N}}
\newcommand{\C}{\mathbb{C}}

\newcommand{\Z}{\mathcal{Z}}
\newcommand{\X}{\mathcal{X}}
\newcommand{\Y}{\mathcal{Y}}
\renewcommand{\H}{\mathbb{H}}
\newcommand{\GP}{\mathcal{G}}
\newcommand{\defeq}{\overset{\mathrm{def}}{=}}

\newcommand{\conv}{\underset{n \rightarrow \infty}{\longrightarrow}}
\newcommand{\dconv}{\overset{\mathrm{d}}{\conv}}
\newcommand{\pconv}{\overset{\Pr}{\conv}}

\renewcommand{\d}{\mathrm{d}}

\newcommand{\ppi}{{\bm\pi}}

\newcommand{\lz}{{\lambda_0}}
\renewcommand{\ln}{{\lambda_n}}
\newcommand{\lsn}{{\lambda_{s_n}^*}}

\newcommand{\hz}{{h_0}}
\newcommand{\hn}{{\hat h_n}}
\newcommand{\hsn}{{h^*_{s_n}}}
\newcommand{\hu}{{h_{s_n}^{*\mathrm{U}}}}
\newcommand{\hasn}{{h_{s_n}^{*\alpha}}}
\newcommand{\hp}{{\hat h_{k_n}}}

\renewcommand{\L}{{\ensuremath{L}}}
\newcommand{\Ln}{\L_n}
\newcommand{\Lsn}{\L^*_{s_n}}
\newcommand{\Lsna}{\L^{*\alpha}_{s_n}}
\newcommand{\risk}{{\ensuremath{\mathcal{R}}}}
\renewcommand{\Rn}{\risk_n}
\newcommand{\Rsn}{\risk^*_{s_n}}

\newcommand{\ds}[1]{{\mathcal{D}_{#1}}}

\newcommand{\dn}{\ds{n}}

\newcommand{\F}[2]{{\mathcal{F}_{#1,#2}}}
\newcommand{\Fn}{{\mathcal{F}_{n}}}
\newcommand{\Fnj}{\F{n}{j}}
\newcommand{\Fnjm}{\F{n}{j-1}}
\newcommand{\Fni}{\F{n}{i}}
\newcommand{\Fnim}{\F{n}{i-1}}
\newcommand{\tF}[2]{{\tilde{\mathcal{F}}_{#1,#2}}}
\newcommand{\tFnj}{\tF{n}{j}}
\newcommand{\tFnjm}{\tF{n}{j-1}}
\newcommand{\tFni}{\tF{n}{i}}
\newcommand{\tFnim}{\tF{n}{i-1}}

\DeclarePairedDelimiter{\pa}{(}{)}
\DeclarePairedDelimiter{\av}{|}{|}
\DeclarePairedDelimiter{\no}{\|}{\|}
\DeclarePairedDelimiter{\cb}{\{}{\}}
\DeclarePairedDelimiter{\br}{[}{]}
\newcommand{\E}[1]{\Ex\br*{#1}}
\newcommand{\dprod}[2]{\langle#1,#2\rangle_\H}
\newcommand{\dprodE}[2]{\langle#1,#2\rangle_E}

\newcommand{\V}[1]{\Var\pa*{#1}}
\renewcommand{\P}[1]{\Pr\pa*{#1}}

\makeatletter
\@ifclassloaded{beamer}
{}
{
	\declaretheorem[name=Theorem,numberwithin=section]{theorem}
	\declaretheorem[name=Lemma,numberlike=theorem,refname={lemma,lemmas},Refname={Lemma,Lemmas}]{lemma}
	\declaretheorem[name=Proposition,numberlike=theorem,refname={proposition,propositions},Refname={Proposition,Propositions}]{proposition}
	\declaretheorem[name=Corollary,numberlike=theorem,refname={corollary,corollaries},Refname={Corollary,Corollaries}]{corollary}
	\declaretheorem[name=Assumption,refname={assumption,assumptions},Refname={Assumption,Assumptions}]{assumption}
	
	\declaretheorem[name=Remark,numbered=no,refname={remark,remarks},Refname={Remark,Remarks}]{remark} 
	\declaretheorem[name=Example,refname={example,examples},Refname={Example,Examples}]{example}
	\newtheorem{primeassumptioninner}{Assumption}
	\newenvironment{primeassumption}[1]{
		\renewcommand\theprimeassumptioninner{#1'}
		\primeassumptioninner
	}{\endprimeassumptioninner}

}
\makeatother

\title{
  Subsampling for supervised learning in reproducing kernel Hilbert spaces
  }
\date{\today}
\author{Eyal \textsc{Vayness} \& Maxime \textsc{Sangnier} \\ {\footnotesize Sorbonne Université, LPSM, Paris, France}}

\graphicspath{{./}}

\myacronym{iid}{i.i.d.}{independent and identically distributed}{-}
\myacronym{rkhs}{RKHS}{reproducing kernel Hilbert space}{reproducing kernel Hilbert spaces}
\myacronym{erm}{ERM}{empirical risk minimizer}{empirical risk minimizers}
\myacronym{srm}{SRM}{subsampled risk minimizer}{subsampled risk minimizers}
\myacronym{ssp}{SSP}{subsampling probability}{subsampling probabilities}
\myacronym{sspv}{SSPV}{subsampling probability vector}{subsampling probability vector}
\myacronym{rff}{RFF}{random Fourier features}{random Fourier features}

\begin{document}
	\justifying
	\maketitle
	\pagenumbering{arabic}


\begin{abstract}

In the era of big data, subsampling became a common practice in statistical learning.
By selecting a subgroup of individuals based on which the learner is trained, subsampling aims at reducing the computational cost and time of the estimation step, and ideally leads to a decrease of its energy consumption and carbon footprint.
This work focuses on a nonparametric setting, in which the hypotheses set lies in a reproducing kernel Hilbert space, and the estimator is a minimizer of an empirical risk reweighted à la Horvitz-Thompson.
By studying the asymptotic properties of this estimator, we reveal an optimal subsampling scheme (regarding the trace of the covariance operator) and show that it can be used via plug-in.
A numerical study on synthetic and real-world datasets shows the practicability and the benefit of the proposed approach.
\end{abstract}

\section{Introduction}

	As of the last couple decades, the overwhelming data accumulation in various domains such as medical research, road safety, bioinformatics, physics, genetics and social media, led scientists to rethink standard statistical learning methods.
	Beyond the computational barrier (in time and memory) induced by the data deluge and the economic cost of training models on large datasets, the overconsumption of energy resources also became a concerning topic for the scientific community.
	With this in mind, it is of interest to reduce the computational cost of learning methods, particularly regarding their training phase, which often involves minimizing an objective function and is more demanding than inference. 

	A first thought to limit the computational burden of training procedures, is to leverage more efficient optimization algorithms, such as active set strategies \citep{johnson_blitz_2015}, stochastic gradient methods \citep{bottou2018optimization,clemencon_optimal_2019} and random coordinate descents \citep{shalev-shwartz_stochastic_2013,fercoq_coordinatedescent_2019a}, to cite only a few.
	However, these approaches still work on the whole training set, which may include either redundant or quite uninformative observations.
	Data set reduction methods can then be used in order to reduce the effective size of the dataset, while producing only informative datapoints.
	Sketching, for instance, is a way to reduce the dataset by generating new individuals, which are linear combinations of the original ones \citep{pilanci_randomized_2015, raskutti_statistical_2016,wang_sketched_2018}. 
	Another approach is the coreset technique, which randomly pick points from de dataset in order to uniformly approximate the objective function \citep{huggins2016coresets,munteanu_coresets_2018,karnin_discrepancy_2019,samadian_unconditional_2020,braverman2021efficient}.
	Nonuniform subsampling of points can also be implemented thanks to their leverage scores \citep{mahoney_randomized_2011,drineas_faster_2011,woodruff_sketching_2014,ma_statistical_2015}, which are based on the covariates but not on the target values.
	Stepping back, subsampling can also be analyzed for general inclusion probabilities, from the excess risk point of view \citep{clemencon2016learning,han_complex_2021}
	and the estimator asymptotic distribution side \citep{wang_optimal_2018,wang2022sampling,han_leverage_2025a}.
	In the latter case, it is even possible to exhibit an optimal subsampling scheme, which minimizes the trace of the limit covariance matrix.
	Several other approaches exist, such as sampling with influence functions \citep{ting_optimal_2018} and leveraging minimum enclosing ball from computational geometry \citep{tsang2005core}.

	At this step, we remark that many methods, such as stochastic optimization and sketching, are task-agnostic, which leaves an area of improvement. 
	Moreover, if coreset techniques propose data-dependent sampling procedures, these ones are aimed at uniform approximation the objective function and consequently often leverage intractable probabilities \citep{munteanu_coresets_2018,braverman2021efficient}.
	On the contrary, optimal subsampling proposes an easy-to-implement way to extract informative individuals from the original dataset, such that the trace of the asymptotic covariance matrix is minimal \citep{wang2022sampling}.

	In this work, we focus on nonparametric supervised learning, and in particular on kernel methods, where the hypotheses lie in a \rkhs \citep{berlinet_reproducing_2004}.
	In this case, few of the methods described above can be used (they have been mainly designed for parametric inference), the most notable being sketching \citep{yang_randomized_2017,elahmad2023fast,elahmad2024sketch} and specific approaches tailored for support vector machines, such as minimum enclosing balls \citep{tsang2005core} and stochastic subgradient algorithm \citep{shalev-shwartz_pegasos_2011}.
	However, kernel learning can rely on kernel approximation techniques to lighten the numerical training cost,
	which consists in replacing either the Gram matrix by a low-rank approximation (for instance, thanks to the Nyström method \citep{williams2000using,drineas_nystrom_2005,rudi_less_2015}), or the feature map by a finite-dimensional mapping (thanks to \rff \citep{rahimi2007random,brault_random_2016,rudi_generalization_2017,li_unified_2021}).
	Since the computational cost of a kernel method is mainly function of the size of the optimization problem to solve (which is basically represented by the spatial complexity of the design or Gram matrix of the effective dataset),
	the appeal of the Nyström method, \rff and sketching is to cut down the computational burden from \(O(n^2)\) (for an initial dataset of size \(n\)) to \(O(nm)\), where \(m\) an approximation parameter (the dimension of the subspace for the Nyström method, the number of random projections for \rff, the number of linear combinations for the sketching approach).
	If this is a big gain, we remark that the computational cost still depends on the size \(n\) of the original dataset, which may be very large in practice.
	On the contrary, subsampling \(s\) datapoints would lead to a \(O(s^2)\) cost, which is largely less.

	We thus aim at designing a data-dependent subsampling procedure for nonparametric supervised learning in \rkhss.
	From the methodological side, we follow the line of research initiated by \citet{wang_optimal_2018} (then followed by numerous papers including \citep{wang_more_2019,wang2022sampling,han_leverage_2025a,shao_optimal_2025}), which determines the subsampling scheme minimizing the trace of the asymptotic covariance matrix of a M-estimator of an empirical risk reweighted à la Horvitz-Thompson.
	From the theoretical side, we study asymptotics in \rkhss with tools coming from functional random probability spaces \citep{christmann2008support,hsing2015theoretical} with pieces of elements from martingale theory \citep{hall1980martingale}.
	These two aspects distinguish our work from that by \citet{wang2022sampling} since we place ourself in a particular nonparametric setting and employ different techniques of proof.
	The latter are also different from the empirical process techniques used by \citet{hable2012asymptotic,sancetta2021estimation} for analyzing usual \erms.

	In \Cref{sec:background}, we first revisit asymptotic normality of the \erm in \rkhss, based on the asymptotically finite-dimensional property \citep{hsing2015theoretical,vandervaart2023weak}.
	This result corroborates the ones presented in \citep{hable2012asymptotic,sancetta2021estimation} and comes with a complementary technique of proof.
	Then, in \Cref{sec:asymp_erm}, the \srm is presented as well as its convergence and asymptotic distribution.
	\Cref{sec:optimal} exhibits the optimal subsampling scheme and \Cref{sec:piloted_estimator} presents a way to implement it in practice.
	Le latter strategy is numerically evaluated on synthetic and real-world datasets in \Cref{sec:numerical}, and favorably compared to state-of-the-art approximation techniques for kernel methods.\footnote{The Python source code for reproducing the numerical experiments is provided at \href{https://gitlab.com/Yaloude/rkhs-subsampling}{https://gitlab.com/Yaloude/rkhs-subsampling}.}
	All proofs are gathered in \Cref{sec:proofs}.


\section{Background} \label{sec:background}
  Let $Z \in \Z \subset \R^{d+1}$ be the random variable of interest with unknown distribution.
  We focus here on supervised learning and assume that \(Z\) has a canonical decomposition: $Z = (X, Y) \in \X \times \Y \subset \R^d \times \R$ where $X$ is the vector of covariates and $Y$ is the response.
  Denote $\ell : \R \times \Y \rightarrow \R$ the loss function associated to the learning task of interest,
  and assume that \(\ell\) is non-negative, convex and twice-differentiable in its first argument.
  The setting defined this way suits both binary classification: $\Y = \cb{\pm 1}$ and for instance $\ell(\hat y, y) = \log(1 + \exp(-y \hat y))$ for logistic regression (where $\hat y \in \R$ is a prediction compared to $y \in \Y$);
  as well as regression: $\Y = \R$ and for example $\ell(\hat y, y) = \frac{1}{p} |\hat y - y|^p$ for $L^p$ regression.

  The working set of predictors is $\H$, the \rkhs induced by some kernel $k : \X \times \X \rightarrow \R$ (i.e.\ a symmetric and positive definite function).
  In other words, $\H$ is the closure of the span of the map by $k$ of elements of $\X$, i.e.
  \begin{equation*}
	  \H \defeq \overline{\spn \cb{K_x : x \in \X}},
  \end{equation*}
  where $K_x = k(x, \cdot)$ for all $x \in \X$.

  Given the previous elements, the population risk function of the learning task considered is:
  \begin{equation} \label{eq:true_loss}
	  \L : h \in \H \mapsto \E{\ell(h(X), Y)},
  \end{equation}
  which could be minimized so as to obtain an optimal predictor.
  However two problems arise:
  first, the distribution of \(Z = (X, Y)\) is assumed unknown (such that the expectation in \eqref{eq:true_loss} cannot be computed);
  second, the current assumptions do not ensure existence of a minimizer of \(L\) in \(\H\).
  To solve those problems, we turn towards the empirical version of \eqref{eq:true_loss}:
  \begin{equation} \label{eq:empirical_loss}
	  \Ln(h) \defeq \frac{1}{n} \sum_{i = 1}^n \ell(h(X_i), Y_i),
  \end{equation}
  based on a dataset $\dn = \cb{ Z_1, \dots, Z_n } = \cb{ (X_1, Y_1), \dots, (X_n, Y_n) }$ made of \iid random variables sharing the same distribution as \(Z = (X, Y)\),
  and add a ridge penalization tuned by some parameter $\lz > 0$ for the true risk and by $\ln > 0$ (a possibly random parameter) for the empirical one.
  The objective functions, namely the true and empirical regularized risks now read:
  \begin{equation} \label{eq:true_risk}
	  \risk(h) \defeq \L(h) + \frac{\lz}{2} \| h \|_\H^2,
  \end{equation}
  and
  \begin{equation} \label{eq:empirical_risk}
	  \Rn(h) \defeq \Ln(h) + \frac{\ln}{2} \| h \|_\H^2.
  \end{equation}
  Let us remark that minimizing \(\risk\) (respectively \(\Rn\)) is equivalent to minimizing \(L\) (respectively \(L_n\)) inside a ball of \(\H\), the radius of which depending both on the penalization parameter \(\lz\) (respectively \(\ln\)) and the data (see for example \citep{sancetta2021estimation}).
  In addition, the framework above offers the chance to take \(\ln \neq \lz\), which makes sens in practice since penalizing is also a way to prevent overfitting, on the condition that the penalization parameter \(\ln\) is tuned with respect to the sample size \(n\) and the data (roughly speaking, the less the data, the stronger the penalization).

  Now, by convexity of the loss \(\ell\) (with respect to its first argument) and \(\lz > 0\), the true regularized risk \(\risk\) is lower semi-continuous and strongly convex, and admits a unique minimizer in \(\H\), denoted \(\hz\) \citep[][Lemma 2.15]{christmann2008support}.
  The same holds true for the empirical regularized risk \(\Rn\) (since \(\ln > 0\)), the minimizer of which is denoted \(\hn \in \H\).

  \subsection{Assumptions}
    Before introducing the proposed subsampling estimator, we aim at presenting the asymptotic distribution of \(\hn\) as a stepping stone to our main result.
    This requires a few assumptions, which will be needed all along our analysis.

    \begin{assumption} \label{hyp:kernel}
	    The kernel $k : \X \times \X \rightarrow \R$ is continuous and bounded: there exists $\kappa > 0$ such that $\sup_{x \in \X} \sqrt{k(x, x)} \leq \kappa$.
    \end{assumption}

    \Cref{hyp:kernel} is standard in the analysis of estimators lying in \rkhss (see for instance \citep{bartlett_rademacher_2002}).
    In particular, the continuity of \(k\) (jointly with $\X$ being separable) leads the \rkhs \(\H\) to be separable \citep[][Lemma 4.33]{christmann2008support}.
    This means that there exists a complete orthonormal system $(e_j)_{j \in \N}$ of $\H$ (i.e.\ $\| e_j \|_\H = 1$ and $\dprod{e_j}{e_k} = 0$ for all $j,k \in \N$, $k \neq j$),
    which will enable to define the trace of a linear operator in \(\H\) in \Cref{sec:asymp_erm}.
    The bound \(\kappa\), as for it, leads to bounded predictors (\(\sup_{\bm x \in \X}|h(\bm x)| < \infty\), for all \(h \in \H\)) \citep[Lemma 4.28]{christmann2008support},
    which helps in obtaining convergence of empirical quantities (see \Cref{sec:proofs}).

    Henceforth, let $\ell^\prime$ and $\ell^{\prime\prime}$ be respectively the first and second partial derivatives of $\ell$ with respect to its first variable.

    \begin{assumption} \label{hyp:loss}
	    For all $y \in \Y$ and $B > 0$, the loss $\ell$ and its derivatives, $\ell^\prime$ and $\ell^{\prime\prime}$, are $\phi_B(y)$-Lipschitz in their first variable on $[-B, B]$, i.e.\ for $k \in \{ 0, 1, 2 \}$:
	    \begin{equation*}
		    | \ell^{(k)}(\hat y_1, y) - \ell^{(k)}(\hat y_2, y) | \leq \phi_B(y) | \hat y_1 - \hat y_2 |,
	    \end{equation*}
	    for all $\hat y_1, \hat y_2 \in [-B, B]$.
    \end{assumption}

    \begin{assumption} \label{hyp:distribution}
	    For all $B > 0$ and $k \in \{ 0, 1, 2 \}$, $\E{ \phi_B(Y)^4 }$ and $\E{ \ell^{(k)}(0, Y)^4 }$ exists and are finite.
    \end{assumption}

    \Cref{hyp:loss} asks for $\ell$ and its derivatives to be locally Lipschitz continuous in their first argument.  
    For simplicity, we consider the same Lipschitz coefficient $\phi_B(\cdot)$ for all derivatives.
    In particular, if \(\ell\) is thrice-differentiable, then one may choose, for all \(y \in \Y\) and \(B > 0\), \(\phi_B(y) = \max_{k \in \{1, 2, 3\}} \sup_{t \in [-B, B]} \ell^{(k)}(t, y)\).
    \Cref{hyp:distribution} as for it, is a common moment assumption in asymptotic analyses, except that it is required up to the degree \(4\) in order to handle the Horvitz-Thompson reweighting introduced in \Cref{sec:asymp_erm}. 
    In addition, \Cref{hyp:loss,hyp:distribution} (with \Cref{hyp:kernel}) yields that $\ell^\prime(\hz(X), Y) K_X$ and $\ell^{\prime\prime}(\hz(X), Y) K_X$ have a finite fourth order moment.

    At last, \Cref{hyp:ln_speed} asks for the convergence of the regularization parameter $\ln$ towards $\lz$ at a rate greater than $1/\sqrt{n}$, which is basically that of \(L_n(h)\) towards \(L(h)\), for any \(h \in \H\).

    \begin{assumption} \label{hyp:ln_speed}
	    The (random) regularization parameter sequence $(\ln)_{n>0}$ verifies that $\ln > 0$ almost surely for all $n > 0$ and
	    \begin{equation*}
		    \sqrt{n} (\ln - \lz) \pconv 0.
	    \end{equation*}
    \end{assumption}

  \subsection{Asymptotic normality of \erm}
    With the assumptions stated above, our first step is to derive gradients of the regularized risks.
    Indeed, since $\ell \in \mathscr{C}^2$, the empirical regularized risk $\Rn$ is twice-differentiable
    and its minimizer $\hn$ verifies the following stationarity condition:
    \begin{equation}
	    \nabla \Rn(\hn) = 0 \quad \Leftrightarrow \quad \ln \hn = - \nabla \Ln(\hn),
    \end{equation}
    with
    \begin{equation}
	    \nabla \Ln(h) = \frac{1}{n} \sum_{i = 1}^n \ell^\prime (h(X_i), Y_i) K_{X_i}, 
    \end{equation}
    for all $h \in \H$.
    In addition, letting ``$\otimes$'' being the outer product in \(\H\)
    (i.e.\ for all $h, f \in \H$,
    \((h \otimes f)\) is a linear operator from \(\H\) to \(\H\) such that,
    for all \(g \in \H\), $(h \otimes f)g \defeq \dprod{f}{g}h$)
    one may also compute the Hessian operators of the empirical risks:
    \begin{equation}
	    \nabla^2\Ln(h) = \frac{1}{n} \sum_{i=1}^n \ell^{\prime\prime}(h(X_i), Y_i) K_{X_i} \otimes K_{X_i}
	    \quad \text{and} \quad
	    \nabla^2\Rn(h) = \nabla^2\Ln(h) + \ln \mathrm{id}_\H, 
    \end{equation}
    where \(\mathrm{id}_\H\) is the identity operator in \(\H\).

    The same applies to the true regularized risk \(\risk\):
    according to \Cref{lem:leibniz}, under \Cref{hyp:kernel,hyp:loss,hyp:distribution}, \(\risk\) is twice differentiable at \(\hz\), with:
    \[
      \nabla\risk(\hz) = \nabla\L(\hz) + \lz \hz
      \quad \text{and} \quad
      \nabla^2\risk(\hz) = \nabla^2\L(\hz) + \lz \mathrm{id}_\H,
    \]
    where
    \[
      \nabla \L(\hz) = \E{ \ell^\prime (\hz(X), Y) K_X } 
      \quad \text{and} \quad
      \nabla^2\L(\hz) = \E{ \ell^{\prime\prime}(\hz(X), Y) K_X \otimes K_X }. 
    \]
    Let us remark that \(\nabla^2 \risk\) is continuous at \(\hz\) and that \(\nabla^2 \risk(\hz)\) is invertible \citep[Lemma A.5]{hable2012asymptotic}.

    We are now ready to state asymptotic normality of \(\hn\).
    For this purpose, let us recall that we say that \(H\) is a Gaussian process in \(\H\) with mean \(0\) and covariance operator \(\Sigma : \H \to \H\), denoted \(H \sim \GP(0, \Sigma)\),
    if for any positive integer \(m\) and \(h_1, \dots, h_m \in \H\), \((\dprod{H}{h_1}, \dots, \dprod{H}{h_m})^\top \sim \mathcal N(0, (\dprod{\Sigma h_i}{h_j})_{1 \le i, j \le m})\).

    \begin{theorem}[Asymptotic normality of $\hn$] \label{thm:erm_normality}
	    Under \Cref{hyp:kernel,hyp:loss,hyp:distribution,hyp:ln_speed}, one has
	    \begin{equation*}
		    \sqrt{n} \nabla^2\risk(\hz) (\hn - \hz) \dconv \GP(0, \Sigma),
	    \end{equation*}
	    with $\Sigma \defeq \E{ \ell^\prime(\hz(X), Y)^2 K_X \otimes K_X } - \nabla\L(\hz) \otimes \nabla\L(\hz)$.
    \end{theorem}

    The proof of \Cref{thm:erm_normality} (see \Cref{app:proof_erm_normality}) is close to that in a parametric model, thanks to a Cramér-Wold-type result.
    The latter holds for, given assumptions on the considered \rkhs \(\H\), random sequences are asymptotically finite-dimensional in the sense of \citep[][Chapter 1.8]{vandervaart2023weak} or \citep[][Theorem 7.7.5]{hsing2015theoretical}.

    Asymptotic properties of the \erm in the non-parametric framework of \rkhss are known.
    Main results include consistency \citep{christmann2008support},
    outlining of a limit Gaussian process \citep{hable2012asymptotic},
    and asymptotic normality in a ball of \(\H\) \citep{sancetta2021estimation}.
    \Cref{thm:erm_normality} is similar to a result by \citet{hable2012asymptotic,hable2012asymptotica},
    but the proofs differ since \citet{hable2012asymptotic,hable2012asymptotica} make use of a functional delta method on the empirical process.
    \Cref{thm:erm_normality} is nonetheless of interest for it and its proof serve as stepping stones to our main result.



\section{Asymptotics of subsampling} \label{sec:asymp_erm}
  We now introduce, in this section, the two considered subsampling strategies (with and without replacement),
  along with the corresponding estimators, nicknamed \srm in both cases.
  We then prove asymptotic normality of both estimators with general subsampling schemes,
  and discuss in the next section the choice of the latter.

  Let \(\Delta_{n-1} = \{\boldsymbol x \in (0, 1)^n, \sum_{i=1}^n x_i = 1\}\) be the set of probability vectors of size \(n\),
and consider \(\ppi_n \in \Delta_{n-1}\) a \ssp vector, which may ultimately depend on the dataset \(\dn\).
  A subsampling scheme of size \(s_n \le n\) consists in randomly picking points \(Z_i\) in the dataset \(\dn\) with inclusion probabilities \(\pi_{n, i}\).
  More formally, we consider, in the context of this paper, the two following subsampling strategies: one with replacement and another without.

  \paragraph{Subsampling with replacement}
  In subsampling \emph{with} replacement (also called Multinomial subsampling),
  one draws independently \(s_n\) point indices:
  \begin{equation*}
	  I_{n,1}, \dots, I_{n,s_n} | \dn \overset{\iid}{\sim} \sum_{i = 1}^n \pi_{n, i} \delta_{\{ i \}},
  \end{equation*}
  where \(\delta_{\{ i \}}\) is the Dirac measure at \(i\).
  Since the choice is with replacement, each point \(Z_i\) appears \(\gamma_{n, i} = \sum_{j = 1}^{s_n} \mathbbm{1}_{\{ I_{n,j} = i \}}\) times in the subsampled dataset,
  with \(\gamma_{n, i}\) being referred to as the multiplicity or the number of occurrences of \(Z_i\) (which may be null if \(Z_i\) is not picked).
  In this situation
  \begin{equation}\label{equ:gamma_multinomial}
    (\gamma_{n, 1}, \dots, \gamma_{n, n}) | \dn \sim \text{Mul}(s_n, \ppi_n),
  \end{equation}
  (Multinomial distribution with parameters \(s_n\) and \(\ppi_n\))
  and the subsample size (counting multiplicities) is equal to \(s_n\): \(\sum_{i=1}^n \gamma_{n, i} = s_n\) almost surely.

  \paragraph{Subsampling without replacement}
  In subsampling \emph{without} replacement (or Poisson subsampling),
  ones draws independently for each point \(Z_i\):
  \begin{equation}\label{equ:gamma_poisson}
    \gamma_{n, i} | \dn \overset{i.i.d.}{\sim} \text{Ber}(s_n \pi_{n, i}),
  \end{equation}
  (Bernoulli distribution with parameter \(s_n \pi_{n, i}\))
  and decides that it should enter the subsampled dataset when the multiplicity (or inclusion indicator) verifies \(\gamma_{n, i} = 1\).
  This is possible assuming that $\pi_{n, i} \leq s_n^{-1}$ for all $n > 0$ and $i \in \{ 1, \dots, n \}$,
  and in this situation, the subsample size is in average equal to \(s_n\): \(\E{ \sum_{i=1}^n \gamma_{n, i} | \dn} = s_n\).

  In both situation, the subsampling schemes can be summarized as incorporating the \(i\)-th observation \(Z_i\), \(\gamma_{n, i}\) times to the subsampled dataset, where \(\gamma_{n,i} \in \N\) is a random variable such that $\E{ \gamma_{n,i} \middle| \dn} = s_n \pi_{n,i}$.
  The subsamples are then the $Z_i$s such that $\gamma_{n,i} > 0$, for $i \in \{ 1, \dots, n \}$,
  and we define the subsampled and regularized risks as follows:
  \begin{gather}
	  L^*_{s_n}(h) \defeq \frac{1}{n} \sum_{i=1}^{n} \frac{\ell(h(X_i), Y_i)}{s_n \pi_{n,i}} \gamma_{n,i}, \label{eq:subsample_loss} \\
	  \Rsn(h) \defeq L^*_{s_n}(h) + \frac{\lsn}{2} \| h \|^2_\H, \label{eq:subsample_risk}
  \end{gather}
  where \(\lsn > 0\) is a (possibly random) regularization parameter.
  Let also \(\hsn\) be the minimizer of \(\Rsn\).
  By Fermat's rule, one has:
  \begin{equation}
  -\lsn \hsn = \nabla \Lsn(\hsn).
  \end{equation}

  The inverse probability weighting à la Horvitz-Thompson of \Cref{eq:subsample_loss} is a mean to produce unbiased estimations of the empirical and true risks (\Cref{eq:empirical_loss,eq:true_loss}):
   $\E{ \Lsn(h) \middle| \dn} = \Ln(h)$ and $\E{ \Lsn(h)} = \L(h)$.
   This is a critical point in order to get limit results, such as consistency and asymptotic normality.
   However, this is not enough.
   First, the regularization parameter \(\lsn\) has to converge to \(\lz\) (for consistency),
   and with a rate at least \(1/\sqrt{s_n}\), for asymptotic normality.
   This is formalized by \Cref{hyp:lsn_speed} below.

   \begin{primeassumption}{\ref{hyp:ln_speed}} \label{hyp:lsn_speed}
	  The (random) regularization parameter sequence $(\lambda^*_{s_n})_{n>0}$ verifies that $\lsn > 0$ almost surely for all $n > 0$ and
	  \begin{equation*}
		  \sqrt{s_n} (\lsn - \lz)	 \pconv 0.
	  \end{equation*}
  \end{primeassumption}

  Next, the \ssps \(\pi_{n, i}\) should not converge too quickly to \(0\), in order for the subsampled risk \eqref{eq:subsample_loss} not to explode.
  This is described in \Cref{hyp:ssp}, which states that \ssps  have to be asymptotically of the order \(1/n\).

  \begin{assumption} \label{hyp:ssp}
	  The probability vector sequence $(\ppi_n)_{n \in \N}$ satisfies
	  \begin{equation*}
		  \E{ \max_{1 \le i \le n} \cb*{\frac{1}{n \pi_{n,i}}}^2 } = O(1).
	  \end{equation*}
  \end{assumption}

\begin{remark}

	The intuition that \ssps \(\pi_{n, i}\) cannot converge to \(0\) faster than \(1/n\) is not surprising,
	and can also be read in previous works.
	As an example, \citet{han_leverage_2025a} leverage an assumption in expectation similar to \Cref{hyp:ssp} but with averaged inverse \ssps instead of their maximum.
	Remark that considering such type of assumption with our sketch of proof would require to go up to the order 8.

	Similarly, \citet{wang2022sampling} require $\max_{1 \le i \le n} \cb*{n \pi_{n,i}}^{-1} = O_\Pr(1)$, which can be obtained from \Cref{hyp:ssp} by Markov's inequality:
	for all $M > 0$,
	\begin{equation*}
		\P{ \max_{1 \le i \le n} \cb*{\frac{1}{n \pi_{n,i}}} > M } \leq \frac{1}{M^2} \E{ \max_{1 \le i \le n} \cb*{\frac{1}{n \pi_{n,i}}}^2 } = \frac{1}{M^2} O(1),
	\end{equation*}
	hence, $\max_{1 \le i \le n} \cb*{n \pi_{n,i}}^{-1} = O_\Pr(1)$ and more generally $\max_{1 \le i \le n} \cb*{n \pi_{n,i}}^{-p} = O_\Pr(1)$ for all positive integer $p$.
\end{remark}

The next result states \srm consistency (which can actually be obtained without a rate of convergence for \(\lsn\)) in both subsampling schemes, as \(n \to \infty\) and \(s_n \to \infty\).

\begin{theorem}[Consistency of $\hsn$] \label{thm:srm_consistency_core}
	Under \Cref{hyp:kernel,hyp:loss,hyp:distribution,hyp:ssp},
	and \(\lsn \pconv \lz\) (or \Cref{hyp:lsn_speed}),
	one has
	\begin{equation*}
		\| \hsn - \hz \|_\H \pconv 0,
	\end{equation*}
	whether $\hsn$ is the Multinomial (\Cref{equ:gamma_multinomial}) or the Poisson (\Cref{equ:gamma_poisson}) subsampling estimator.
\end{theorem}

We will show in the next theorems that the \srm \(\hsn\) is asymptotically normal to the target \(\hz\).
Obviously, it is expected that the rate of convergence of \(\hsn\) as well as its asymptotic covariance operator are different from that of the \erm \(\hn\).
In order to reveal the alteration of the asymptotic covariance operator, we consider inclusion probabilities, which depend on the observation of interest as \(\pi_{n, i} \propto \psi(Z_i)\), where \(\psi\) is a positive absolute interest (or importance) function.
More formally, let $\psi : \Z \rightarrow \R^*_+$ be a deterministic map independent from \(n\).
Then, for all $n > 0$, the \ssp vector $\ppi_n$ reads
\begin{equation} \label{eq:ssp}
	\pi_{n, i} = \frac{\psi(Z_i)}{\sum_{j = 1}^n \psi(Z_j)}, \qquad \forall i \in \{ 1, \dots, n \}.
\end{equation}
If this writing seems restrictive (since, except for the normalization, each \ssp \(\pi_{n, i}\) depends only on \(Z_i\)), we shall remark that it is consistent with optimal \ssps found by \citet{wang2022sampling} for parametric models, without considering beforehand a particular form for the vector \(\ppi_n\).
It is however required in order to understand how the asymptotic covariance operator is affected by the subsampling procedure.
Now, \Cref{hyp:nonexplosive_cov} provides a condition on \(\psi\) for the asymptotic covariance operator not to explode.
Equipped with it, \Cref{thm:srm_normality_R,thm:srm_normality_P} state asymptotic normality of the \srm \(\hsn\).

\begin{assumption} \label{hyp:nonexplosive_cov}
  The interest function \(\psi\) is such that
  \[
    \E{\psi(Z)} < \infty
    \quad \text{and} \quad
    \E{ \psi^{-1}(Z) \ell^\prime(\hz(X), Y)^2 k(X, X) } < \infty.
  \]
\end{assumption}

\begin{theorem}[Asymptotic normality of Multinomial subsampling (\Cref{equ:gamma_multinomial})] \label{thm:srm_normality_R}
	Under \Cref{hyp:kernel,hyp:loss,hyp:distribution,hyp:lsn_speed,hyp:ssp,hyp:nonexplosive_cov},
	if one furthermore assumes that $s_n / n \underset{n \rightarrow \infty}{\longrightarrow} c \geq 0$,
	then
	\begin{equation} \label{eq:hsn_asymp_normality_R}
		\sqrt{s_n} \nabla^2\risk(\hz) (\hsn - \hz) \dconv \GP(0, \Sigma^*_\mathrm{M}(\psi) + c \Sigma),
	\end{equation}
	where
	\begin{equation*}
		\Sigma^*_\mathrm{M}(\psi) \defeq \E{ \frac{\E{ \psi(Z) }}{\psi(Z)} \ell^\prime(\hz(X), Y)^2 K_X \otimes K_X } - \nabla\L(\hz) \otimes \nabla\L(\hz).
	\end{equation*}
\end{theorem}

\begin{theorem}[Asymptotic normality of Poisson subsampling (\Cref{equ:gamma_poisson})] \label{thm:srm_normality_P}
	Under \Cref{hyp:kernel,hyp:loss,hyp:distribution,hyp:lsn_speed,hyp:ssp,hyp:nonexplosive_cov},
	if one furthermore assumes that $s_n / n \underset{n \rightarrow \infty}{\longrightarrow} c \in [0, 1]$, $\pi_{n,i} \leq 1 / s_n$ for all $n > 0$ and $i \in \{ 1, \dots, n \}$,
	then
	\begin{equation} \label{eq:poisson_hsn_asymp_normality}
		\sqrt{s_n} \nabla^2\risk(\hz) (\hsn - \hz) \dconv \GP(0, \Sigma^*_{\mathrm{P},c}(\psi) + c\Sigma),
	\end{equation}
	where
	\begin{equation*}
		\Sigma^*_{\mathrm{P},c}(\psi) \defeq \E{ \frac{\E{ \psi(Z) }}{\psi(Z)} \ell^\prime(\hz(X), Y)^2 K_X \otimes K_X } - c\E{\ell^\prime(\hz(X), Y)^2 K_X \otimes K_X }.
	\end{equation*}
\end{theorem}

\begin{remark}
	If \(c > 0\), \Cref{eq:hsn_asymp_normality_R,eq:poisson_hsn_asymp_normality} can be written:
	\begin{gather*}
		\sqrt{n} \nabla^2\risk(\hz) (\hsn - \hz) \dconv \GP(0, c^{-1} \Sigma^*_\mathrm{M}(\psi) + \Sigma)
		\quad \text{(\Cref{thm:srm_normality_R})}\\
		\sqrt{n} \nabla^2\risk(\hz) (\hsn - \hz) \dconv \GP(0, c^{-1} \Sigma^*_{\mathrm{P},c}(\psi) + \Sigma)
		\quad \text{(\Cref{thm:srm_normality_P})},
	\end{gather*}
	showing that the \srm \(\hsn\) has the same \(1/\sqrt n\)-rate of convergence to \(\hz\) than the \erm \(\hn\),
	but \emph{increased} covariance operators.
	In this situation, as expected, the less data is used (\(c \ll 1\)), the \emph{greater} the covariance operator.
\end{remark}

As a byproduct, \Cref{thm:srm_normality_R,thm:srm_normality_P} also give the asymptotic normality of the \srm, denoted \(\hu\) below, when subsampling with uniform \ssps:
for all \(n > 0\) and all \(i \in \{1, \dots, n\}\), \(\pi_{n, i} = 1/n\).
Let us remark that, in this case, \Cref{hyp:ssp,hyp:nonexplosive_cov} are directly granted.

\begin{corollary}[Asymptotic normality of uniform Multinomial subsampling (\Cref{equ:gamma_multinomial})] \label{cor:uniform_ssp_R}
	Under \Cref{hyp:kernel,hyp:loss,hyp:distribution,hyp:lsn_speed},
	if one furthermore assumes that $s_n / n \underset{n \rightarrow \infty}{\longrightarrow} c \geq 0$,
	then the uniform \srm \(\hu\) verifies:
	\begin{equation*}
		\sqrt{s_n} \nabla^2\risk(\hz) (\hu - \hz) \dconv \GP(0, (1+c) \Sigma).
	\end{equation*}
\end{corollary}

\begin{corollary}[Asymptotic normality of uniform Poisson subsampling (\Cref{equ:gamma_poisson})] \label{cor:uniform_ssp_P}
	Under \Cref{hyp:kernel,hyp:loss,hyp:distribution,hyp:lsn_speed},
	if one furthermore assumes that $s_n / n \conv c \in [0, 1]$,
	then the uniform \srm \(\hu\) verifies:
	\begin{equation*}
		\sqrt{s_n} \nabla^2\risk(\hz) (\hu - \hz) \dconv \GP(0, \Sigma^\mathrm{U}_{\mathrm{P}, c}),
	\end{equation*}
	where $\Sigma^\mathrm{U}_{\mathrm{P}, c} \defeq \E{\ell^\prime(\hz(X), Y)^2 K_X \otimes K_X } - c\nabla\L(\hz) \otimes \nabla\L(\hz)$.
\end{corollary}

\begin{remark}
  For Poisson subsampling, choosing \(c = 1\) implies that, asymptotically, all observations are part of the subsample.
  In this case, \(\hsn\) has convergence rate \(1/\sqrt n\) (see the preceding remark) and asymptotic covariance operator \(\Sigma^\mathrm{U}_{\mathrm{P}, 1} = \Sigma\).
  In other words, the \srm has the same performance as the \erm.
\end{remark}

\section{Variance reduction} \label{sec:optimal}

  Since the role of subsampling is to reduce the dataset size by randomly picking points, it is of interest to choose these few points in order to maximize the \srm estimator precision.
  Once the limit ratio \(c = \lim_{n\to\infty} s_n/n\) has been chosen (see \Cref{thm:srm_normality_R,thm:srm_normality_P}), it is about \emph{minimizing} the asymptotic covariance operator.
  The most common way to do it is to reduce the covariance information to a scalar, and to minimize the latter.
  Several manners of reducing the covariance information exist and are well documented in the literature of optimal design (see for instance \citep{imberg_optimal_2023}).
  In this study, we follow the work by \citet{wang2022sampling} and focus on L-optimality, which consists in computing the trace of the covariance, preconditioned by a well-chosen operator (here \((\nabla^2 R(h_0))^{-1}\), which exists according to \citep[Lemma A.5]{hable2012asymptotic}).
  This choice is consistent with the dependence of the covariance operator to \(\psi\) and is computationally advantageous.

  In our infinite dimensional framework based on \rkhss, the trace of a self-adjoint bounded linear map does not differ much from the finite dimensional case.
  For such an operator $\Lambda : \H \to \H$, it can be defined as $\Tr(\Lambda) \defeq \sum_{j=1}^\infty \dprod{\Lambda e_j}{e_j}$,
  and, in the particular case where \(\Lambda\) is the outer product of two vectors, i.e.\ $\Lambda = f \otimes h$ (with $f, h \in \H$), then $\Tr(\Lambda) = \dprod{f}{h}$
  (see \citep{hsing2015theoretical} for more details).

  Going back to \Cref{thm:srm_normality_P,thm:srm_normality_R}, let \(\Sigma^*(\psi)\) be either \(\Sigma^*_\mathrm{M}(\psi)\) or \(\Sigma^*_{\mathrm{P},c}(\psi)\), depending on the subsampling strategy.
  In both cases, L-optimality boils down to minimizing
  \[
  	\Tr(\Sigma^*(\psi)) = \E{ \frac{\E{ \psi(Z) }}{\psi(Z)} \ell^\prime(\hz(X), Y)^2 k(X, X) } + C,
  \]
  with respect to \(\psi\), where \(C\) is a constant term.
  \Cref{prop:optimal_ssp_R} gives a closed-form expression of such an L-optimal absolute interest map \(\psi\).

  \begin{remark}
  	Since the two subsampling strategies considered here lead to the same convergence rate of the estimator \(\hsn\) (depending on the value of the limit ratio \(c = \lim_{n\to\infty} s_n/n\)),
  	they can be compared to each other based on L-optimality.
	For this purpose, consider a fixed absolute interest function \(\psi\) and let
	\begin{align*}
		\Delta &\defeq \Tr (\Sigma^*_\mathrm{M}(\psi)) - \Tr(\Sigma^*_{\mathrm{P},c}(\psi)) = c \E{\ell^\prime(\hz(X), Y)^2 k(X, X)} - \| \nabla L(\hz) \|_\H^2.
	\end{align*}
	Whenever \(c = 0\) (i.e.\ $s_n = o(n)$), one has \(\Delta < 0\), hence Multinomial subsampling surpasses Poisson subsampling.
	On the contrary, when \(c = 1\), one has \(\Delta = \Tr(\Sigma) > 0\), thus Poisson subsampling is superior to Multinomial subsampling.

	It results that the choice between the two strategies depends on the limit ratio \(c\), which is a quantity defined by the practitioner:
	if \(c\) is small, Multinomial subsampling is preferable to Poisson subsampling, while the converse holds when \(c\) is close to \(1\).
	Let us remark that this conclusion is slightly different from the parametric case studied in \citep{wang2022sampling}, which concludes that both subsampling schemes have the same asymptotically estimation efficiency when \(c = 0\).
	This is due to the presence of a quadratic regularization in our setting, which does not exist in the parametric framework.

  \end{remark}

	\begin{proposition}[L-optimal map \(\psi\)] \label{prop:optimal_ssp_R}
		The minimization problem
		\begin{equation} \label{opb:l-optimality}
			\operatorname{minimize}_{\psi : \Z \rightarrow \R^*_+} \ \Tr\pa*{ \Sigma^*(\psi) },
			\tag{$\mathcal{P}_\psi$}
		\end{equation}
		has a unique solution (up to a multiplicative factor and an \(L^2\) equivalence w.r.t.\ the law of \(Z\)),
		being \(\psi^\star\), defined for all $\bm z = (\bm x, y) \in \Z$, by
		\begin{equation} \label{eq:optimal_psi_R}
			\psi^\star(\bm z) \defeq | \ell^\prime(\hz(\bm x), y) | \sqrt{k(\bm x, \bm x)}.
		\end{equation}
	\end{proposition}

	If it is clear that, given our technical \Cref{hyp:kernel,hyp:loss,hyp:distribution}, the L-optimal absolute interest function \(\psi^\star\) verifies \Cref{hyp:nonexplosive_cov},
	it is not true neither for \Cref{hyp:ssp}, nor for the basic requirement of Poisson subsampling
	(i.e.\ that all \ssps are bounded by \(1/s_n\)).
	This is discussed in the two forthcoming remarks.

	\begin{remark}
		As one may notice, \Cref{prop:optimal_ssp_R} does not take into account that $\psi$ has to verify \Cref{hyp:ssp} in order to be used in application of \Cref{thm:srm_normality_R,thm:srm_normality_P}.
		Indeed, incorporating \Cref{hyp:ssp} as a constraint into Problem~\eqref{opb:l-optimality} would lead to an intractable problem.
		It results that \(\psi^\star\) may not satisfy \Cref{hyp:ssp}.
		However, in practice, the optimal \ssps computed thanks to \(\psi^\star\) are then adjusted such that \Cref{hyp:ssp} is granted (see \Cref{sec:piloted_estimator}).
	\end{remark}

	\begin{remark}
		Poisson subsampling (\Cref{equ:gamma_poisson}) also requires that, for all \(n > 0\),
		\begin{equation}\label{equ:Poisson_contraints}
			\text{for all } i \in \{1, \dots, n\},
			\quad
			\frac{\psi(Z_i)}{\sum_{j=1}^n \psi(Z_j)} \leq \frac{1}{s_n},
		\end{equation}
		which, again, may not be true for \(\psi^\star\).
		If one would like to consider these inequalities as part of Problem~\eqref{opb:l-optimality},
		it would mean handling an infinite amount of constraints (inequalities have to hold true for all \(n > 0\)).
		Moreover, even for a fixed value of \(n\) (keeping in mind the practical use of the proposed procedure),
		the latter constraints would lead to inconsistent (even lack of) solutions, since a candidate \(\psi\) may always be smoothly altered so as to satisfy pointwise conditions.
		As a result, Conditions~\eqref{equ:Poisson_contraints} are not considered when determining an optimal absolute interest function \(\psi^\star\), yet, as explained in \Cref{sec:piloted_estimator}, optimal \ssps computed with \(\psi^\star\) are then post-processed in order to be valid inclusion probabilities for Poisson subsampling.

		Let us remark that, in the parametric case, \citet{wang2022sampling} follow the same path but minimizing (for a fixed value of \(n\)) the trace of an empirical version of \(\Sigma^*(\psi)\) subject to Conditions~\eqref{equ:Poisson_contraints}.
		They then exhibit optimal \ssps proportional to \(\psi^\star \wedge H\), where \(H\) is a data-dependent threshold.
		This workaround does not fit our framework, since the absolute interest map \(\psi\) is thought deterministic and independent from \(n\).
	\end{remark}

	Let us now investigate some examples of supervised learning settings,
	which illustrate the shape of the L-optimal absolute interest function \(\psi^\star\), revealed in \Cref{prop:optimal_ssp_R}.
	For this purpose and for simplicity, we consider a radial basis function kernel \(k\), implying that for all \(\bm x \in \X\), \(k(\bm x, \bm x)\) is constant.

  \begin{example}[Logistic regression] \label{ex:log_reg}
	  In the classification setting of penalized kernel logistic regression, one has,
	  for $y \in \{ -1, 1 \}$ and $\hat y \in \R$:
	  \begin{equation*}
		  \ell(\hat y, y) = \log (1 + \exp(- y \hat y))
		  \quad \text{and} \quad
		  \ell^\prime(\hat y, y) = -y \sigma(- y \hat y),
	  \end{equation*}
	  where $\sigma : t \mapsto 1 / (1 + e^{-t})$ is the sigmoid function.
	  The optimal \ssps are then proportional to
	  \begin{equation*}
		  | \ell^\prime(\hz(\bm x), y) | = \sigma(-y \hz(\bm x)),
		  \quad (\bm x, y) \in \X \times \Y.
	  \end{equation*}
	  Hence a datapoint \(Z_i = (X_i, Y_i)\) has a high probability of being picked if $Y_i \hz(X_i) \ll 0$,
	  i.e.\ if \(X_i\) is wrongly classified by the true function \(\hz\) (\(\hz(X_i)\) and \(Y_i\) have different signs) and with a high confidence (\(|\hz(X_i)|\) is large).
  \end{example}

  \begin{example}[$L^p$ regression] \label{ex:lp_reg}
	  In the $L^p$ (with $p \geq 1$) kernel ridge regression setting, one has,
	  for \(y, \hat y \in \R\):
	  \begin{equation*}
		  \ell(\hat y, y) = \frac{1}{p} |\hat y - y|^p
		  \quad \text{and} \quad
		  \ell^\prime(\hat y, y) = \sign(\hat y - y) | \hat y - y |^{p-1}.
	  \end{equation*}
	  The optimal \ssps are thus proportional to
	  \begin{equation*}
		  | \ell^\prime(\hz(\bm x), y) | = | \hz(\bm x) - y |^{p-1},
		  \quad (\bm x, y) \in \X \times \Y.
	  \end{equation*}
	  Hence a datapoint \(Z_i = (X_i, Y_i)\) has a high probability of being picked if \(Y_i\) is poorly predicted by the true function \(\hz\) (\(|Y_i - \hz(X_i)| \gg 0\)).
	  This is for instance the case for noisy observations.

	  In addition, let us remark that for high values of \(p\), the \ssps (through \(| \ell^\prime(\hz(\bm x), y) |\)) are squashed to \(0\) when \(\hz(\bm x) \approx y\), thus imitating an \(\epsilon\)-insensitive loss \citep{sangnier_data_2017}, with the nice property of being infinitely differentiable,
	  and producing \ssps of different magnitudes.
  \end{example}

  \begin{example}[Quartic-quadratic Huber] \label{ex:qqh_reg} 
	  As stated above, the $L^p$ loss with a high value of $p$ produces \ssps of different magnitudes, attributing near zero inclusion probabilities to well predicted datapoints.
	  On the converse, it assigns high \ssps to datapoints, the outputs \(y\) of which are far from the predictions \(\hz(\bm x)\).
	  As a result, and as it may be expected by using an \(L^p\) loss, the \ssps are not robust to outliers.

	  To keep the benefit of an almost \(\epsilon\)-insensitive loss (i.e.\ \ssps with different magnitudes) while being more robust to outliers,
	  one may defined a ``Huberized'' \(L^2\) loss (see \citep[Section~10.6]{hastie_elements_2013}),
	  which is quartic around \(0\) and quadratic elsewhere.
	  For this purpose, let $\epsilon > 0$ be a given threshold, and consider
	  \begin{equation*}
		  \ell_\epsilon(\hat y, y) \defeq \begin{cases}
			  \frac{1}{4} (\hat y - y)^4 &\text{if}\ | \hat y - y | \leq \epsilon, \\
			  \frac{3}{2} \epsilon^2 (\hat y - y)^2 - 2 \epsilon^3 | \hat y - y | + \frac{3}{4} \epsilon^4 &\text{otherwise},
		  \end{cases}
	  \end{equation*}
	  which is twice-differentiable, with
	  \begin{equation*}
		  \ell^\prime_\epsilon(\hat y, y) = \begin{cases}
			  (\hat y - y)^3 &\text{if}\ | \hat y - y | \leq \epsilon, \\
			  3 \epsilon^2 (\hat y - y) - 2 \epsilon^3 \sign(\hat y - y) &\text{otherwise}.
		  \end{cases}
	  \end{equation*}

	  In the same manner as for \(L^p\) regression, the optimal \ssps are as large as the output \(y\) is wrongly predicted by \(\hz(\bm x)\),
	  nevertherless, there is a regime shift whether the prediction is at least (or at most) $\epsilon$ far away from the output $y$.
	  This leads to \ssps of different magnitudes, which are more robust to outliers than those obtained with an \(L^p\) loss.
  \end{example}

\section{Piloted subsampling estimator} \label{sec:piloted_estimator}

	We address in this section the practical side of subsampling with the L-optimal absolute interest function \(\psi^\star\), exhibited in \Cref{prop:optimal_ssp_R}.
	First, as it can be noticed, \(\psi^\star\) makes use of the unknown regularized risk minimizer \(\hz\).
	Nevertheless, the latter can be estimated by plug-in
	as soon as we are provided with prior information on \(\hz\), which is assumed to be true in the form of a naive estimator \(\hp\) of \(\hz\) (called a pilot estimator, and computed with \(k_n\) points).
	Based on it, L-optimal \ssps can be estimated as
	\(\tilde \pi_{n, i} \propto |\ell^\prime(\hp(X_i), Y_i)| \sqrt{k(X_i, X_i)}\)
	for all \(i \in \{1, \dots, n\}\).
	Second, as explained in the previous section, there is no guarantee that the map \(\psi^\star\) leads to \ssps which verify \Cref{hyp:ssp}.
	For this reason, it is proposed to smooth estimated \ssps as:
	\[
		\tilde \pi_{n, i}^\alpha = (1-\alpha) \tilde \pi_{n, i} + \frac \alpha n,
	\]
	for all \(i \in \{1, \dots, n\}\), where \(\alpha \in (0, 1]\) is a hyperparameter.
	As a last point, since \ssps should be less than \(1/s_n\) for Poisson subsampling, one may apply a post hoc truncation, such as \(\tilde \ppi_n^\alpha \wedge s_n^{-1}\).
	Everything together provides data driven \ssps satisfying all requirements and usable to build an \srm \(\hasn\), called an \(\alpha\)-piloted subsampling estimator.

	\Cref{thm:srm_alpha_pilot_normality_R,thm:srm_alpha_pilot_normality_P} below state that the latter estimator is consistent to \(\hz\) and asymptotically normal.
	In particular, it exhibits a covariance operator (which is well-defined and finite by the very definition of \(\psi^\star\) and \Cref{hyp:kernel,hyp:loss,hyp:distribution}), which, compared to that of \Cref{thm:srm_normality_R,thm:srm_normality_P}, is slightly impacted by smoothing the \ssps (tuned by the hyperparameter \(\alpha\)) but not by the use of a plug-in estimation of \(\psi^\star\).

\begin{theorem}[Asymptotic normality of $\alpha$-piloted Multinomial subsampling (\Cref{equ:gamma_multinomial})] \label{thm:srm_alpha_pilot_normality_R}
	Under \Cref{hyp:kernel,hyp:loss,hyp:distribution,hyp:lsn_speed}, let $\hp$ be a consistent estimator of $\hz$, $\alpha \in (0, 1]$ and define the probability vectors $\tilde\ppi_n$ and $\tilde\ppi_n^\alpha$ such that
	\begin{equation*}
		\tilde\pi_{n, i} = \frac{|\ell^\prime(\hp(X_i), Y_i)| \sqrt{k(X_i, X_i)}}{\sum_{j=1}^n |\ell^\prime(\hp(X_j), Y_j)| \sqrt{k(X_j, X_j)}} \quad \text{and} \quad \tilde \pi_{n,i}^\alpha = (1 - \alpha) \tilde \pi_{n, i} + \frac{\alpha}{n},
	\end{equation*}
	for all $n > 0$ and $i \in \{ 1, \dots, n \}$.
	Let $\hasn$ be the \srm of the subsampled regularized risk with \ssp vector $\tilde \ppi_n^\alpha$. If one furthermore assumes that $s_n / n \conv c \geq 0$, then
	\begin{equation*}
		\sqrt{s_n} \nabla^2\risk(\hz) (\hasn - \hz) \dconv \GP(0, \Sigma^*_\mathrm{M}(\alpha) + c\Sigma),
	\end{equation*}
	where
	\begin{equation*}
		\Sigma^*_\mathrm{M}(\alpha) \defeq \E{ \frac{\Ex{ \psi^\star(Z) }}{\Psi_\alpha^\star(Z)} \ell^\prime(\hz(X), Y)^2 K_X \otimes K_X } - \nabla\L(\hz) \otimes \nabla\L(\hz),
	\end{equation*}
	with $\Psi^\star_\alpha(Z) \defeq (1 - \alpha) \psi^\star(Z) + \alpha \Ex\psi^\star(Z)$.
\end{theorem}

\begin{theorem}[Asymptotic normality of $\alpha$-piloted Poisson subsampling (\Cref{equ:gamma_poisson})] \label{thm:srm_alpha_pilot_normality_P}
	Under \Cref{hyp:kernel,hyp:loss,hyp:distribution,hyp:lsn_speed}, let $\hp$ be a consistent estimator of $\hz$, $\alpha \in (0, 1]$ and define the probability vectors $\tilde\ppi_n$ and $\tilde\ppi_n^\alpha$ such that
	\begin{equation*}
		\tilde\pi_{n, i} = \frac{|\ell^\prime(\hp(X_i), Y_i)| \sqrt{k(X_i, X_i)}}{\sum_{j=1}^n |\ell^\prime(\hp(X_j), Y_j)| \sqrt{k(X_j, X_j)}} \quad \text{and} \quad \tilde \pi_{n,i}^\alpha = (1 - \alpha) \tilde \pi_{n, i} + \frac{\alpha}{n},
	\end{equation*}
	for all $n > 0$ and $i \in \{ 1, \dots, n \}$.
	Let $\hasn$ be the \srm of the subsampled regularized risk with \ssp vector $\tilde \ppi_n^\alpha \wedge s_n^{-1}$.
	If one furthermore assumes that $s_n / n \conv c \in [0, 1]$, then
	\begin{equation*}
		\sqrt{s_n} \nabla^2\risk(\hz) (\hasn - \hz) \dconv \GP(0, \Sigma^*_{\mathrm{P},c}(\alpha) + c\Sigma),
	\end{equation*}
	where
	\begin{equation*}
		\Sigma^*_{\mathrm{P},c}(\alpha) \defeq \E{ \frac{\E{ \psi^\star(Z) } \ell^\prime(\hz(X), Y)^2}{ \Psi^\star_\alpha(Z) \wedge (c^{-1} \E{\psi^\star(Z)})} K_X \otimes K_X } - c \E{ \ell^\prime(\hz,(X), Y)^2 K_X \otimes K_X }.
	\end{equation*}
\end{theorem}

A particularity of \(\alpha\)-piloted Poisson subsampling (\Cref{thm:srm_alpha_pilot_normality_P}) is the use of a post hoc truncation of \ssps: \(\tilde \ppi_n^\alpha \wedge s_n^{-1}\),
in order to get valid Bernoulli distribution parameters.
It has two direct consequences.
First, the subsampled dataset size is in expectation less than or equal to \(s_n\) rather than equal to \(s_n\).
Second, a thresholdind based on \(c^{-1} \E{\psi^\star(Z)}\) appears in the
denominator of the asymptotic covariance operator $\Sigma_{\mathrm{P}, c}^*$ of \(\hasn\).

In practice, the pilot estimator \(\hp\) can be easily obtained by picking uniformly a small number of datapoints and by minimizing an empirical regularized risk.
In this situation, \Cref{thm:srm_consistency_core} ensures consistency of \(\hp\).
The proposed workable procedure for computing an \(\alpha\)-piloted subsampling estimator is summarized \Cref{alg:alpha_piloted}.

	\begin{algorithm}
		Require \(s_n \le n\) (desired subsampled dataset size), \(k_n \ll s_n\) (pilot dataset size), \(\alpha \in (0, 1]\) (smoothing hyperparameter).

		\begin{enumerate}
			\item Consider a random permutation \(\sigma\) on \(\{1, \dots, n\}\) and build a pilot estimator
			\[
				\hat h_{k_n} \in \argmin_{h \in \H} \cb*{ \frac{1}{k_n} \sum_{i=1}^{k_n} \ell(h(X_{\sigma(i)}), Y_{\sigma(i)}) + \frac{\lambda_{k_n}^*}{2} \| h \|^2_\H }. 
			\]
			\item Estimate optimal \ssps by plug-in:
			\[
				\tilde\pi_{n, i} = \frac{|\ell^\prime(\hp(X_i), Y_i)| \sqrt{k(X_i, X_i)}}{\sum_{j=1}^n |\ell^\prime(\hp(X_j), Y_j)| \sqrt{k(X_j, X_j)}},
			\]
			for all \(i \in \{1, \dots, n\}\).
			\item Smooth the estimated \ssps:
			\[
				\tilde \pi_{n,i}^\alpha = (1 - \alpha) \tilde \pi_{n, i} + \frac{\alpha}{n},
			\]
			for all \(i \in \{1, \dots, n\}\).
			\item Subsample datapoints either with replacement (with \ssp vector \(\tilde \ppi_n^\alpha\))
			or without replacement (with \ssp vector \(\tilde \ppi_n^\alpha \wedge s_n^{-1}\))
			and get multiplicities \(\gamma_{n, i}^\alpha\).
			\item Build the \(\alpha\)-piloted \srm:
			\[
				\hasn \in \argmin_{h \in \H} \cb*{ \frac{1}{n} \sum_{i=1}^{n} \frac{\ell(h(X_i), Y_i)}{s_n \tilde \pi_{n,i}^\alpha} \gamma_{n,i}^\alpha + \frac{\lsn}{2} \| h \|^2_\H }.
			\]
		\end{enumerate}
		\caption{\(\alpha\)-piloted subsampling estimator}
		\label{alg:alpha_piloted}
	\end{algorithm}


\section{Numerical experiments} \label{sec:numerical}
	\subsection{Setting}
		In this section, we empirically assess the performance of the proposed \(\alpha\)-piloted subsampling method (\Cref{alg:alpha_piloted}) compared to a naive approach (uniform subsampling)
		and three state-of-the-art kernel approximation and dataset reduction techniques: the Nyström method \citep{williams2000using,drineas_nystrom_2005}, Random Fourier Features \citep{rahimi2007random} and kernel sketching \citep{yang_randomized_2017, elahmad2023fast}.\footnote{The Python source code for reproducing the numerical experiments is provided at \href{https://gitlab.com/Yaloude/rkhs-subsampling}{https://gitlab.com/Yaloude/rkhs-subsampling}.}
		By performance, it is meant first the generalization error of the methods (\Cref{fig:varying_alpha}),
		second their computational effort, that is the curve mapping the generalization error to the training CPU time (\Cref{fig:effort,fig:space_budget_effort,fig:synthetic_space_budget_effort_mixed}).
		Let us now describe the competitive methods.

		\paragraph{\(\alpha\)-piloted subsampling (\srm)}
		This is the method proposed in \Cref{alg:alpha_piloted} with \(k_n = s_n/3\) and \(\alpha = 0.2\).
		The \emph{subsampling budget} is here \(b_n = k_n+s_n\), where a quarter is used for computing the pilot estimator \(\hp\) and the remaining for the final subsampling estimator \(\hasn\).

		\paragraph{Uniform subsampling}
		This corresponds to Multinomial or Poisson subsampling with uniform \ssps: $\ppi_n^\mathrm{U} = (1 / n, \dots, 1 / n)$ (see \Cref{cor:uniform_ssp_R,cor:uniform_ssp_P}),
		and constitutes a naive, dataset unrelated and problem agnostic subsampling strategy used as a baseline.
		The subsampling budget is here \(b_n = s_n\).

		\paragraph{Nyström}
		The Nyström method is a kernel approximation technique, where the Gram matrix \( K = (k(X_i, X_j))_{1 \leq i,j \leq n} \) is replaced by a low rank approximation computed on \( m \in \N \) randomly chosen observations from $\dn$ as basis vectors.
		From the memory footprint point of view, it results in an \( n \times m \) data matrix instead of the \( n \times n \) Gram matrix \( K \) for \erm.

		\paragraph{Random Fourier features (\rff)}
		The canonical feature map \(\bm x \in \X \mapsto k(\bm x, \cdot) \in \H\) is approximated by a \(p\)-valued function, where \(p \in \N\) is the number of random cosine basis functions chosen.
		This method yields a data matrix of size \(n \times p\) compared to $n \times n$ for \erm.

		\paragraph{Sketching (SKT)}
		Sketching is a dataset reduction method, basically working on \( l \in \N \) random linear combinations of the observations.
		We chose, for the comparison, the \( p \)-sparsified Rademacher sketch, as introduced by \citet{elahmad2023fast}, with probability \( p = 30 / n \) of non-empty entry in the sketch matrix.
		Let us remark that, if the memory footprint for the optimization procedure is \(n \times l\), the method requires first to compute the \(n \times n\) Gram matrix before getting the sketched one.\\

		The methods above are all compared based on a Gaussian radial basis kernel: \( k(\bm x, \bm x^\prime) = e^{- \gamma\| \bm x - \bm x^\prime \|_2^2} \) for all \( \bm x, \bm x^\prime \in \R^d \),
		with parameter $\gamma$ set to the inverse of the empirical variance of the observations (rounded to two significant digits), unless said otherwise.
		For this comparison, two synthetic datasets (corresponding to a classification and a regression task) and a real-world dataset coming from the UCI repository (Covertype) are employed and described below.

		\paragraph{Synthetic classification}
		The dataset is represented in \Cref{subfig:classification_dataset}:
		the input vector is such that \(X \sim \mathrm{Unif}([0, 1]^2)\) and three random centers are drawn from the same uniform distribution.
		Denoting \(d(\bm x)\) the minimum distance between a spatial observation \(\bm x\) and the latter centers,
		the binary classification labels are such that \(Y | X \sim \mathrm{Rad}(\sigma(100 (d(X) - 0.25)))\),
		where \(\mathrm{Rad}(p)\) is the Rademarcher distribution with parameter \(p \in (0, 1)\),
		and \(\sigma : t \in \R \mapsto 1 / (1 + e^{-t})\) is the sigmoid function.

		For this task, a million observations (independent from the training data) are used for computing the test errors,
		and the logistic loss is used with regularization parameters \(\lsn = \lz = 10^{-5}\).

		\begin{figure}[ht]
			\centering
			\begin{subfigure}{0.48\linewidth}
				\includegraphics[width=\linewidth]{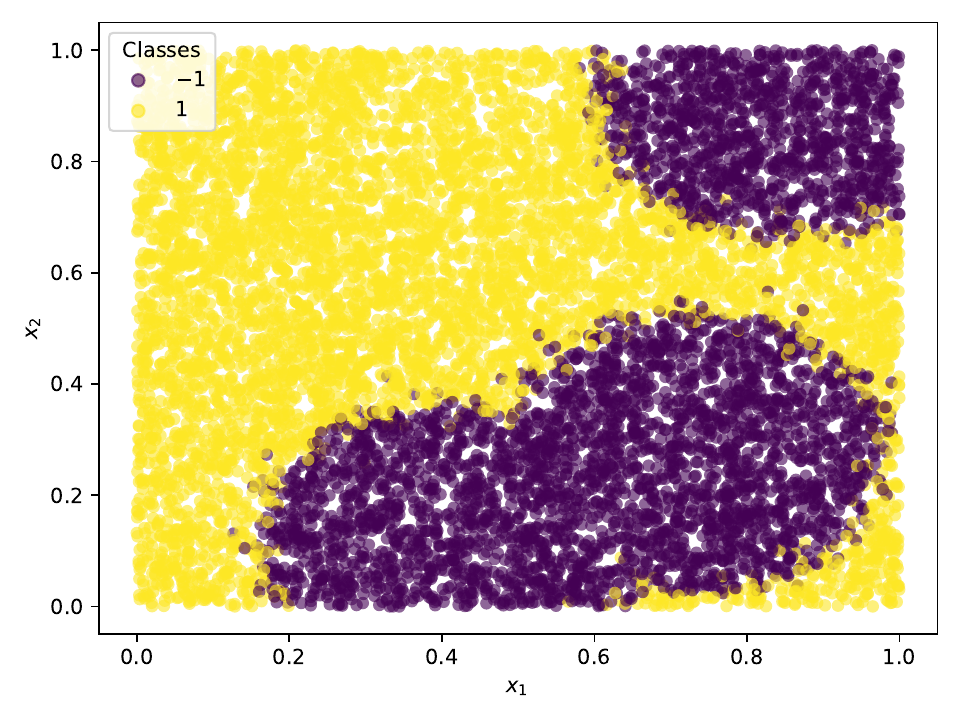}
				\caption{Classification}
				\label{subfig:classification_dataset}
			\end{subfigure}
			\hfill
			\begin{subfigure}{0.48\linewidth}
				\includegraphics[width=\linewidth]{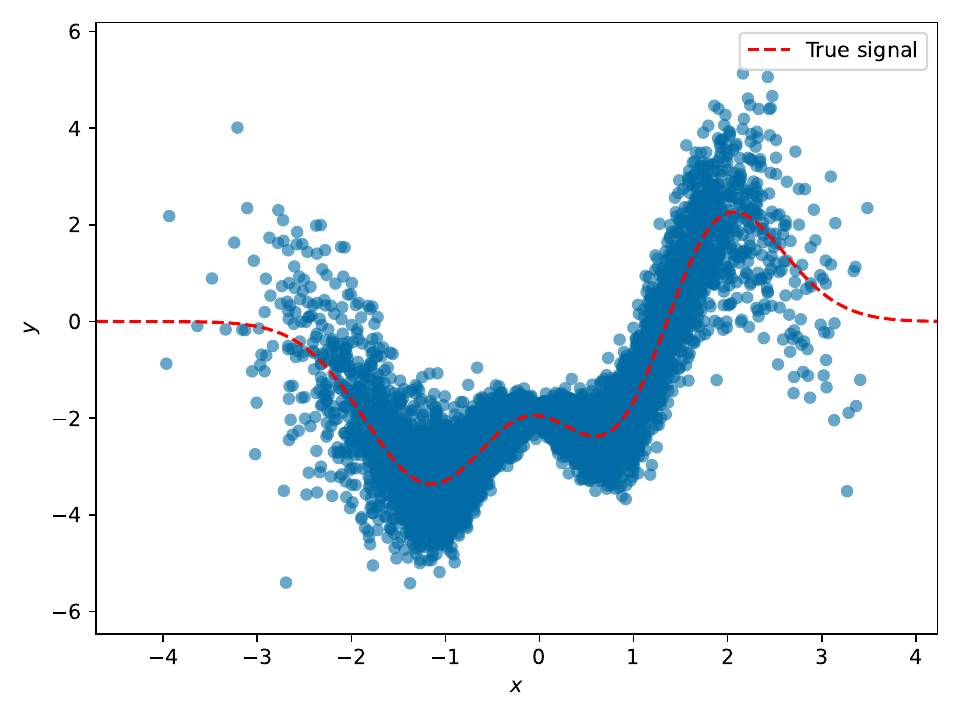}
				\caption{Regression}
				\label{subfig:regression_dataset}
			\end{subfigure}
			\caption{Visualization of the synthetic datasets used in the numerical experiments.}
			\label{fig:datasets}
		\end{figure}

		\paragraph{Synthetic regression}
		The dataset is represented in \Cref{subfig:regression_dataset}:
		the function to estimate is \(h^\star =(5 / \sqrt{\sum_{1 \le i, j \le m} \beta_i^0 \beta_j^0 k(x_i^0, x_j^0)}) \sum_{i=1}^m \beta_i^0 k(x_i^0, \cdot)\) (i.e.\ \( \| h^\star \|_\H = 5 \)),
		where \(m = 10\), \(x_1^0\) to \(x_m^0\) are random centers independently drawn according to a standard Gaussian distribution, and factors \(\beta_1^0\) to \(\beta_m^0\) are sampled from the same distribution.
		Observations are \((X_i, Y_i)\)s, where for all \(i \in \{1, \dots, n\}\),
		\begin{equation*}
			Y_i \defeq h^\star(X_i) + \pa*{ \frac{1}{2} |X_i| + 0.1 } \varepsilon_i,
		\end{equation*}
		with $\varepsilon_i \overset{i.i.d.}{\sim} \mathcal{N}(0, 1)$.

		For this task, a \(10^6\)-sample independent from the training data is used for computing test errors,
		and the previously defined quartic-quadratic Huber loss (with \(\varepsilon = 1\)) is used with regularization parameters \(\lsn = \lz = 10^{-3}\).
		The latter loss has the benefit of being almost insensitive to small residues, while not less robust to outliers than the least squares loss.
		Let us remark that, since \(\lz > 0\), the true regularized risk minimizer \(\hz\) is different from \(h^\star\).

		\begin{figure}[ht]
			\centering
			\begin{subfigure}{0.48\linewidth}
				\includegraphics[width=\linewidth]{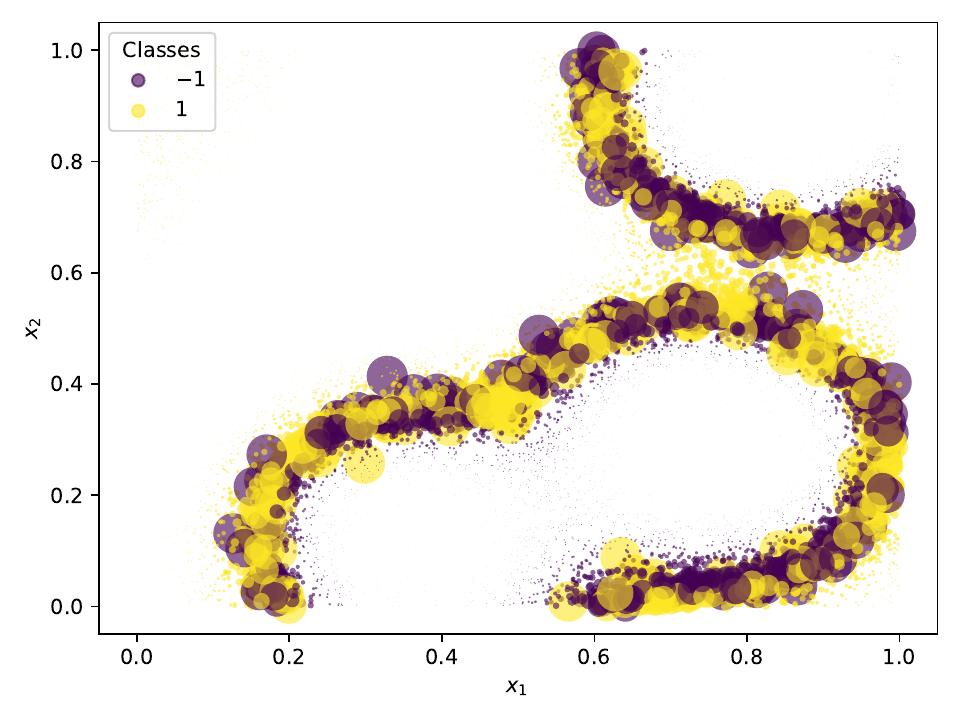}
				\caption{Classification}
				\label{subfig:classification_ssp_visu}
			\end{subfigure}
			\hfill
			\begin{subfigure}{0.48\linewidth}
				\includegraphics[width=\linewidth]{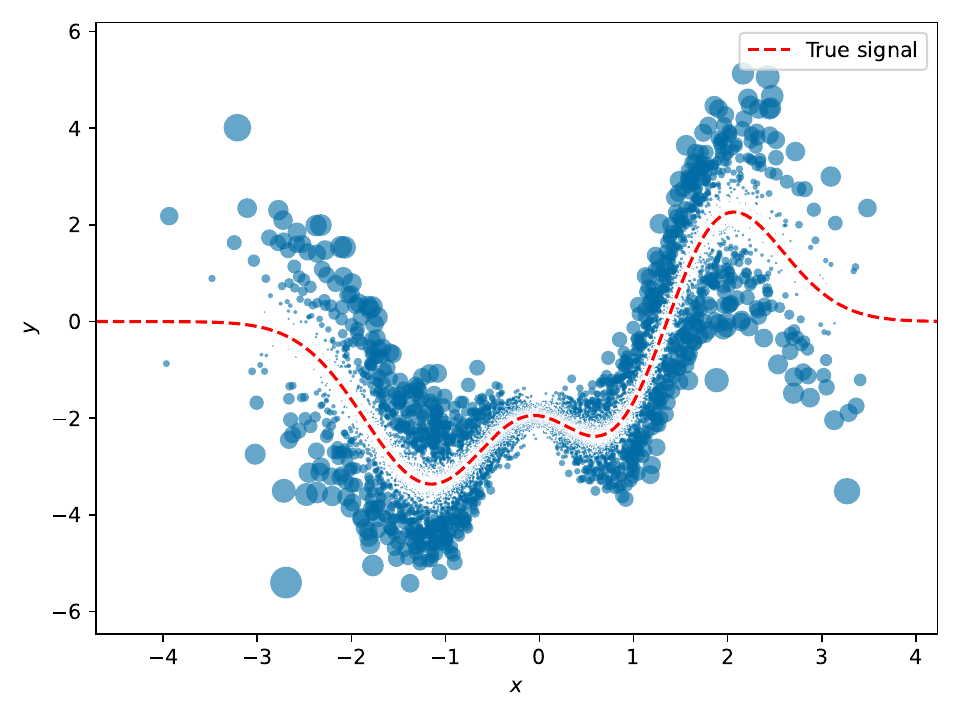}
				\caption{Regression}
				\label{subfig:regression_ssp_visu}
			\end{subfigure}
			\caption{Visualization of the L-optimal \ssps on the synthetic datasets
			(the size of each datapoint is proportional to its probability of being subsampled).}
			\label{fig:ssp_visu}
		\end{figure}

		For both synthetic datasets, \Cref{fig:ssp_visu} depicts the same points as in \Cref{fig:datasets} but with the size of each point proportional to its L-optimal \ssp.
		It clearly reveals that important points for building the \srm \(\hsn\) are at the frontiers of classes for classification, and away from \(h^\star\) for regression.

		\paragraph{UCI Covertype}
		The \href{https://archive.ics.uci.edu/dataset/31/covertype}{CovType} dataset is a real-world repository of biological attributes on seven different forest cover types.
		It contains 581,012 records of 54 features each, with 10 of which being quantitative and the rest being binary.
		We aim at separating the last class from the others, which results in
		an imbalanced binary classification problem,
		handled with the logistic loss (meeting \Cref{hyp:loss}).

		The dataset is split into two parts: 100,000 points are used for training and the rest for testing.
		Quantitative covariates are standardized and hyperparameters are set to \( \gamma = 1 \) and \( \lsn = \lz = 5 \cdot 10^{-6} \).

	\subsection{Results} \label{sec:numerical_results}
		Methods are compared in three different settings, described below.

		\paragraph{Varying smoothing hyperparameter \(\alpha\)}
		As a first numerical experiment, we study the impact of the smoothing hyperparameter \(\alpha \in (0, 1]\) on the generalization error.
		\Cref{fig:varying_alpha} depicts classification and regression errors (computed on the test set) for \srm with Multinomial (\srm (M)) and Poisson (\srm (P)) subsampling, with respect to \(\alpha\).
		Both methods have the same subsampling budget \(b_n = k_n + s_n\) when \(\alpha < 1\),
		and \(b_n = s_n\) when \(\alpha = 1\).
		In both synthetic and real-world situations, it appears that the higher \(\alpha\) (the closer to uniform subsampling), the higher the generalization error (and both subsampling methods perform roughly the same).
		In particular, for \(\alpha \le 0.4\), the \srm error stays roughly constant and is less than that for uniform subsampling (\(\alpha = 1\)).
		As an exception, L-optimal subsampling without smoothing (\(\alpha = 0\)) for the synthetic classification task is as bad as uniform subsampling.
		This may be so because the optimal map \(\psi^\star\) does not verify assumptions of \Cref{thm:srm_normality_R,thm:srm_normality_P}.
		As a result of this first lesson, we fix \( \alpha = 0.2 \) for the forthcoming experiments.

		Moreover, as our framework is based on the consistency of \(\hasn\), \Cref{fig:regression_alpha_distances,fig:classification_alpha_distances} in \Cref{sec:additional} provide \( L^2 \), \( L^\infty \) or \( \H \)-distances between \(\hasn\) and \(\hz\).
		The results are similar than those of \Cref{fig:varying_alpha} and show that a value of \(\alpha\) between \(0.1\) and \(0.4\) leads to a better \srm estimator than uniform subsampling.
		This means that, in this setting, statistical and learning performances are in line with each other.

		\begin{figure}[ht]
			\centering
			\begin{subfigure}{0.48\linewidth}
				\includegraphics[width=\linewidth]{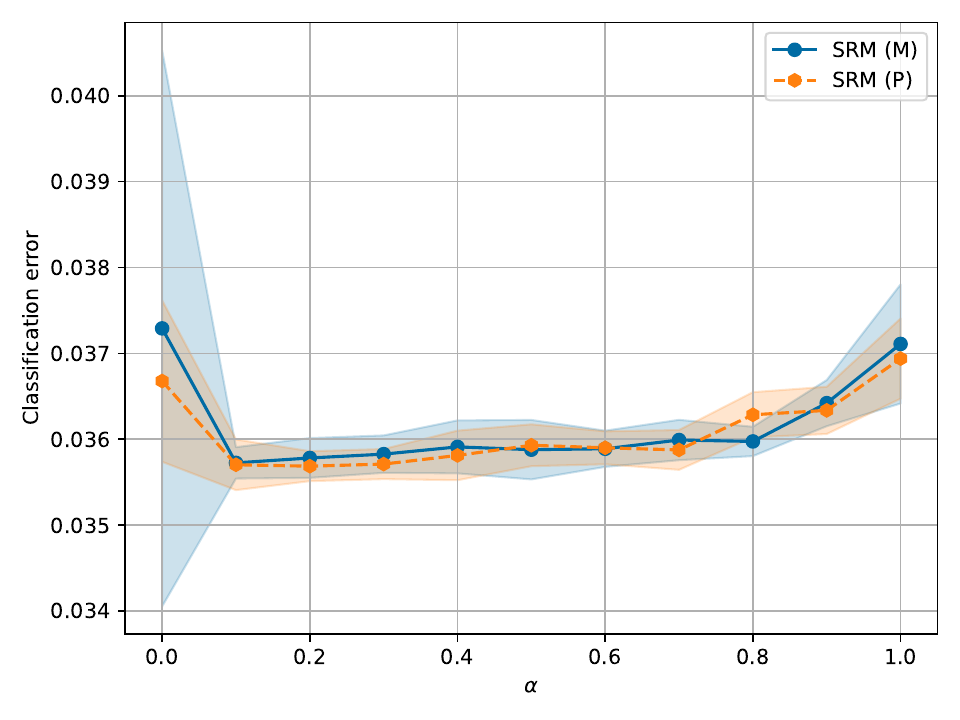}
				\caption{Synthetic classification with \( n = 10^5 \) and \( b_n / n = 0.02 \).}
				\label{subfig:classification_alpha}
			\end{subfigure}
			\hfill
			\begin{subfigure}{0.48\linewidth}
				\includegraphics[width=\linewidth]{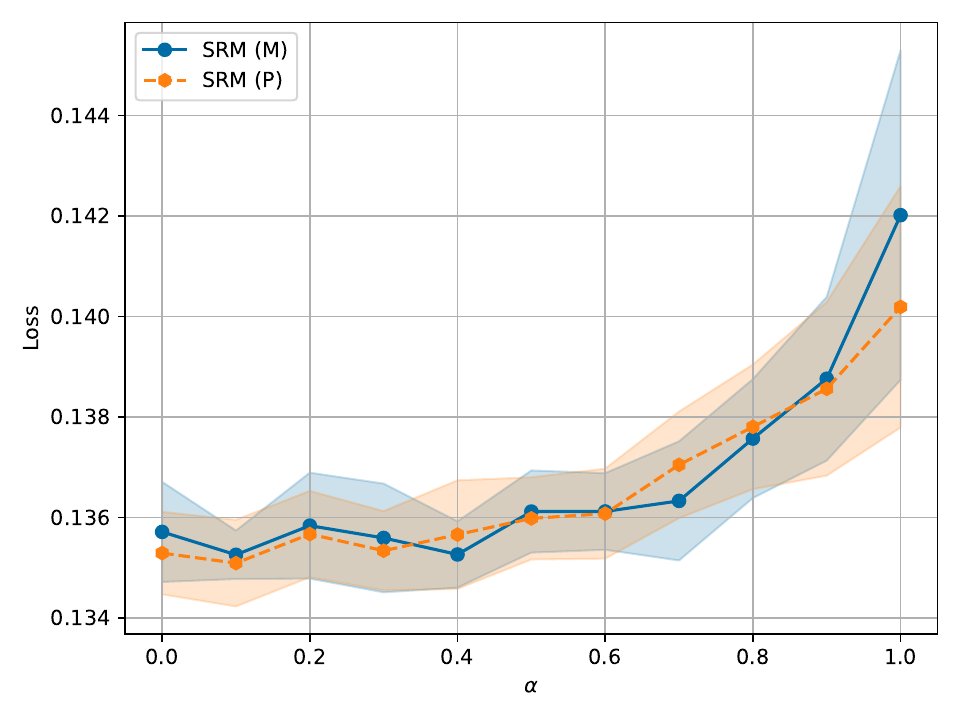}
				\caption{Synthetic regression with \( n = 10^5 \) and \( b_n / n = 0.01 \).}
				\label{subfig:regression_alpha}
			\end{subfigure} \\
			\begin{subfigure}{0.48\linewidth}
				\includegraphics[width=\linewidth]{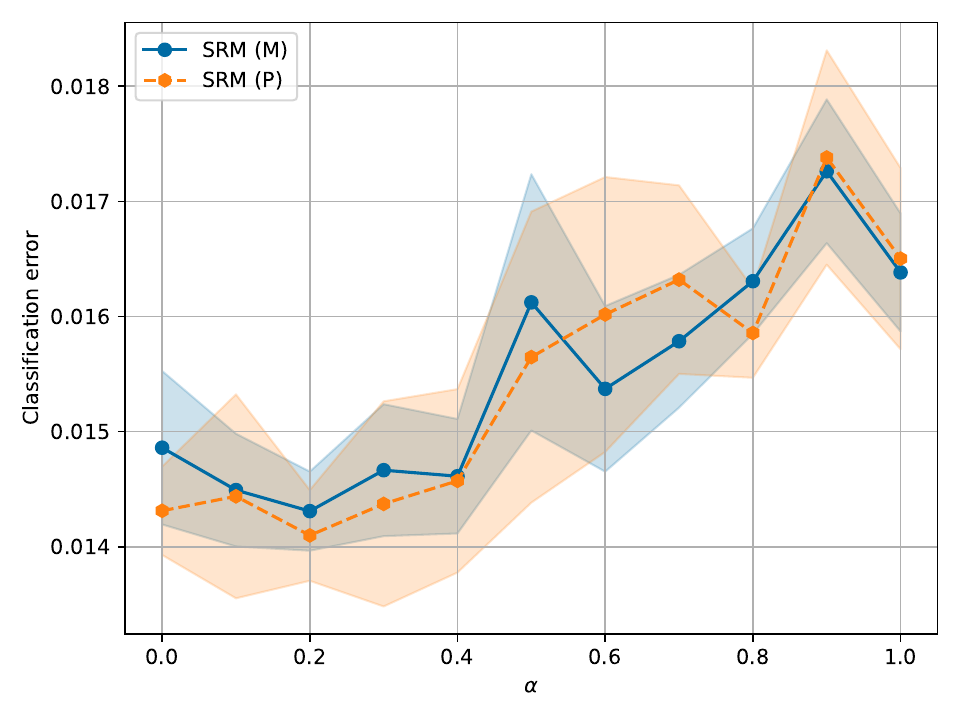}
				\caption{Covtype with \( n = 10^5 \) and \( b_n / n = 0.2 \).}
				\label{subfig:covtype_alpha}
			\end{subfigure}
			\caption{Generalization error with respect to the smoothing hyperparameter \( \alpha \).}
			\label{fig:varying_alpha}
		\end{figure}

		\paragraph{Varying dataset size \( n \), fixed subsampling budget ratio \( b_n / n \)}
		In this second experiment, uniform subsampling is compared to \(\alpha\)-piloted \srm on both synthetic tasks when the dataset size \(n\) increases (from \( 10^4 \) to \( 5 \cdot 10^5 \)) but
		the subsampling budget ratio \(b_n / n\) is fixed.
		\Cref{fig:effort} represents the computational effort of each method, i.e.\ the CPU training time with respect to the generalization error, here as a curve parameterized by \(n\) (in particular, the points from left to right on each curve correspond to the same values of \(n\)).
		It appears that for small datasets (right part of the graphs), \srm is generally less efficient than uniform subsampling: for a given generalization error, \srm is more (time) consuming than uniform subsampling, and conversely, for a given training time, \srm is less accurate.
		This may be explained by a spurious information given by the pilot estimator, which is computed on too few data.
		However, the exact converse is true for large dataset sizes \(n\) (left part of the graphs).
		To sum up, for small datasets, uniform subsampling is preferable, while for large datasets, \srm (which is based on asymptotic results) is more advantageous.
		This holds true for both subsampling techniques, which perform similarly.

		Surprisingly, for a given dataset size \(n\), uniform subsampling is longer to train than \(\alpha\)-piloted \srm, while it includes computing a pilot estimator.
		This can actually be explained by the fact that the computational cost is mainly due to computing \(\hasn\), that is solving an optimization problem of size \((b_n - k_n)^2\) for \srm and \(b_n^2\) for uniform subsampling.
		The forthcoming experiment improves on this interpretation bias.

		\begin{figure}[ht]
			\centering
			\begin{subfigure}{0.48\linewidth}
				\includegraphics[width=\linewidth]{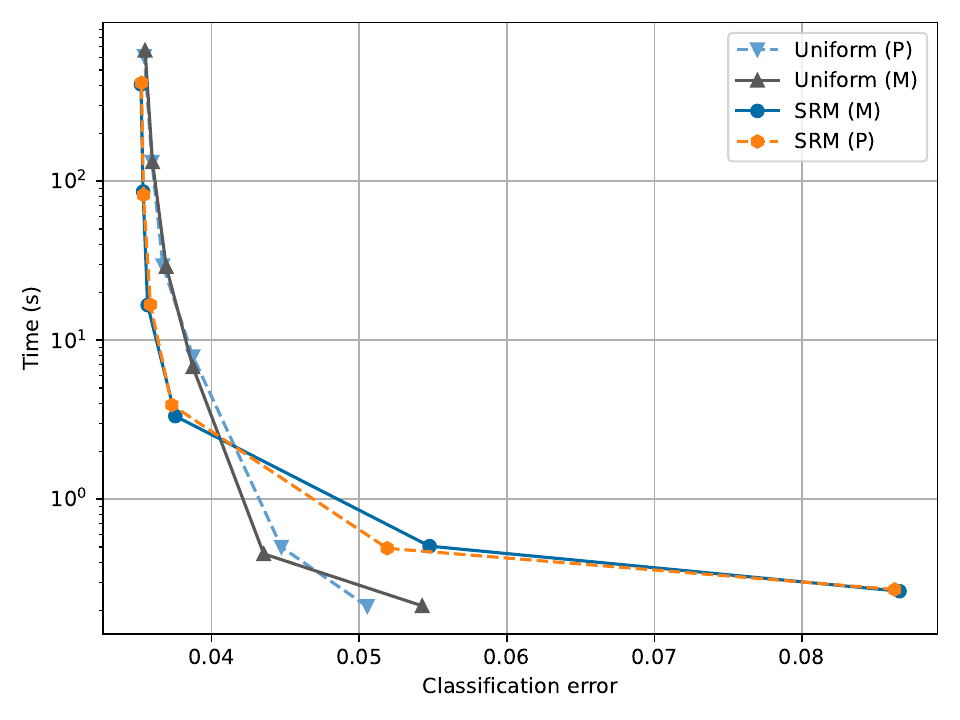}
				\caption{Synthetic classification.}
				\label{subfig:classification_effort}
			\end{subfigure}
			\hfill
			\begin{subfigure}{0.48\linewidth}
				\includegraphics[width=\linewidth]{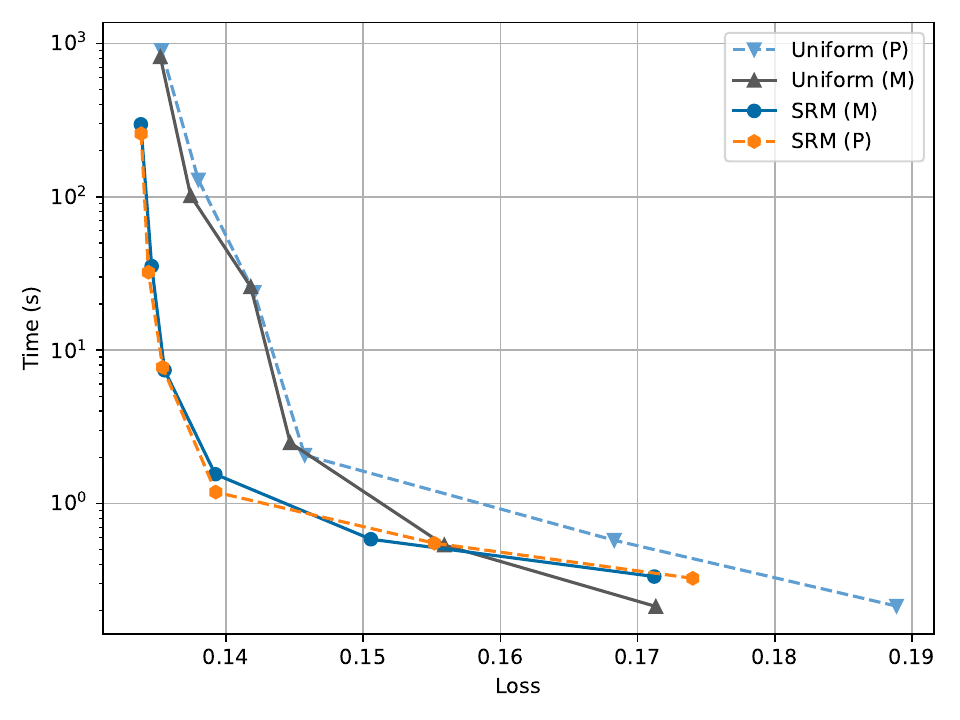}
				\caption{Synthetic regression.}
				\label{subfig:regression_effort}
			\end{subfigure}
			\caption{Computational effort parameterized by the sample size \(n\) (fixed \(b_n/n\)).}
			\label{fig:effort}
		\end{figure}

		\paragraph{Varying space budget, fixed dataset size \(n\)}
		This last numerical experiment is closer to a real situation (and thus also includes the real-world case of cover type classification), where the dataset is given (so \(n\) is fixed), and methods are compared thanks to the computational effort parameterized by the (memory) space budget.
		This one corresponds to the size of the optimization problems to solve, that is \(b_n^2\) for uniform subsampling and \(k_n^2 + (b_n - k_n)^2\) for \(\alpha\)-piloted \srm.
		On that occasion, it is also possible to fairly include state-of-the-art competitors, i.e.\ Nyström, \rff and SKT, for which the space budget is \(nm\), where \(m\) is the approximation parameter of the method: the dimension of the subspace for the Nyström approach, the number of random projections for \rff and the number of linear combinations for the SKT.
		As evidenced by \Cref{fig:space_budget_effort}, since \(n\) is large, \srm is always favorable compared to uniform subsampling.
		It is also more efficient than Nyström, \rff and SKT on the CovType dataset (see \Cref{subfig:covtype_space_budget_effort}).

		Moreover,
		for a given space budget (a given point on the curves), \(\alpha\)-piloted \srm and uniform subsampling are roughly as long to train than each other but the former is more accurate than the latter.
		This means that, for a prescribed memory budget, \srm is preferable to uniform subsampling (for large datasets), but also that estimating the L-optimal \ssps has a negligible computational time.

		\begin{figure}[ht]
			\centering
			\begin{subfigure}{0.48\linewidth}
				\includegraphics[width=\linewidth]{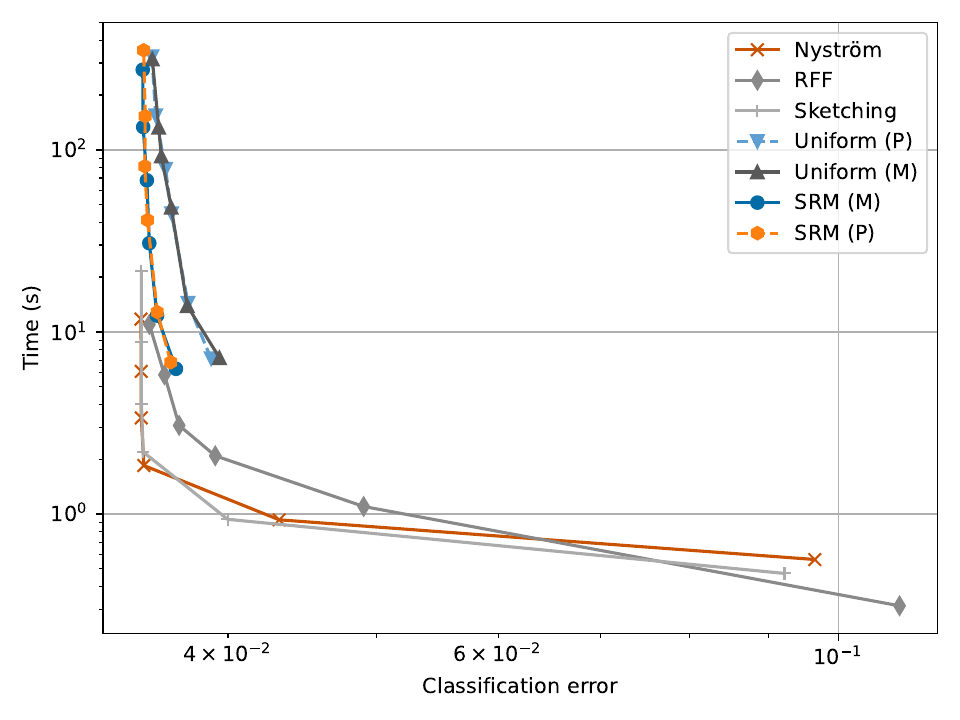}
				\caption{Synthetic classification.}
				\label{subfig:classification_space_budget_effort}
			\end{subfigure}
			\hfill
			\begin{subfigure}{0.48\linewidth}
				\includegraphics[width=\linewidth]{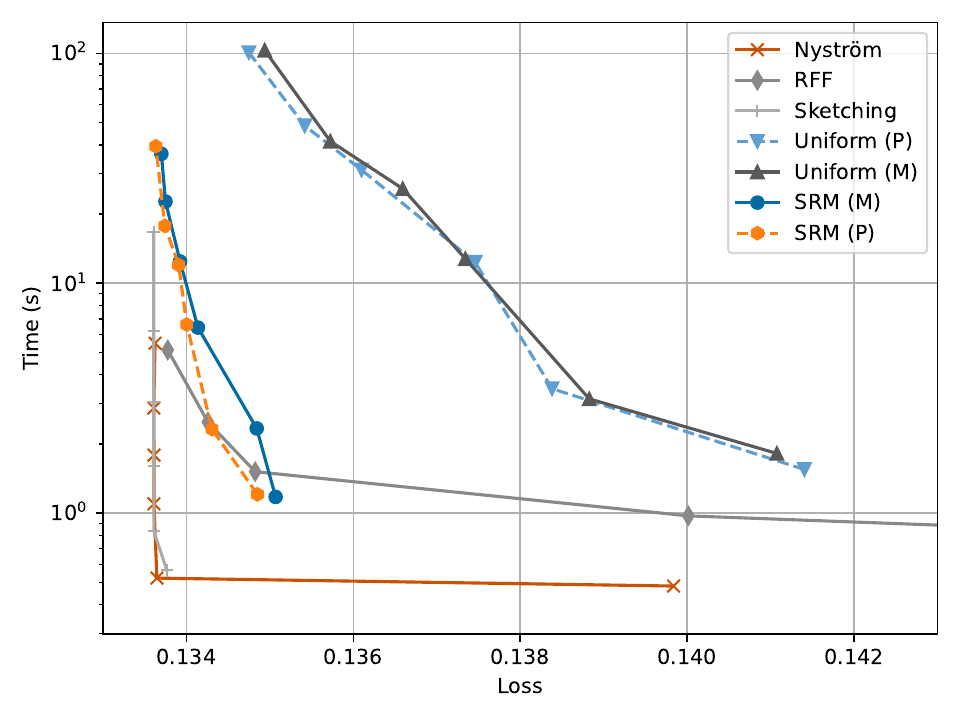}
				\caption{Synthetic regression.}
				\label{subfig:regression_space_budget_effort}
			\end{subfigure} \\
			\begin{subfigure}{0.48\linewidth}
				\includegraphics[width=\linewidth]{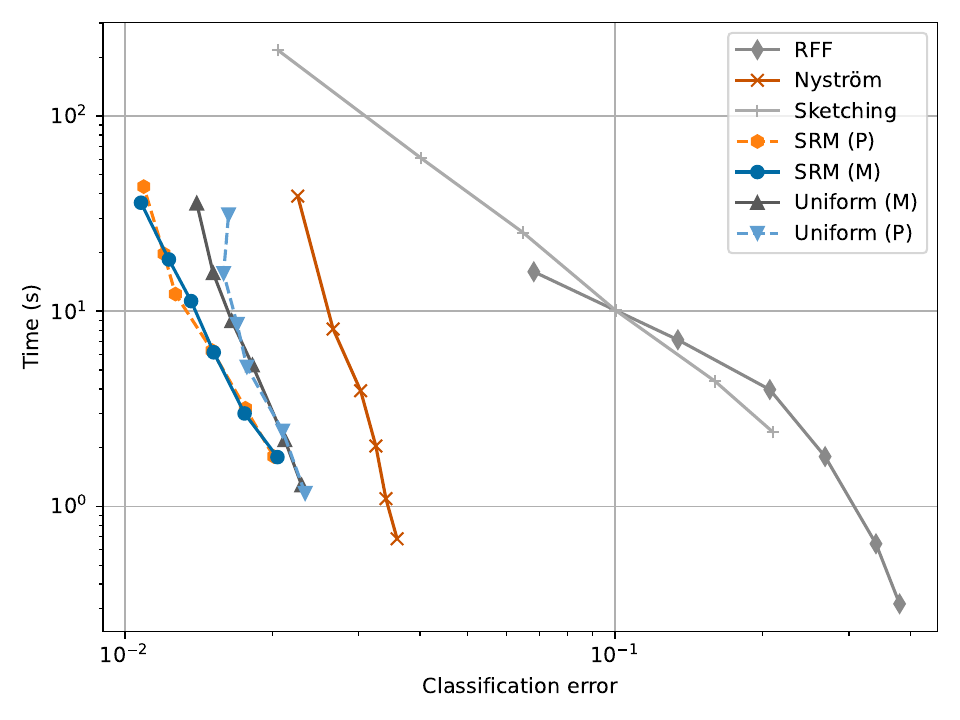}
				\caption{Covtype.}
				\label{subfig:covtype_space_budget_effort}
			\end{subfigure}
			\hfill
			\caption{Computational effort parameterized by the space budget (fixed \( n \)).}
			\label{fig:space_budget_effort}
		\end{figure}

		Regarding the synthetic tasks (\Cref{subfig:classification_space_budget_effort,subfig:regression_space_budget_effort}),
		Nyström, \rff and SKT appear more computationally efficient
		(curves are at the bottom and on the left).
		This may be explained by the low dimension of the input data.
		Nevertheless, subsampling and kernel approximations methods can be considered as complementary techniques for reducing the computational cost of kernel approaches.
		Thus, \Cref{fig:synthetic_space_budget_effort_mixed} presents numerical results when uniform subsampling and \srm estimators are computed with the Nyström approximation method.
		It results in increased performances for subsampling techniques, making \(\alpha\)-piloted \srm more appealing than Nyström, \rff and SKT alone regarding the computational effort.

		\begin{figure}[ht]
			\centering
			\begin{subfigure}{0.48\linewidth}
				\includegraphics[width=\linewidth]{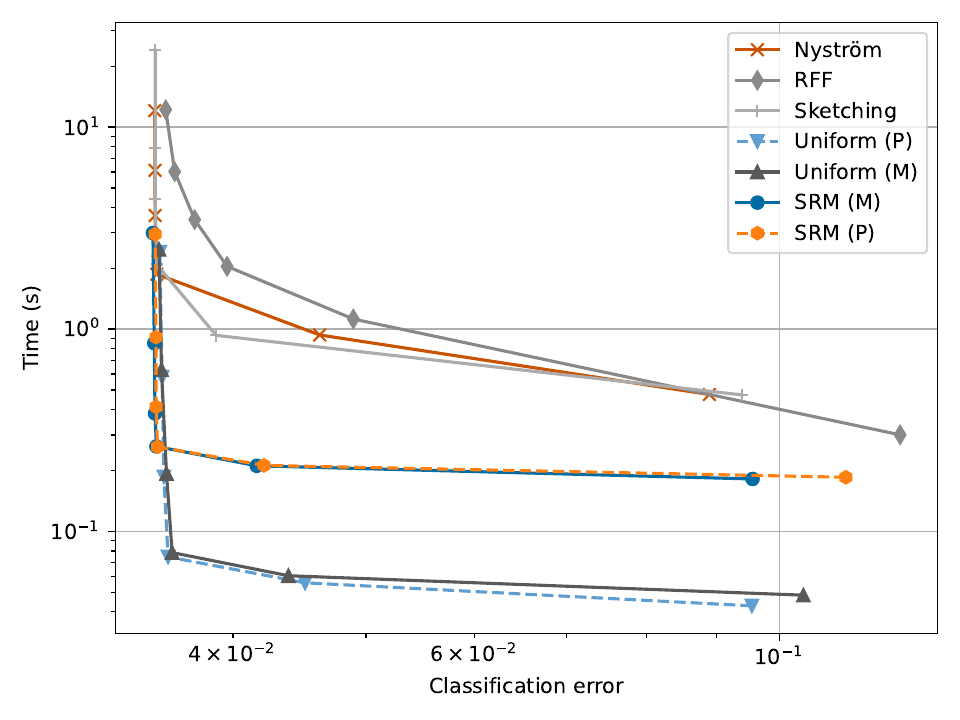}
				\caption{Synthetic classification.}
				\label{subfig:classification_space_budget_effort_mixed}
			\end{subfigure}
			\hfill
			\begin{subfigure}{0.48\linewidth}
				\includegraphics[width=\linewidth]{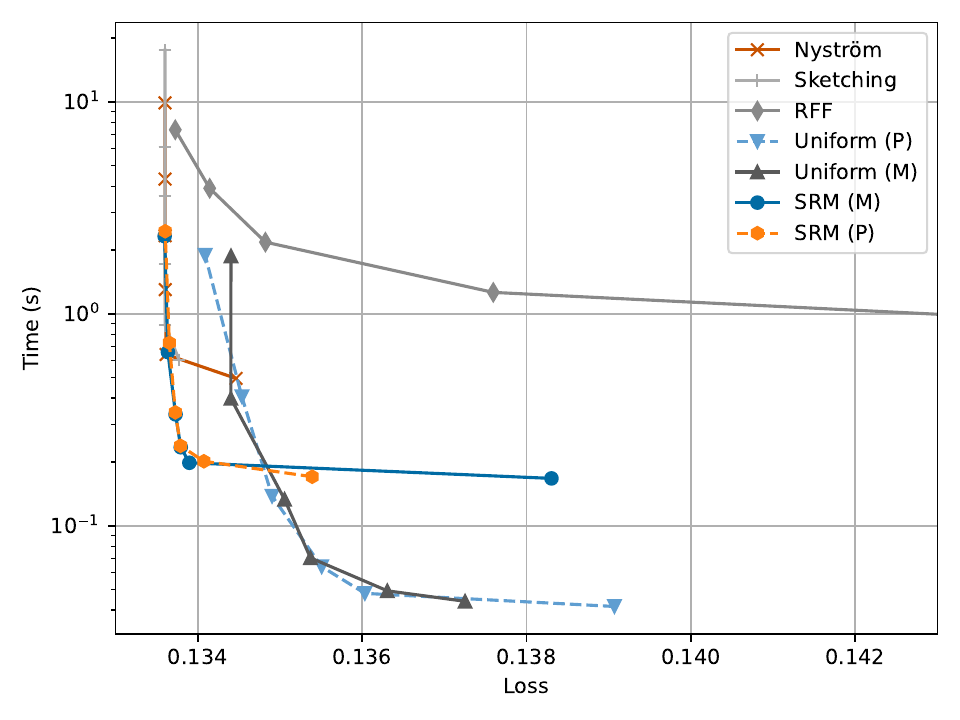}
				\caption{Synthetic regression.}
				\label{subfig:regression_space_budget_effort_mixed}
			\end{subfigure}
			\caption{Computational effort parameterized by the space budget (fixed \( n \)) for mixed Nyström-subsampling estimators.}
			\label{fig:synthetic_space_budget_effort_mixed}
		\end{figure}


\section{Conclusion and discussion}

This paper addresses the design of a reduced dataset by subsampling a large one, with the aim of computing a kernel-based estimator for classification or regression.
Based on asymptotic results, it is possible to randomly choose observations leading to a minimal trace for the limit covariance operator of the subsampling estimator.
This L-optimal subsampling scheme can then be estimated by plug-in, based on a pilot estimator, and smoothed in order to provide a robust practical method.
The latter is compatible with state-of-the-art kernel approximation techniques and can, for instance, be coupled with the Nyström method, leading to a very computationally efficient framework for learning in \rkhss.
Numerical experiments demonstrate the appeal of the method compared to uniform subsampling and kernel approximations, in particular for large initial datasets.

The work presented here opens several directions of research.
First, the theory is based on a regularized population risk and could, at the price of a substantial work, be extended to a non-regularized risk minimizer.
Second, the proposed approach is basically a two-step algorithm, learning first a pilot estimator (as well as estimating the subsampling probabilities and building the subsampled dataset) and then learning the final subsampling estimator.
One may conceive a multistage (or online) procedure, where the pilot estimator and the subsampling probabilities are sequentially adjusted in order to leverage more the random choice of consistent individuals.

\section*{Acknowledgments}
This project was partially funded by the alliance AI4IDF.
The authors are grateful to Jeffrey Näf for insightful discussions on asymptotics in \rkhss.

	\printbibliography
	\appendix
	\setcounter{equation}{0}
	\setcounter{figure}{0}
	\renewcommand{\theequation}{\thesection.\arabic{equation}}
	\renewcommand{\thefigure}{\thesection.\arabic{figure}}

\section{Proofs} \label{sec:proofs}

\subsection{Derivatives of \(L\)}
  \begin{lemma}\label{lem:leibniz}
    Under \Cref{hyp:kernel,hyp:loss,hyp:distribution}, \(L\) is twice differentiable at \(\hz\), with:
    \[
      \nabla \L(\hz) = \E{ \ell^\prime (\hz(X), Y) K_X } 
      \quad \text{and} \quad
      \nabla^2\L(\hz) = \E{ \ell^{\prime\prime}(\hz(X), Y) K_X \otimes K_X }. 
    \]
  \end{lemma}

  \begin{proof}
    Let $B_1 \defeq \kappa (\| \hz \|_\H + 1)$ and $\phi \defeq \phi_{B_1}$.
    By Cauchy-Schwarz and \Cref{hyp:kernel}, for all $h \in \H$ such that $\| h \|_\H \leq 1$,
    \begin{equation*}
    | (\hz + h)(X) | = | \dprod{\hz + h}{K_X} | \leq \| K_X \|_\H \| \hz + h \|_\H \leq B_1.
    \end{equation*}

    Now, by the inequality $(a + b)^4 \leq 8(a^4 + b^4)$ for all $(a, b) \in \R^2$,
    and \Cref{hyp:loss}:
    \begin{align*}
      \ell((\hz + h)(X), Y)^4 &\leq 8 \pa*{ \ell(0, Y)^4 + | \ell((\hz + h)(X), Y) - \ell(0, Y) |^4 } \\
      &\leq 8 \pa*{ \ell(0, Y)^4 + \phi(Y)^4 | (\hz + h)(X) |^4 } \\
      &\leq 8 \pa*{ \ell(0, Y)^4 + B_1^4\phi(Y)^4 }.
    \end{align*}
    Thus, by \Cref{hyp:distribution}, $\L$ is dominated in a neighborhood of $\hz$ by a four times integrable function.
    As the previous set of inequalities remains true for $\ell^{(k)}$ with $k \in \{ 0, 1, 2 \}$, it yields, by the dominated convergence theorem and Leibniz integral rule, that \(L\) is continuous and Fréchet differentiable at $\hz$ with:
    \[
      \nabla \L(\hz) = \E{ \ell^\prime (\hz(X), Y) K_X }.  
    \]
    In the same manner, it comes:
    \[
      \nabla^2\L(\hz) = \E{ \ell^{\prime\prime}(\hz(X), Y) K_X \otimes K_X }. \qedhere 
    \]
  \end{proof}

\subsection{Proof of \Cref{thm:erm_normality}} \label{app:proof_erm_normality}

\begin{theorem}[Consistency of $\hn$] \label{thm:erm_consistency}
	Under \Cref{hyp:kernel,hyp:loss,hyp:distribution}, if one furthermore assumes that $\ln \pconv \lz$, one has
	\begin{equation*}
		\| \hn - \hz \|_\H \pconv 0.
	\end{equation*}
\end{theorem}

Part of the following proof is adapted from Theorem 5.9 of \citep{christmann2008support}.

\begin{proof}
	As $\ell$ is convex, one has for all $(\bm x, y) \in \Z$,
	\begin{equation*}
		\ell^\prime(\hz(\bm x), y)(\hn(\bm x) - \hz(\bm x)) \leq \ell(\hn(\bm x), y) - \ell(\hz(\bm x), y),
	\end{equation*}
	and
	\begin{equation*}
		\dprod{\hn - \hz}{\nabla \Ln(\hz)} \leq \Ln(\hn) - \Ln(\hz),
	\end{equation*}
	by summing under the empirical distribution.
	Moreover,
	\begin{equation*}
		\ln \dprod{\hn - \hz}{\hz} + \frac{\ln}{2} \| \hn - \hz \|_\H^2 = \frac{\ln}{2} \| \hn \|_\H^2 - \frac{\ln}{2} \| \hz \|_\H^2.
	\end{equation*}
	Hence,
	\begin{align*}
		\frac{\ln}{2} \| \hn - \hz \|_\H^2 + \dprod{\hn - \hz}{\nabla \Ln(\hz) + \ln \hz} \leq \Rn(\hn) - \Rn(\hz) \leq 0,
	\end{align*}
	since $\hn$ is the minimizer of $\Rn$.
	However, $\lz \hz = -\nabla L(\hz)$, thus,
	\begin{equation*}
		\frac{\ln}{2} \| \hn - \hz \|_\H^2 \leq \dprod{\hn - \hz}{\frac{\ln}{\lz} \nabla \L(\hz) - \nabla \Ln(\hz)}.
	\end{equation*}
	Furthermore, by Cauchy-Schwarz,
	\begin{align*}
		\| \hn - \hz \|_\H &\leq \frac{2}{\ln} \| \nabla \Ln(\hz) - \frac{\ln}{\lz}\nabla \L(\hz) \|_\H \\
		&\leq \frac{2}{\ln} \| \nabla \Ln(\hz) - \nabla \L(\hz) \|_\H + \frac{2}{\lz\ln} \av*{ \ln - \lz } \| \nabla\L(\hz )\|_\H,
	\end{align*}
	as $\lz > 0$ and $\ln > 0$.
	Finally, $\ln \pconv \lz$ and $\nabla \Ln(\hz) \pconv \nabla \L(\hz)$, by the law of large numbers, which shows that $\| \hn - \hz \|_\H \pconv 0$.
\end{proof}

\begin{lemma} \label{lem:erm_normality_1}
	Under \Cref{hyp:kernel,hyp:loss,hyp:distribution} one has,
	\begin{equation*}
		\sqrt{n} (\nabla \Ln(\hz) - \nabla\L(\hz)) \dconv \GP(0, \Sigma),
	\end{equation*}
	where $\Sigma \defeq \E{ \ell^\prime(\hz(X), Y)^2 K_X \otimes K_X } - \nabla\L(\hz) \otimes \nabla\L(\hz)$.
\end{lemma}

\begin{proof}
	Let $f \in \H$.
	Define the real-valued random variable $S_n \defeq \dprod{\nabla \Ln(\hz) - \nabla\L(\hz)}{f}$, then,
	\begin{align*}
		S_n &= \frac{1}{n} \sum_{i = 1}^n \ell^\prime(\hz(X_i), Y_i) \dprod{K_{X_i}}{f} - \dprod{\nabla\L(\hz)}{f} \\
		&= \frac{1}{n} \sum_{i = 1}^n \ell^\prime(\hz(X_i), Y_i) f(X_i) - \E{ \ell^\prime (\hz(X), Y) f(X) } \\
		&= \frac{1}{n} \sum_{i = 1}^n \mathcal{L}^\prime(Z_i, \hz, f) - \E{ \mathcal{L}^\prime(Z, \hz, f) },
	\end{align*}
	where $\mathcal{L}^\prime(Z, h, f) \defeq \ell^\prime(h(X), Y)f(X)$.
	$S_n$ is a sum of $n$ i.i.d real random variables with mean,
	\begin{align*}
		\E{ \mathcal{L}^\prime (Z_1, \hz, f) - \E{ \mathcal{L}^\prime (Z, \hz, f) } } = \E{ \mathcal{L}^\prime (Z_1, \hz, f) } - \E{ \mathcal{L}^\prime (Z, \hz, f) } = 0,
	\end{align*}
	and variance,
	\begin{align*}
		\V{ \mathcal{L}^\prime(Z_1, \hz, f) } &= \E{ \pa*{ \mathcal{L}^\prime (Z_1, \hz, f) - \E{ \mathcal{L}^\prime (Z_1, \hz, f) } }^2 } \\
		&= \E{ \mathcal{L}^\prime(Z_1, \hz, f)^2 } - \E{ \mathcal{L}^\prime(Z_1, \hz, f) }^2 \\
		&= \dprod{\Sigma f}{f}.
	\end{align*}
	where $\Sigma \defeq \E{ \ell^\prime(\hz(X), Y)^2 K_X \otimes K_X } - \nabla\L(\hz) \otimes \nabla\L(\hz)$.
	Thus, by the one dimensional central limit theorem,
	\begin{equation*}
		\sqrt{n} \dprod{\nabla\Ln(\hz) - \nabla\L(\hz)}{f} \dconv \mathcal{N}(0, \dprod{\Sigma f}{f}).
	\end{equation*}
	There is left to show that $\sqrt{n} (\nabla\Ln(\hz) - \nabla\L(\hz))$ is asymptotically finite dimensional.
	Let $\varepsilon, \delta > 0$. Since $\E{ \| \ell^\prime(\hz(X), Y) K_X \|^2_\H } < \infty$ by \Cref{hyp:distribution}, there exists $J \in \N$ such that $\E{\sum_{j > J} \ell^\prime(\hz(X), Y)^2 \dprod{K_X}{e_j}^2 } < \delta \varepsilon$.
	Hence by Markov's inequality,
	\begin{equation*}
		\P{ \sum_{j > J} \dprod{\sqrt{n} (\nabla\Ln(\hz) - \nabla\L(\hz))}{e_j}^2 \geq \varepsilon } \leq \frac{n}{\varepsilon} \E{ \sum_{j > J} \dprod{\nabla\Ln(\hz) - \nabla\L(\hz)}{e_j}^2 }.
	\end{equation*}
	Moreover, considering that the random variables in $\nabla\Ln(\hz)$ are \iid\ with mean $\nabla\L(\hz)$ one has,
	\begin{align*}
		\E{ \sum_{j > J} \dprod{\nabla\Ln(\hz) - \nabla\L(\hz)}{e_j}^2 } = \frac{1}{n} \E{ \sum_{j > J} \ell^\prime(\hz(X), Y)^2 \dprod{K_{X}}{e_j}^2 } \\
		- \frac{1}{n}\sum_{j > J} \dprod{\nabla\L(\hz)}{e_j}^2.
	\end{align*}
	Hence, by definition of $J$,
	\begin{align*}
		\P{ \sum_{j > J} \dprod{\sqrt{n} (\nabla\Ln(\hz) - \nabla\L(\hz))}{e_j}^2 \geq \varepsilon } &\leq \frac{1}{\varepsilon} \E{ \sum_{j > J} \ell^\prime(\hz(X), Y)^2 \dprod{K_{X}}{e_j}^2 } \\
		&\leq \delta.
	\end{align*}
	Thus, $\sqrt{n} (\nabla\Ln(\hz) - \nabla\L(\hz))$ is asymptotically finite-dimensional, in the sense that for all $\varepsilon, \delta > 0$, there exists $J \in \N$ such that,
	\begin{equation*}
		\limsup_{n \in \N}\ \P{ \sum_{j > J} \dprod{\sqrt{n} (\nabla\Ln(\hz) - \nabla\L(\hz))}{e_j}^2 \geq \varepsilon } \leq \delta.
	\end{equation*}
	Applying Theorem 7.7.5 of \citep{hsing2015theoretical} yields
	\begin{equation*}
		\sqrt{n} (\nabla\Ln(\hz) - \nabla\L(\hz)) \dconv \GP(0, \Sigma). \qedhere
	\end{equation*}
\end{proof}

\begin{lemma} \label{lem:erm_normality_2}
	Under \Cref{hyp:kernel,hyp:loss,hyp:distribution,hyp:ln_speed},
	\begin{equation*}
		\sqrt{n} \int_0^1 \nabla^2\Rn(\bar h_{n,t}) (\hn - \hz) \,\d t - \sqrt{n} \nabla^2\risk (\hz) (\hn - \hz) \pconv 0,
	\end{equation*}
	with $\bar h_{n,t} = \hz + t(\hn - \hz)$ for all $t \in [0, 1]$.
\end{lemma}

\begin{proof}
	Firstly, in the proof of \Cref{thm:erm_consistency} we showed that,
	\begin{equation*}
		\| \hn - \hz \|_\H \leq \frac{2}{\ln} \| \nabla \Ln(\hz) - \nabla \L(\hz) \|_\H + \frac{2}{\lz\ln} \av*{ \ln - \lz } \| \nabla\L(\hz )\|_\H,
	\end{equation*}
	and from \Cref{lem:erm_normality_1} $\sqrt{n}(\nabla\Ln(\hz) - \nabla\L(\hz))$ converges in distribution to a Gaussian process.
	Thus, $\sqrt{n} \| \hn - \hz \|_\H = O_\Pr(1)$, by the fact that $\sqrt{n}(\ln - \lz)$ converges to 0 in probability.

	Then, let $h \in \H$, $t \in [0, 1]$ and denote $\bar h_{n,t} = \hz + t(\hn - \hz)$.
	By consistency of $\hn$, one has that $\| \bar h_{n,t} - \hz \| \pconv 0$.
	On the one hand, by the expression of $\nabla^2\Rn(\cdot)$ and triangle inequality,
	\begin{align}
		&\sqrt{n} | \dprod{\nabla^2\Rn(\bar h_{n,t}) (\hn - \hz)}{h} - \dprod{\nabla^2\Rn (\hz) (\hn - \hz)}{h} | \notag \\
		&\quad = \av*{ \frac{1}{\sqrt{n}} \sum_{i = 1}^n (\ell^{\prime\prime}(\bar h_{n,t}(X_i), Y_i) - \ell^{\prime\prime}(\hz(X_i), Y_i)) \dprod{K_{X_i} \otimes K_{X_i} (\hn - \hz)}{h} } \notag \\
		&\quad\leq \frac{1}{\sqrt{n}} \sum_{i = 1}^n \av*{ (\ell^{\prime\prime}(\bar h_{n,t}(X_i), Y_i) - \ell^{\prime\prime}(\hz(X_i), Y_i)) \dprod{K_{X_i}}{\hn - \hz} \dprod{K_{X_i}}{h} }. \label{eq:lem_erm_normality_2_part1}
	\end{align}
	Moreover, using \Cref{hyp:loss} for all $i \in \{ 1, \dots, n \}$,
	\begin{align*}
			\av{ \ell^{\prime\prime}(\bar h_{n,t}(X_i), Y_i) - \ell^{\prime\prime}(\hz(X_i), Y_i) } \mathbbm{1}_{\{ \| \bar h_{n,t} - \hz \|_\H < 1 \}} &\leq \phi(Y_i) \av{ \dprod{K_{X_i}}{\bar h_{n,t} - \hz} } \\
			&\leq \phi(Y_i) \av{ \dprod{K_{X_i}}{\hn - \hz} },
		\end{align*}
	as $t \in [0, 1]$.
	By Cauchy-Schwarz and \Cref{hyp:kernel}, one has
	\begin{equation*}
			\av{ \dprod{K_{X_i}}{\hn - \hz} } \leq \| \hn - \hz \|_\H \| K_{X_i} \|_\H \leq \kappa \| \hn - \hz \|_\H,
		\end{equation*}
	and $\av{ \dprod{K_{X_i}}{h} } \leq \kappa \| h \|_\H$.
	Thus, applying \Cref{lem:technical_pconv_2} with the random variables being \eqref{eq:lem_erm_normality_2_part1} and the events $E_n = \{ \| \bar h_n - \hz \|_\H \geq 1 \}$ yields
	\begin{align}
			&\sqrt{n} \int_0^1 | \dprod{\nabla^2\Rn(\bar h_{n,t}) (\hn - \hz)}{h} - \dprod{\nabla^2\Rn (\hz) (\hn - \hz)}{h} | \,\d t \notag \\
			&\quad \leq \sqrt{n} \kappa^3 \| h \|_\H \| \hn - \hz \|_\H^2 \frac{1}{n} \sum_{i = 1}^n \phi(Y_i) + o_\Pr(1) \pconv 0, \label{eq:lem_erm_normality_2_part2}
		\end{align}
	since $n^{-1} \sum_{i = 1}^n \phi(Y_i) = \Ex \phi(Y) + o_\Pr(1) = O_\Pr(1)$, $\sqrt{n} \| \hn - \hz \|_\H = O_\Pr(1)$ and $\| \hn - \hz \|_\H \pconv 0$ by \Cref{thm:erm_consistency}.
	On the other hand, since $\nabla^2\Rn(\hz)$ and $\nabla^2\risk(\hz)$ are self-adjoint operators, one has
	\begin{align*}
		&| \dprod{\nabla^2\Rn(\hz) (\hn - \hz)}{h} - \dprod{\nabla^2\risk(\hz) (\hn - \hz)}{h} | \\
		&\quad = | \dprod{\hn - \hz}{\nabla^2\Rn(\hz) h} - \dprod{\hn - \hz}{\nabla^2\risk(\hz) h} | \\
		&\quad = | \dprod{\hn - \hz}{\nabla^2\Rn(\hz) h - \nabla^2\risk(\hz) h} | \\
		&\quad \leq \| \hn - \hz \|_\H \| \nabla^2\Rn(\hz) h - \nabla^2\risk(\hz) h \|_\H,
	\end{align*}
	by Cauchy-Schwarz inequality.
	Furthermore
	\begin{align*}
		\nabla^2\Rn(\hz) h &= \frac{1}{n} \sum_{i = 1}^n \ell^{\prime\prime}(\hz(X_i), Y_i) h(X_i) K_{X_i} + \ln h \\
		&\quad  \pconv \E{ \ell^{\prime\prime}(\hz(X), Y) h(X) K_X } + \lz h = \nabla^2\risk(\hz) h,
	\end{align*}
	by the law of large numbers and that $\ln \pconv \lz$.
	Whence
	\begin{align}
		&\sqrt{n} \int_0^1 | \dprod{\nabla^2\Rn(\hz) (\hn - \hz)}{h} - \dprod{\nabla^2\risk(\hz) (\hn - \hz)}{h} | \,\d t \notag \\
		&\quad \leq \sqrt{n} \| \hn - \hz \|_\H \| \nabla^2\Rn(\hz) h - \nabla^2\risk(\hz) h \|_\H \pconv 0. \label{eq:lem_erm_normality_2_part3}
	\end{align}
	Finally, combining \eqref{eq:lem_erm_normality_2_part2} and \eqref{eq:lem_erm_normality_2_part3} yields the wanted result
	\begin{equation*}
		\sqrt{n} \int_0^1 \dprod{\nabla^2\Rn(\bar h_{n,t}) (\hn - \hz)}{h} \,\d t - \sqrt{n} \dprod{\nabla^2\risk (\hz) (\hn - \hz)}{h} \pconv 0. \qedhere
	\end{equation*}
\end{proof}

We can now prove the main theorem using the two above lemmas.

\begin{proof}[Proof of \Cref{thm:erm_normality}]
	The fundamental theorem of calculus yields
	\begin{equation*}
		0 = \nabla\Rn(\hn) = \nabla\Rn(\hz) + \int_{0}^{1} \nabla^2\Rn(\bar h_{n,t})(\hn - \hz)\,\d t,
	\end{equation*}
	with $\bar h_{n,t} = \hz + t (\hn - \hz)$ for $t \in [0, 1]$.
	By \Cref{lem:erm_normality_2}, one has
	\begin{equation*}
		- \sqrt{n} \nabla\Rn(\hz) = \sqrt{n} \nabla^2\risk(\hz) (\hn - \hz) + o_\Pr(1).
	\end{equation*}
	Notice that
	\begin{align*}
		\nabla\Rn(\hz) = (\nabla\Ln(\hz) - \nabla\L(\hz)) + (\ln - \lz) \hz.
	\end{align*}
	However, $\sqrt{n}(\ln - \lz) \pconv 0$ by \Cref{hyp:ln_speed}, hence
	\begin{equation*}
		\sqrt{n} \nabla\Rn(\hz) = \sqrt{n} (\nabla\Ln(\hz) - \nabla\L(\hz)) + \sqrt{n}(\ln - \lz) \hz \dconv \GP(0, \Sigma),
	\end{equation*}
	by Slutsky's lemma and \Cref{lem:erm_normality_1}, which concludes the proof.
\end{proof}

\subsection{Subsampling intermediate results}

The two following short lemma deal of general results for convergence in probability.

\begin{lemma} \label{lem:technical_pconv}
	Let $(X_n)_{n \in \N}$ be a sequence of real valued random variables.
	Suppose that for all $\varepsilon > 0$, there exists a sequence $(Y_n^\varepsilon)_{n \in \N}$ such that for all $n \in \N$,
	\begin{equation*}
		0 \leq X_n \leq Y_n^\varepsilon\ \text{a.s.} \qquad \text{and} \qquad Y_n^\varepsilon \pconv \varepsilon.
	\end{equation*}
	Then $(X_n)_{n \in \N}$ converges in probability to 0.
\end{lemma}

\begin{proof}
	Let $\delta > 0$ and $0 < \varepsilon < \delta$.
	\begin{align*}
		\P{ X_n \geq \delta } \leq \P{ Y_n^\varepsilon \geq \delta } &= \P{ Y_n^\varepsilon - \varepsilon \geq \delta - \varepsilon } \leq \P{ | Y_n^\varepsilon - \varepsilon | \geq \delta - \varepsilon } \conv 0. \qedhere
	\end{align*}
\end{proof}

\begin{lemma} \label{lem:technical_pconv_2}
	Let $(A_n)_{n \in \N}$ be a sequence of non negative random variables and  $(E_n)_{n \in \N}$ be a sequence of events such that $\P{E_n} \conv 0$.
	Then $A_n \mathbbm{1}_{E_n} \pconv 0$.
\end{lemma}

\begin{proof}
	Let $\varepsilon > 0$,
	\begin{equation*}
		\P{ A_n \mathbbm{1}_{E_n} > \varepsilon } \leq \P{ \mathbbm{1}_{E_n} = 1 } = \P{E_n} \conv 0. \qedhere
	\end{equation*}
\end{proof}

The hereby theorem and its following proof are a piece of Theorem 3.5 in \citep{hall1980martingale} revised with the conditional Lindeberg-Feller assumption instead of the usual one.

\begin{theorem}[Conditional Lindeberg-Feller] \label{thm:technical_conditional_dconv}
	Let $\{ \eta_{n,j} : n > 0, 1 \leq j \leq k_n \}$ be a real martingale difference array adapted to a filtration array $(\Fnj)_{n > 0, 0 \leq j \leq k_n}$, i.e.\ for all $n > 0$ and $1 \leq j \leq k_n$, $\eta_{n, j}$ is $\Fnj$-measurable and $\E{ \eta_{n, j} \middle| \Fnjm} = 0$.
	Let $\sigma^2 > 0$ such that,
	\begin{gather}
		\sum_{j = 1}^{k_n} \E{ \eta_{n,j}^2 \middle| \Fnjm } \pconv \sigma^2. \label{eq:thm_technical_conditional_dconv_1}
	\end{gather}
	Moreover assume that for all $\varepsilon > 0$,
	\begin{equation}
		\sum_{j = 1}^{k_n} \E{ \eta_{n,j}^2 \mathbbm{1}_{|\eta_{n,j}| > \varepsilon} \middle| \Fnjm } \pconv 0. \label{eq:thm_technical_conditional_dconv_2}
	\end{equation}
	Then for all $t \in \R$, on has,
	\begin{equation}
		\prod_{j = 1}^{k_n} \E{ e^{it \eta_{n,j}} \middle| \Fnjm } \pconv e^{-\frac{1}{2}t^2 \sigma^2}.
	\end{equation}
\end{theorem}

\begin{proof}
	Define the functions $A: \R \rightarrow \C$ and $B: \R \rightarrow \R$ such that for all $x \in \R$
	\begin{equation*}
		e^{it x} = 1 + itx - \frac{1}{2} x^2 + \frac{1}{2} x^2 A(x) \quad \text{and} \quad B(x) = \min(x/3, 2).
	\end{equation*}
	And note that, from \citep[][Lemma 1, p.512]{feller1991introduction}, $|A(x)| \leq B(|x|)$ for all $x \in \R$.
	Letting $\sigma_{n,j}^2 \defeq \E{ \eta_{n,j}^2 \middle| \Fnjm }$ and $\sigma_n^2 \defeq \sum_{j = 1}^{k_n} \sigma_{n,j}^2$, one has
	\begin{equation*}
		\E{ e^{it \eta_{n,j}} \middle| \Fnjm } = 1 - \frac{1}{2} t^2 \sigma_{n,j}^2 + \frac{1}{2} t^2 \E{ \eta_{n,j}^2 A(t \eta_{n,j}) \middle| \Fnjm }.
	\end{equation*}
	Now define $r_n$ by
	\begin{align}
		\log\pa*{\prod_{j = 1}^{k_n} \E{ e^{it \eta_{n,j}} \middle| \Fnjm}} &= \sum_{j = 1}^{k_n} \log\pa*{1 -\frac{1}{2} t^2 \sigma_{n,j}^2 + \frac{1}{2} t^2 \E{ \eta_{n,j}^2 A(t\eta_{n,j}) \middle| \Fnjm}} \notag \\
		&= -\frac{1}{2} t^2 \sigma_n^2 + \frac{1}{2} t^2 \sum_{j = 1}^{k_n} \E{ \eta_{n,j}^2 A(t\eta_{n,j}) \middle| \Fnjm} + r_n. \label{eq:conditional_lf}
	\end{align}
	As $\sigma_n \pconv \sigma^2$ by \eqref{eq:thm_technical_conditional_dconv_1}, it now suffices to show that $r_n \pconv 0$, then that the second term of \eqref{eq:conditional_lf} tends to 0 in probability as well.
	Regarding the first point, denote $z_{n,j} \defeq -\frac{1}{2} t^2 \sigma_{n,j}^2 + \frac{1}{2} t^2 \E{ \eta_{n,j}^2 A(t\eta_{n,j}) \middle| \Fnjm}$ for $n > 0$ and $1 \leq j \leq k_n$, one has that
	\begin{align*}
		\av*{ z_{n,j} } &\leq \frac{1}{2} t^2 \pa*{ \sigma_{n,j}^2 + \E{ \eta_{n,j}^2 |A(t \eta_{n,j})| \middle| \Fnjm}} \\
		&\leq \frac{1}{2} t^2 \pa*{ \sigma_{n,j}^2 + \E{ \eta_{n,j}^2 B(|t \eta_{n,j}|) \middle| \Fnjm}} \\
		&\leq \frac{3}{2} t^2 \sigma_{n,j}^2.
	\end{align*}
	Moreover, for $|z| < 1$ one has \citep[p. 72]{hall1980martingale}
	\begin{equation*}
		|\mathrm{log}(1 + z) - z | \leq \frac{|z|^2}{2(1 - |z|)},
	\end{equation*}
	thus
	\begin{align}
		| r_n | &\leq \sum_{j=1}^{k_n} | \log(1 + z_{n,j}) - z_{n,j} | \notag \\
		&\leq \frac{1}{2} \sum_{j=1}^{k_n} \frac{(\frac{3}{2} t^2 \sigma_{n,j}^2)^2}{1 - \frac{3}{2} t^2 \sigma_{n,j}^2} \mathbbm{1}_{\cb{| z_{n,j} | < 1}} +  \sum_{j=1}^{k_n} | \log(1 + z_{n,j}) - z_{n,j} | \mathbbm{1}_{\cb{| z_{n,j} | \geq 1}} \notag \\
		&\leq \sigma_n^2 \frac{9 t^2}{8}\frac{\tilde\sigma_n^2}{1 - \frac{3}{2}t^4 \tilde\sigma_n^2} + \mathbbm{1}_{\cb*{\tilde\sigma_n^2 \geq \frac{2}{3t^2}}} \sum_{j=1}^{k_n} | \log(1 + z_{n,j}) - z_{n,j} |, \label{eq:proof_thm_technical_conditional_dconv_1}
	\end{align}
	where $\tilde\sigma_n^2 \defeq \max_{1 \leq j \leq k_n} \sigma_{n,j}^2$.
	Let $\gamma > 0$, for each $1 \leq j \leq k_n$
	\begin{equation*}
		\sigma_{n,j}^2 \leq \gamma^2 + \E{ \eta_{n, j}^2 \mathbbm{1}_{\{ | \eta_{n, j} | > \gamma \}} \middle| \Fnj },
	\end{equation*}
	thus, with \eqref{eq:thm_technical_conditional_dconv_2}, one has
	\begin{align*}
		\tilde \sigma_n &\leq \gamma^2 + \max_{1 \leq j \leq k_n} \E{ \eta_{n, j}^2 \mathbbm{1}_{\{ | \eta_{n, j} | > \gamma \}} \middle| \Fnj } \\
		&\leq \gamma^2 + \sum_{j=1}^{k_n} \E{ \eta_{n, j}^2 \mathbbm{1}_{\{ | \eta_{n, j} | > \gamma \}} \middle| \Fnj } \pconv \gamma^2.
	\end{align*}
	Therefore $\tilde \sigma_n^2 \pconv 0$, by application of \Cref{lem:technical_pconv}.
	Thus, the first term of the right hand side of \eqref{eq:proof_thm_technical_conditional_dconv_1} tends to 0 in probability and the second term as well by \Cref{lem:technical_pconv_2}.

	Regarding the second point, let $\varepsilon > 0$ and observe that
	\begin{align*}
		\av*{ \sum_{j = 1}^{k_n} \E{ \eta_{n,j}^2 A(t\eta_{n,j}) \middle| \Fnjm} } &\leq \sum_{j = 1}^{k_n} \E{ \eta_{n,j}^2 |A(t\eta_{n,j})| \middle| \Fnjm} \\
		&\leq \sum_{j = 1}^{k_n} \E{ \eta_{n,j}^2 B(|t\eta_{n,j}|) (\mathbbm{1}_{\{ |t\eta_{n,j}| \leq \varepsilon \}} + \mathbbm{1}_{\{ |t\eta_{n,j}| > \varepsilon \}}) \middle| \Fnjm} \\
		&\leq \frac{\varepsilon}{3} \sigma_n^2 + 2 \sum_{j = 1}^{k_n} \E{ \eta_{n,j}^2 \mathbbm{1}_{\{ |t\eta_{n,j}| > \varepsilon \}} \middle| \dn} \pconv \frac{\varepsilon}{3} \sigma^2,
	\end{align*}
	by \eqref{eq:thm_technical_conditional_dconv_1} and \eqref{eq:thm_technical_conditional_dconv_2}.
	Thus, by application of \Cref{lem:technical_pconv}
	\begin{equation*}
		\sum_{j = 1}^{k_n} \E{ \eta_{n,j}^2 A(t\eta_{n,j}) \middle| \Fnjm} \pconv 0.
	\end{equation*}
	Whence, $\prod_{j=1}^{k_n} \E{e^{it\eta_{n,j}} \middle| \Fnjm} \pconv \exp\pa*{-\frac{1}{2}t^2 \sigma^2}$ for all $t \in \R$.
\end{proof}

\begin{lemma} \label{lem:technical_sum_asymp_finite_dim}
	Let $(A_n)_{n \in \N}$ and $(B_n)_{n \in \N}$ two asymptotically finite dimensional random variable sequences in $\H$ and $(a_n)_{n \in \N}$ a real non-zero sequence that converges to $a \in \R$.
	Then $(a_n A_n + B_n)_{n \in \N}$ is asymptotically finite dimensional as well.
\end{lemma}

\begin{proof}
	By definition, for all $\varepsilon, \delta > 0$, there exists $J_a, J_b \in \N$ such that
	\begin{equation*}
		\limsup_{n \in \N} \P{ \sum_{j > J_a} \dprod{A_n}{e_j}^2 \geq \frac{\varepsilon}{4 (a^2 + 1)} } < \frac{\delta}{2},
	\end{equation*}
	and
	\begin{equation*}
		\limsup_{n \in \N} \P{ \sum_{j > J_b} \dprod{B_n}{e_j}^2 \geq \frac{\varepsilon}{4} } < \frac{\delta}{2}.
	\end{equation*}
	Denote $J \defeq \max(J_a, J_b)$.
	Using the inequality $(a + b)^2 \leq 2 (a^2 + b^2)$, one has
	\begin{align*}
		&\P{ \sum_{j > J} \dprod{a_n A_n + B_n}{e_j}^2 \geq \varepsilon } \\
		&\qquad \leq \P{ \sum_{j > J} \dprod{a_n A_n}{e_j}^2 + \sum_{j > J} \dprod{B_n}{e_j}^2 \geq \frac{\varepsilon}{2} } \\
		&\qquad \leq \P{ \sum_{j > J} \dprod{A_n}{e_j}^2 \geq \frac{\varepsilon}{4 a_n^2} } + \P{\sum_{j > J} \dprod{B_n}{e_j}^2 \geq \frac{\varepsilon}{4} } \\
		&\qquad \leq \P{ \sum_{j > J_a} \dprod{A_n}{e_j}^2 \geq \frac{\varepsilon}{4 a_n^2} } + \P{\sum_{j > J_b} \dprod{B_n}{e_j}^2 \geq \frac{\varepsilon}{4} }.
	\end{align*}
	Moreover, there exists $N \in \N$ such that for all $n \geq N$, $a_n^2 \leq a^2 + 1$.
	Therefore
	\begin{equation*}
		\P{ \sum_{j > J_a} \dprod{A_n}{e_j}^2 \geq \frac{\varepsilon}{4 a_n^2} } \leq \P{ \sum_{j > J_a} \dprod{A_n}{e_j}^2 \geq \frac{\varepsilon}{4 (a^2 + 1)} }.
	\end{equation*}
	As a consequence, by subadditivity of the $\limsup$,
	\begin{equation*}
		\limsup_{n \in \N} \P{ \sum_{j > J} \dprod{a_n A_n + B_n}{e_j}^2 \geq \varepsilon } \leq \frac{\delta}{2} + \frac{\delta}{2} = \delta,
	\end{equation*}
	and $(a_n A_n + B_n)_{n \in \N}$ is asymptotically finite dimensional.
\end{proof}

Let us consider the two following subsampling strategies, where \(\gamma_{n, i}\) refers to the multiplicity (or the number of occurrences) of each point \(Z_i\) in the subsampled dataset (which may be null if \(Z_i\) is not picked).

\paragraph{Subsampling with replacement}
One draws independently \(s_n\) point indices:
\begin{equation*}
  I_{n,1}, \dots, I_{n,s_n} | \dn \overset{\iid}{\sim} \sum_{i = 1}^n \pi_{n, i} \delta_{\{ i \}},
\end{equation*}
where \(\delta_{\{ i \}}\) is the Dirac measure at \(i\).
Each point \(Z_i\) appears \(N_{n,i} \defeq \gamma_{n, i} = \sum_{j = 1}^{s_n} \mathbbm{1}_{\{ I_{n,j} = i \}}\) times in the subsampled dataset

\paragraph{Subsampling without replacement}
Assuming that $\pi_{n, i} \leq s_n^{-1}$ for all $n > 0$ and $i \in \{ 1, \dots, n \}$, ones draws independently for each point \(Z_i\)
\begin{equation*}
  \gamma_{n, i} | \dn \sim \text{Ber}(s_n \pi_{n, i}),
\end{equation*}
and decides that it should enter the subsampled dataset when the inclusion indicator verifies \(\delta_{n,i} \defeq \gamma_{n, i} = 1\).

The following proposition is a weak law of large numbers for Multinomial and Poisson subsampling.
It is written with an a general Hilbert space $E$ with the aim of applying it with either the reals $\R$ and the regular product or the \rkhs of study $\H$ and its scalar product $\dprod{\cdot}{\cdot}$.

\begin{proposition}[Subsampling LLN] \label{prop:wlln}
	Let $(E, \dprodE{\cdot}{\cdot})$ be a Hilbert space and $f$ a function from $\Z$ to $E$ such that $\E{ \| f(Z) \|_E^4 } < \infty$. Then, if $(\ppi_n)_{n > 0}$ verifies \Cref{hyp:ssp}, one has,
	\begin{equation*}
		\frac{1}{n} \sum_{j=1}^{s_n} \frac{f(Z_i)}{s_n \pi_{n,i}} \gamma_{n,i} \pconv \Ex f(Z),
	\end{equation*}
	whether $\gamma_{n,i} = N_{n,i}$, in the case of Multinomial subsampling, or $\gamma_{n,i} = \delta_{n,i}$, in the case of Poisson subsampling.

	Moreover, if it is assumed that $s_n = O(n)$ then,
	\begin{equation*}
		\no*{ \frac{1}{n} \sum_{j=1}^{s_n} \frac{f(Z_i)}{s_n \pi_{n,i}} \gamma_{n,i} - \Ex f(Z) }_E = O_\Pr(s_n^{-1/2}).
	\end{equation*}
\end{proposition}

By construction, in Poisson subsampling, $s_n \leq n$ hence $s_n = O(n)$.
In the case of (Multinomial) subsampling, the goal is to reduce the effective dataset size.
Hence one would largely want $s_n = O(n)$ and the second point of the above proposition always holds in this context.

\begin{proof}
	To simplify the notation, first define
	\begin{gather*}
		P f \defeq \Ex f(Z), \\
		P_n f \defeq \frac{1}{n} \sum_{i = 1}^n f(Z_i), \\
		P_{s_n}^* f \defeq \frac{1}{n} \sum_{j=1}^{s_n} \frac{f(Z_i)}{s_n \pi_{n,i}} \gamma_{n,i},
	\end{gather*}
	and recall that $P_n f \pconv P f$ by the usual weak law of large numbers.
	By the tower property one has
	\begin{equation*}
		\E{ P_{s_n}^* f } = \frac{1}{n} \sum_{i=1}^n \E{ \frac{f(Z_i)}{s_n \pi_{n,i}} \E{ \gamma_{n,i} \middle| \dn } } = \frac{1}{n} \sum_{i=1}^n \Ex f(Z_i) = Pf,
	\end{equation*}
	so $\E{ \| P_{s_n}^* f - Pf \|_E^2 } = \E{ \| P_{s_n}^* f \|_E^2 } - \| P f \|_E^2$.
	Let us now show that for all $i, k \in \{ 1, \dots, n \}$,
	\begin{equation} \label{eq:prop_wlln_R_part1}
		\E{ \gamma_{n,i} \gamma_{n,k} \middle| \dn } \leq s_n \pi_{n,i} \mathbbm{1}_{i = k} + s_n^2 \pi_{n, i} \pi_{n,k},
	\end{equation}
	for Multinomial and Poisson subsampling.
	On the one hand,
	\begin{align*}
		\E{ N_{n,i} N_{n,k} \middle| \dn } &= \E{ \sum_{j,m=1}^{s_n} \mathbbm{1}_{\{ I_{n,j = i} \}} \mathbbm{1}_{\{ I_{n,m = k} \}} \middle| \dn } \\
		&= \sum_{j,m=1}^{s_n} \P{ I_{n,j = i}, I_{n,m = k} \middle| \dn } \\
		&= \sum_{j=1}^{s_n} \P{ I_{n,j = i}, I_{n,j = k} \middle| \dn } + \sum_{j \neq m}^{s_n} \P{ I_{n,j = i} \middle| \dn } \P{ I_{n,m = k} \middle| \dn } \\
		&= s_n \pi_{n, i} \mathbbm{1}_{i = k} + s_n (s_n - 1) \pi_{n,i} \pi_{n,k} \\
		&\leq s_n \pi_{n, i} \mathbbm{1}_{i = k} + s_n^2 \pi_{n,i} \pi_{n,k}.
	\end{align*}
	On the other hand,
	\begin{align*}
		\E{ \delta_{n,i} \delta_{n,k} \middle| \dn } &= \E{ \delta_{n,i} \middle| \dn } \mathbbm{1}_{i = k} + \E{ \delta_{n,i} \middle| \dn } \E{ \delta_{n,k} \middle| \dn } \mathbbm{1}_{i \neq k} \\
		&= s_n \pi_{n,i} \mathbbm{1}_{i = k} + s_n^2 \pi_{n,i} \pi_{n,k} \mathbbm{1}_{i \neq k} \\
		&\leq s_n \pi_{n,i} \mathbbm{1}_{i = k} + s_n^2 \pi_{n,i} \pi_{n,k}.
	\end{align*}
	Hence
	\begin{align*}
		\E{ \| P_{s_n}^* f \|_E^2 \middle| \dn } & = \frac{1}{n^2 s_n^2} \sum_{i,k=1}^n \frac{\dprodE{f(Z_i)}{f(Z_k)}}{\pi_{n,i} \pi_{n,k}} \E{ \gamma_{n,i} \gamma_{n,k} \middle| \dn } \\
		&\leq \frac{1}{s_n n^2} \sum_{i=1}^n \frac{\| f(Z_i) \|_E^2}{\pi_{n,i}} + \frac{1}{n^2} \sum_{i, k = 1}^{n} \dprodE{f(Z_i)}{f(Z_k)} \\
		&= \frac{1}{s_n n^2} \sum_{i=1}^n \frac{\| f(Z_i) \|_E^2}{\pi_{n,i}} + \| P_n f \|_E^2.
	\end{align*}
	On the one hand, by Cauchy-Schwarz,
	\begin{align*}
		\E{ \frac{1}{n^2} \sum_{i = 1}^n \frac{\| f(Z_i) \|_E^2}{\pi_{n,i}} } &\leq \sqrt{\E{ \max_{1 \le i \le n}\cb*{\frac{1}{n\pi_{n, i}}}^2 }} \sqrt{\E{ \pa*{ \frac{1}{n} \sum_{i=1}^n \| f(Z_i) \|_E^2 }^2 }}.
	\end{align*}
	However,
	\begin{equation*}
		\E{ \max_{1 \le i \le n}\cb*{\frac{1}{n\pi_{n, i}}}^2 } = O(1),
	\end{equation*}
	and,
	\begin{align*}
		\E{ \pa*{ \frac{1}{n} \sum_{i=1}^n \| f(Z_i) \|_E^2 }^2 } &= \frac{1}{n^2} \sum_{i=1}^n \E{ \| f(Z_i) \|_E^4 } + \frac{1}{n^2} \sum_{i \neq j} \E{ \| f(Z_i) \|_E^2 \| f(Z_j) \|_E^2 } \\
		&= \frac{1}{n} \E{ \| f(Z_1) \|_E^4 } + \frac{n (n - 1)}{n^2} \E{ \| f(Z_1) \|_E^2 }^2 = O(1).
	\end{align*}
	Hence,
	\begin{equation*}
		\E{ \frac{1}{s_n n^2} \sum_{i = 1}^n \frac{\| f(Z_i) \|_E^2}{\pi_{n,i}} } = O(s_n^{-1}).
	\end{equation*}
	On the other hand, by independence of the observations $Z_1, \dots, Z_n$,
	\begin{align*}
		\E{ \| P_n f \|_E^2 } &= \frac{1}{n^2} \sum_{i,m=1}^n \E{ \dprodE{f(Z_i)}{f(Z_m)} } \\
		&= \frac{1}{n^2} \sum_{i = 1}^n \E{ \| f(Z_i) \|_E^2 } + \frac{1}{n^2}  \sum_{i \neq m}^n \E{ \dprodE{f(Z_i)}{f(Z_m)} } \\
		&= \frac{1}{n} \E{ \| f(Z) \|_E^2 } + \frac{n (n - 1)}{n^2} \| P f \|_E^2.
	\end{align*}
	Thus,
	\begin{align*}
		\E{ \| P_{s_n}^* f \|_E^2 } &\leq \E{ \frac{1}{s_n n^2} \sum_{i = 1}^n \frac{\| f(Z_i) \|_E^2}{\pi_{n,i}} } + \frac{1}{n} \E{ \| f(Z) \|_E^2 } + \frac{n - 1}{n} \| P f \|_E^2 \\
		&\leq \E{ \frac{1}{s_n n^2} \sum_{i = 1}^n \frac{\| f(Z_i) \|_E^2}{\pi_{n,i}} } + \frac{1}{n} \E{ \| f(Z) \|_E^2 } + \| P f \|_E^2,
	\end{align*}
	and
	\begin{align}
		\E{ \| P_{s_n}^* f \|_E^2 }  - \| Pf \|_E^2 &\leq \E{ \frac{1}{s_n n^2} \sum_{i = 1}^n \frac{\| f(Z_i) \|_E^2}{\pi_{n,i}} } + \frac{1}{n} \E{ \| f(Z) \|_E^2 } \notag \\
		&= O(s_n^{-1}) + O(n^{-1}) \conv 0 \label{eq:prop_wlln_R_part2}.
	\end{align}
	Finally, by Chebyshev's inequality, for all $\varepsilon > 0$,
	\begin{align*}
		\P{ \| P_{s_n}^* f - P f \|_E \geq \varepsilon } &\leq \frac{1}{\varepsilon^2} \E{ \| P_{s_n}^* f - Pf \|_E^2 } \\
		&= \frac{1}{\varepsilon^2} \pa*{ \E{ \| P_{s_n}^* f \|_E^2 } - \| Pf \|_E^2 } \conv 0,
	\end{align*}
	whence $P_{s_n}^* f \pconv Pf$.

	Assume now that $s_n = O(n)$.
	From \eqref{eq:prop_wlln_R_part2} one has
	\begin{align*}
		s_n (\E{ \| P_{s_n}^* f \|_E^2 } - \| Pf \|_E^2) &= s_n (O(s_n^{-1}) + O(n^{-1})) = O(1).
	\end{align*}
	Whence, there exists $M > 0$ and $N \in \N$ such that for $n \geq N$,
	\begin{align*}
		s_n \pa{ \E{ \| P_{s_n}^* f \|_E^2 } - \| Pf \|_E^2 } \leq M^2.
	\end{align*}
	Finally, by Chebyshev's inequality, for all $\varepsilon > 0$,
	\begin{align*}
		\P{ \sqrt{s_n} \| P_{s_n}^* f - P f \|_E \geq M \varepsilon^{-1/2} } &\leq \varepsilon\frac{s_n}{M^2} \E{ \| P_{s_n}^* f - Pf \|_E^2 } \\
		&= \varepsilon\frac{s_n}{M^2} \pa*{\E{ \| P_{s_n}^* f \|_E^2 } - \| Pf \|_E^2} \\
		&\leq \varepsilon,
	\end{align*}
	and therefore $\| P_{s_n}^* f - P f \|_E = O_\Pr(s_n^{-1/2})$.
\end{proof}

We hereby give the proof of \Cref{thm:srm_consistency_core} that is valid for Poisson and Multinomial subsampling via the law of large numbers proposition above.
Also notice that as $\| \hn - \hz \|_\H \pconv  0$ from \Cref{thm:erm_consistency}, triangle inequality implies that $\| \hsn - \hn \|_\H \pconv 0$ as well.

\begin{proof}[Proof of \Cref{thm:srm_consistency_core}]
	Following the same steps as for the proof of \Cref{thm:erm_consistency} but replacing $\hn$ by $\hsn$ and $\Ln$ by $\Lsn$, it comes that
	\begin{equation*}
		0 \leq \| \hsn - \hz \|_\H \leq \frac{2}{\lsn} \| \nabla \Lsn(\hz) - \nabla L(\hz) \|_\H + 2 \av*{ \frac{1}{\lsn} - \frac{1}{\lz} } \| \nabla\L(\hz) \|_\H.
	\end{equation*}
	with $\nabla \L(\hz) = \E{ \ell^\prime(\hz(X), Y) K_X }$ and
	\begin{equation*}
		\nabla\Lsn(\hz) = \frac{1}{n} \sum_{i = 1}^n \gamma_{n,i} \frac{\ell^\prime(\hz(X_i), Y_i)}{s_n \pi_{n, i}} K_{X_i}.
	\end{equation*}
	By \Cref{hyp:kernel,hyp:distribution}, $\ell^\prime(\hz(X), Y) K_X \in \H$ has a fourth order moment.
	Therefore, $\| \nabla\Lsn(h) - \nabla \L(h) \|_\H \pconv 0$ by application of \Cref{prop:wlln}, and eventually $\| \hsn - \hz \|_\H \pconv 0$ since $\lsn \pconv \lz$ by \Cref{hyp:lsn_speed}.
\end{proof}

The following lemma is valid for Poisson and Multinomial subsampling via the law of large numbers proposition above.

\begin{lemma} \label{lem:srm_normality_2}
	Under \Cref{hyp:kernel,hyp:loss,hyp:distribution}, \ref{hyp:lsn_speed} and  \ref{hyp:ssp},
	\begin{equation*}
		\sqrt{s_n} \int_0^1 \nabla^2\Rsn(\bar h_{n,t})(\hsn - \hz)\,\d t - \sqrt{s_n} \nabla^2\risk(\hz)(\hsn - \hz) \pconv 0,
	\end{equation*}
	with $\bar h_{n,t} = \hz + t(\hsn - \hz)$ for all $t \in [0, 1]$, either in the framework of Multinomial or Poisson subsampling.
\end{lemma}

\begin{proof}
	Firstly, from \Cref{thm:srm_consistency_core} we showed that
	\begin{equation*}
		\| \hsn - \hz \|_\H \leq \frac{2}{\lsn} \| \nabla \Lsn(\hz) - \nabla L(\hz) \|_\H + 2 \frac{1}{\lz\lsn} \av*{ \lsn - \lz } \| \nabla\L(\hz) \|_\H.
	\end{equation*}
	By the second part of the subsampling law of large numbers \Cref{prop:wlln}, one has
	\begin{equation*}
		\sqrt{s_n} \| \nabla \Lsn(\hz) - \nabla L(\hz) \|_\H = O_\Pr(1).
	\end{equation*}
	Moreover, $\sqrt{s_n} (\lsn - \lz) \pconv 0$.
	Therefore, $\sqrt{s_n} \| \hsn - \hz \|_\H = O_\Pr(1)$.

	Then, let $h \in \H$, $t \in [0, 1]$ and denote $\bar h_{n,t} = \hz + t(\hsn - \hz)$.
	By consistency of $\hsn$, one has that $\| \bar h_{n,t} - \hz \|_\H \pconv 0$.
	On the one hand, by the expression of $\nabla^2 \Rsn(\cdot)$ and triangle inequality one has
	\begin{align}
		&\sqrt{s_n}| \dprod{\nabla^2\Rsn(\bar h_{n,t}) (\hsn - \hz)}{h} - \dprod{\nabla^2\Rsn (\hz) (\hsn - \hz)}{h} | \notag \\
		&\leq \av*{ \frac{\sqrt{s_n}}{n} \sum_{i=1}^{n} \frac{\ell^{\prime\prime}(\bar h_{n,t}(X_i), Y_i) - \ell^{\prime\prime}(\hz(X_i), Y_i)}{s_n \pi_{n,i}} \gamma_{n,i} \dprod{K_{X_i} \otimes K_{X_i} (\hsn - \hz)}{h} } \notag \\
		&\leq \frac{\sqrt{s_n}}{n} \sum_{i=1}^{n} \av*{ \frac{\ell^{\prime\prime}(\bar h_{n,t}(X_i), Y_i) - \ell^{\prime\prime}(\hz(X_i), Y_i)}{s_n \pi_{n,i}} \gamma_{n,i} \dprod{\hsn - \hz}{K_{X_i}} \dprod{h}{K_{X_i}} } \label{eq:lem_srm_normality_2_p1}.
	\end{align}
	Moreover, using \Cref{hyp:loss} for all $i \in \{ 1, \dots, n \}$,
	\begin{align*}
		\av{ \ell^{\prime\prime}(\bar h_{n,t}(X_i), Y_i) - \ell^{\prime\prime}(\hz(X_i), Y_i) } \mathbbm{1}_{\{ \| \bar h_{n,t} - \hz \|_\H < 1 \}} &\leq \phi(Y_i) \av{ \dprod{K_{X_i}}{ \bar h_{n,t} - \hz} } \\
		&\leq \phi(Y_i) \av{ \dprod{K_{X_i}}{ \hsn - \hz} },
	\end{align*}
	since $t \in [0, 1]$.
	By Cauchy-Schwarz one has
	\begin{equation*}
		\av{ \dprod{K_{X_i}}{\hsn - \hz} } \leq \| K_{X_i} \|_\H \| \hsn - \hz \|_\H \leq \kappa \| \hsn - \hz \|_\H,
	\end{equation*}
	and $\av{ \dprod{K_{X_i}}{h} } \leq \kappa \| h \|_\H$.
	Furthermore, $\sqrt{s_n} \| \hsn - \hz \|_\H \pconv 0$, $\| \hsn - \hz \|_\H \pconv 0$, and by \Cref{hyp:distribution} $\phi(Z)$ has a fourth order moment hence
	\begin{equation*}
		\frac{1}{n} \sum_{i=1}^{n} \frac{\phi(Y_i)}{s_n \pi_{n,i}} \gamma_{n,i} \pconv \Ex \phi(Y),
	\end{equation*}
	by application of \Cref{prop:wlln}.
	Thus, applying \Cref{lem:technical_pconv_2} with the random variables being \eqref{eq:lem_srm_normality_2_p1} and the events $E_n = \{ \| \bar h_{n,t} - \hz \|_\H \geq 1 \}$ one has
	\begin{align}
		&\sqrt{s_n} \int_0^1 | \dprod{\nabla^2\Rsn(\bar h_{n,t}) (\hsn - \hz)}{h} - \dprod{\nabla^2\Rsn (\hz) (\hsn - \hz)}{h} | \,\d t \notag \\
		&\quad \leq \sqrt{s_n} \kappa^3 \| h \|_\H \| \hsn - \hz \|_\H^2 \frac{1}{n} \sum_{i=1}^{n} \frac{\phi(Y_i)}{s_n \pi_{n,j}} \gamma_{n,i} + o_\Pr(1) \pconv 0. \label{eq:lem_srm_normality_2_part1}
	\end{align}
	On the other hand, since $\nabla^2\Rsn(\hz)$ and $\nabla^2\risk(\hz)$ are self-adjoint operators, one has
	\begin{align*}
		&| \dprod{\nabla^2\Rsn(\hz) (\hsn - \hz)}{h} - \dprod{\nabla^2\risk(\hz) (\hsn - \hz)}{h} | \\
		&\quad = | \dprod{\hsn - \hz}{\nabla^2\Rsn(\hz) h} - \dprod{\hsn - \hz}{\nabla^2\risk(\hz) h} | \\
		&\quad = | \dprod{\hsn - \hz}{\nabla^2\Rsn(\hz) h - \nabla^2\risk(\hz) h} | \\
		&\quad \leq \| \hsn - \hz \|_\H \| \nabla^2\Rsn(\hz) h - \nabla^2\risk(\hz) h \|_\H,
	\end{align*}
	by Cauchy-Schwarz inequality.
	Furthermore, $\ell^{\prime\prime}(\hz(X), Y) h(X) K_{X} \in \H$ has a fourth order moment by \Cref{hyp:kernel,hyp:distribution} hence
	\begin{equation*}
		\| \nabla^2\Lsn(\hz) h- \nabla^2\L(\hz) h \|_\H \pconv 0,
	\end{equation*}
	by application of \Cref{prop:wlln}.
	Moreover
	\begin{align*}
		\| \nabla^2\Rsn(\hz) h - \nabla^2\risk(\hz) h \|_\H &\leq \| \nabla^2\Lsn(\hz) h- \nabla^2\L(\hz) h \|_\H + \av{ \lsn - \lz } \| h \|_\H.
	\end{align*}
	Hence
	\begin{equation*}
		\| \nabla^2\Rsn(\hz) h - \nabla^2\risk(\hz) h \|_\H \pconv 0,
	\end{equation*}
	since $\lsn \pconv \lz$.
	Therefore
	\begin{align}
		& \sqrt{s_n} \int_0^1 | \dprod{\nabla^2\Rsn(\hz) (\hsn - \hz)}{h} - \dprod{\nabla^2\risk(\hz) (\hsn - \hz)}{h} | \,\d t \notag \\
		&\quad \leq \sqrt{s_n} \| \hsn - \hz \| \| \nabla^2\Rsn(\hz) h - \nabla^2\risk(\hz) h \|_\H \pconv 0. \label{eq:lem_srm_normality_2_part2}
	\end{align}
	Finally combining \eqref{eq:lem_srm_normality_2_part1} and \eqref{eq:lem_srm_normality_2_part2} yields the wanted result
	\begin{equation*}
		\sqrt{s_n} \int_0^1 \dprod{\nabla^2\Rsn(\bar h_{n,t})(\hsn - \hz)}{h}\,\d t - \sqrt{s_n} \dprod{\nabla^2\risk(\hz)(\hsn - \hz)}{h} \pconv 0.\qedhere
	\end{equation*}
\end{proof}

The following lemma can be applied for Multinomial or Poisson subsampling.

\begin{lemma} \label{lem:srm_normality_1_bis}
	Under \Cref{hyp:kernel,hyp:loss,hyp:distribution,hyp:ssp}, the quantity  $\sqrt{s_n}(\nabla\Lsn(\hz) - \nabla\Ln(\hz))$ is asymptotically finite dimensional, i.e.\ for all $\varepsilon, \delta > 0$, there exists $J \in \N$ such that
	\begin{equation*}
		\limsup_{n \in \N}\ \P{ \sum_{j > J} \dprod{\sqrt{s_n}(\nabla\Lsn - \nabla\Ln)(\hz)}{e_j}^2 \geq \varepsilon } < \delta,
	\end{equation*}
	whether considering Poisson or Multinomial subsampling.
\end{lemma}

\begin{proof}
	Let $\varepsilon, \delta > 0$ and $J \in \N$.
	By Markov's inequality
	\begin{equation*}
		\P{ \sum_{j > J} \dprod{\sqrt{s_n}(\nabla\Lsn - \nabla\Ln)(\hz)}{e_j}^2 \geq \varepsilon } \leq \frac{s_n}{\varepsilon} \E{ \sum_{j > J} \dprod{(\nabla\Lsn - \nabla\Ln)(\hz)}{e_j}^2 }.
	\end{equation*}
	Expanding the right hand side together with the tower property yields the quantity
	\begin{align*}
		&\frac{s_n}{n^2 \varepsilon} \E{ \sum_{j > J} \sum_{i,k=1}^n \frac{\mathcal{L}^\prime(Z_i, \hz, e_j) \mathcal{L}^\prime(Z_k, \hz, e_j)}{s_n^2 \pi_{n,i} \pi_{n,k}} (\gamma_{n,i} - s_n \pi_{n,i}) (\gamma_{n,k} - s_n \pi_{n,k}) } \\
		&\quad = \frac{1}{n^2 \varepsilon} \E{ \sum_{j > J} \sum_{i,k=1}^n \frac{\mathcal{L}^\prime(Z_i, \hz, e_j) \mathcal{L}^\prime(Z_k, \hz, e_j)}{s_n \pi_{n,i} \pi_{n,k}} \E{ (\gamma_{n,i} - s_n \pi_{n,i}) (\gamma_{n,k} - s_n \pi_{n,k}) \middle| \dn } }.
	\end{align*}
	However, from \eqref{eq:prop_wlln_R_part1}, for all $i, k \in \{ 1, \dots, n \}$ one has
	\begin{align*}
		\E{ (\gamma_{n,i} - s_n \pi_{n,i}) (\gamma_{n,k} - s_n \pi_{n,k}) \middle| \dn } &= \E{ \gamma_{n,i} \gamma_{n,k} \middle| \dn } - s_n \pi_{n,i} \E{ \gamma_{n,k} \middle| \dn } \\
		&\qquad - s_n \pi_{n,k} \E{ \gamma_{n,i} \middle| \dn } + s_n^2 \pi_{n,i} \pi_{n,k} \\
		&= \E{ \gamma_{n,i} \gamma_{n,k} \middle| \dn } - s_n^2 \pi_{n,i} \pi_{n,k} \\
		&\leq s_n \pi_{n,i} \mathbbm{1}_{i = k}
	\end{align*}
	Thus
	\begin{align*}
		&\frac{s_n}{\varepsilon} \E{ \sum_{j > J} \dprod{(\nabla\Lsn - \nabla\Ln)(\hz)}{e_j}^2 } \\
		&\quad \leq \frac{1}{\varepsilon} \E{ \sum_{j > J} \frac{1}{n^2} \sum_{i = 1}^n \frac{\mathcal{L}^\prime(Z_i, \hz, e_j)^2}{\pi_{n,i}} } \\
		&\quad \leq \frac{1}{n \varepsilon} \sum_{i = 1}^n \E{ \max_{k \in [n]} \cb*{ \frac{1}{n \pi_{n, k}} } \sum_{j > J} \mathcal{L}^\prime(Z_i, \hz, e_j)^2 } \\
		&\quad = \frac{1}{\varepsilon} \E{ \max_{k \in [n]} \cb*{ \frac{1}{n \pi_{n, k}} } \sum_{j > J} \mathcal{L}^\prime(Z, \hz, e_j)^2 } \\
		&\quad \leq \frac{1}{\varepsilon} \sqrt{\E{ \max_{ \in [n]} \cb*{ \frac{1}{n \pi_{n, i}} }^2}} \sqrt{\E{ \pa*{ \sum_{j > J} \ell^\prime(\hz(X), Y)^2 \dprod{K_X}{e_j}^2 }^2 }}.
	\end{align*}
	However, $\| \ell^\prime(\hz(X), Y) K_X \|_\H^2 < \infty$ hence
	\begin{equation*}
		Z^{(J)} \defeq \sum_{j > J} \ell^\prime(\hz(X), Y)^2 \dprod{K_X}{e_j}^2 \overset{\Pr}{\underset{J \rightarrow \infty}{\longrightarrow}} 0,
	\end{equation*}
	as the rest of a converging series.
	Furthermore, $\| \ell^\prime(\hz(X), Y) K_X \|_\H^2$ dominates the series $(Z^{(J)})_{J \geq 0}$ and has a second order moment by \Cref{hyp:distribution} whence, by the dominated convergence theorem in $L^2$, one has
	\begin{equation*}
		\E{ \pa*{ \sum_{j > J} \ell^\prime(\hz(X), Y)^2 \dprod{K_X}{e_j}^2 }^2 } \underset{J \rightarrow \infty}{\longrightarrow} 0.
	\end{equation*}
	Moreover, it is assumed from \Cref{hyp:ssp} that there exists $N \in \N$ and $C > 0$ such that for all $n \geq N$,
	\begin{equation*}
		\E{ \max_{1 \le i \le n} \cb*{ \frac{1}{n \pi_{n, i}} }^2} \leq C^2.
	\end{equation*}
	Hence for $J$ large enough such that
	\begin{equation*}
		\sqrt{\E{ \pa*{ \sum_{j > J} \ell^\prime(\hz(X), Y)^2 \dprod{K_X}{e_j}^2 }^2 }} < \delta \frac{\varepsilon}{C},
	\end{equation*}
	one has
	\begin{equation*}
		\limsup_{n \in \N}\ \P{ \sum_{j > J} \dprod{\sqrt{s_n}(\nabla\Lsn - \nabla\Ln)(\hz)}{e_j}^2 \geq \varepsilon } < \delta.
	\end{equation*}
	Thus $\sqrt{s_n}(\nabla\Lsn(\hz) - \nabla\Ln(\hz))$ is asymptotically finite dimensional.
\end{proof}

The following lemma is the main lemma for the proofs of \Cref{thm:srm_normality_R,thm:srm_normality_P,thm:srm_alpha_pilot_normality_R,thm:srm_alpha_pilot_normality_P}.

\begin{lemma} \label{lem:srm_normality_final}
	Let $(H_n)_{n > 0}$ be a sequence of $\H$ valued random variables and $(\Fn)_{n > 0}$ a sequence of $\sigma$-algebras.
	Assume that $H_n$ is asymptotically finite-dimensional, $\sigma(\dn) \subset \Fn$ and for all $f \in \H$,
	\begin{equation}
		\E{ e^{it \dprod{H_n}{f}} \middle| \Fn } \pconv e^{-\frac{1}{2}\dprod{\Sigma^* f}{f}}, \label{eq:lem_srm_normality_final_part1}
	\end{equation}
	for some covariance operator $\Sigma^*$.
	Then
	\begin{equation*}
		H_n + \sqrt{\frac{s_n}{n}} \sqrt{n} (\nabla\Ln(\hz) - \nabla\L(\hz)) \dconv \GP(0, \Sigma^* + c\Sigma).
	\end{equation*}
\end{lemma}

\begin{proof}
	Let $f \in \H$ and define
	\begin{gather*}
		S_n^* \defeq \dprod{H_n}{f}, \\
		S_n \defeq \sqrt{n} \dprod{\nabla\Ln(\hz) - \nabla\L(\hz)}{f}.
	\end{gather*}
	Let $t \in \R$ and consider the convergence of the characteristic function, i.e.\
	\begin{equation*}
		\av*{ \Ex e^{it(S_n^* + \sqrt\frac{s_n}{n} S_n)} - e^{-\frac{1}{2}t^2 \dprod{(\Sigma^* + c\Sigma) f}{f}} }.
	\end{equation*}
	First adding and subtracting $e^{it\sqrt\frac{s_n}{n} S_n} e^{-\frac{1}{2}t^2\dprod{\Sigma^* f}{f}}$ and using the triangle inequality one gets that the previous display is upper bounded by
	\begin{equation*}
		\av*{\E{ e^{it \sqrt\frac{s_n}{n} S_n} (e^{it S_n^*} - e^{-\frac{1}{2}t^2 \dprod{\Sigma^* f}{f}}) }} + e^{-\frac{1}{2}t^2 \dprod{\Sigma^* f}{f}} \av*{\E{ e^{it \sqrt\frac{s_n}{n} S_n} } - e^{-\frac{1}{2}t^2 c\dprod{\Sigma f}{f}}}.
	\end{equation*}
	By \Cref{lem:erm_normality_1}, $S_n \dconv \mathcal{N}(0, \dprod{\Sigma f}{f})$ and since $s_n / n \conv c$, Slutsky's theorem yields that $\sqrt{\frac{s_n}{n}} S_n \dconv \mathcal{N}(0, c\dprod{\Sigma f}{f})$.
	Thus
	\begin{equation*}
		\av*{ \E{ e^{it \sqrt\frac{s_n}{n} S_n} } - e^{-\frac{1}{2}t^2 c \dprod{\Sigma f}{f}}} \conv 0.
	\end{equation*}
	Moreover, $S_n$ is $\dn$-measurable thus
	\begin{align*}
		\av*{\E{ e^{it\sqrt\frac{s_n}{n} S_n} (e^{it S_n^*} - e^{-\frac{1}{2}t^2 \dprod{\Sigma^* f}{f}}) }} &= \av*{ \E{ e^{it \sqrt\frac{s_n}{n} S_n} \pa*{ \E{ e^{it S_n^*} \middle| \Fn } - e^{-\frac{1}{2}t^2 \dprod{\Sigma^* f}{f}} }}} \\
		&\leq \E{\av*{ \E{ e^{it S_n^*} \middle| \Fn } - e^{-\frac{1}{2}t^2 \dprod{\Sigma^* f}{f}} }}.
	\end{align*}
	However, $\av*{ \E{ e^{it S_n^*} \middle| \dn} } \leq 1$.
	Thus, by \eqref{eq:lem_srm_normality_final_part1} and the dominated convergence theorem
	\begin{equation*}
		\E{\av*{ \E{ e^{it S_n^*} \middle| \Fn } - e^{-\frac{1}{2}t^2 \dprod{\Sigma^* f}{f}} }} \conv 0.
	\end{equation*}
	As a consequence,
	\begin{equation*}
		S_n^* + \sqrt\frac{s_n}{n} S_n \dconv \mathcal{N}(0, \dprod{(\Sigma^* + c\Sigma) f}{f}).
	\end{equation*}
	Moreover, by \Cref{lem:erm_normality_1} $\sqrt{n}(\nabla\Ln(\hz) - \nabla\L(\hz))$ is asymptotically finite dimensional and so is $H_n$, by assumption.
	Then, by \Cref{lem:technical_sum_asymp_finite_dim} the sum of two asymptotically finite dimensional sequences of random variables in $\H$ is itself asymptotically finite dimensional.
	Therefore
	\begin{equation*}
		H_n + \sqrt{n}(\nabla\Ln(\hz) - \nabla\L(\hz)) \dconv \GP(0, \Sigma^* + c \Sigma). \qedhere
	\end{equation*}
\end{proof}

\subsection{Proof of \Cref{thm:srm_normality_R}}

\begin{lemma} \label{lem:srm_normality_1}
	Under \Cref{hyp:kernel,hyp:loss,hyp:distribution,hyp:ssp}, one has for all $t \in \R$ and $f \in \H$,
	\begin{equation*}
		\E{\exp\pa*{it S_{s_n}^*} \middle| \dn} \pconv \exp\pa*{-\frac{1}{2} t^2 \dprod{\Sigma^*_\mathrm{M}(\psi) f}{f}},
	\end{equation*}
	where
	\begin{equation*}
		S_{s_n}^* \defeq \sqrt{s_n} \dprod{(\nabla\Lsn(\hz) - \nabla\Ln(\hz))}{f},
	\end{equation*}
	and
	\begin{equation*}
		\Sigma^*_\mathrm{M}(\psi) \defeq \E{\frac{\E {\psi (Z)}}{\psi (Z)} \ell^\prime(\hz(X), Y)^2 K_X \otimes K_X} - \nabla\L(\hz) \otimes \nabla\L(\hz).
	\end{equation*}
\end{lemma}

\begin{proof}
	The goal is to apply \Cref{thm:technical_conditional_dconv}.
	Recall that
	\begin{align*}
		\nabla\Lsn(\hz) &= \frac{1}{n} \sum_{i = 1}^n \frac{\ell^\prime(\hz(X_i), Y_i)}{s_n \pi_{n, i}} N_{n,i} K_{X_i} \\
		&= \frac{1}{s_n} \sum_{j = 1}^{s_n} \sum_{i = 1}^n \frac{\ell^\prime (\hz(X_i), Y_i)}{n\pi_{n,i}} K_{X_i} \mathbbm{1}_{\{ I_{n,j} = i \}},
	\end{align*}
	Let $f \in \H$ and define $S_{s_n}^* \defeq \dprod{\sqrt{s_n} (\nabla\Lsn - \nabla\Ln)(\hz)}{f}$ a real-valued random variable and $\mathcal{L}^\prime(Z, h, f) \defeq \ell^\prime(h(X), Y)f(X)$.
	Then
	\begin{align*}
		S_{s_n}^* &= \sum_{j = 1}^{s_n} \frac{1}{\sqrt{s_n}} \pa*{ \frac{1}{n} \sum_{i = 1}^n \frac{\mathcal{L}^\prime(Z_i, \hz, f)}{\pi_{n,i}} \mathbbm{1}_{\{ I_{n,j} = i \}} - \dprod{\nabla\Ln(\hz)}{f} } \\
		&= \sum_{j = 1}^{s_n} \eta_{n,j},
	\end{align*}
	where
	\begin{equation*}
		\eta_{n,j} \defeq \frac{1}{\sqrt{s_n}} \pa*{ \frac{1}{n} \sum_{i = 1}^n \frac{\mathcal{L}^\prime(Z_i, \hz, f)}{\pi_{n,i}} \mathbbm{1}_{\{ I_{n,j} = i \}} - \dprod{\nabla\Ln(\hz)}{f} },
	\end{equation*}
	for all $n \in \N$ and $j \in [s_n]$.
	Moreover, define the filtrations $(\Fnj)_{n \in \N, 1 \leq j \leq s_n}$ such that $\Fnj$ is the $\sigma$-algebra generated by $\dn$ and $I_{n,1}, \dots, I_{n,j}$ for $1 \leq j \leq s_n$ and $n \in \N$.
	Notice that for all $n \in \N$ and $j \in \{ 1, \dots, s_n \}$ the conditional distributions of $\eta_{n,j}$ with respect to $\dn$ and $\Fnjm$ are the same.
	Remark as well that, conditionally to $\dn$, $S_{s_n}^*$ is a sum of \iid\ random variables such that $\E{ \eta_{n, j} \middle| \dn } = 0$ since $\P{ I_{n,j} = i \middle| \dn } = \pi_{n, i}$.
	Hence $\{ \eta_{n,j} : n \in \N, 1 \leq j \leq s_n \}$ is a martingale difference array for the filtration $(\Fnj)_{n \in \N, 1 \leq j \leq s_n}$.
	Moreover, its conditional variance reads
	\begin{align*}
		\sum_{j = 1}^{s_n} \E{ \eta_{n, j}^2 \middle| \Fnjm} &= s_n \E{ \eta_{n, 1}^2 \middle| \dn } \\
		&= \E{ \pa*{ \frac{1}{n} \sum_{i = 1}^n \frac{\mathcal{L}^\prime(Z_i, \hz, f)}{\pi_{n,i}} \mathbbm{1}_{\{ I_{n,1} = i \}} - \dprod{\nabla\Ln(\hz)}{f} }^2 \middle| \dn }.
	\end{align*}
	Since $\E{ \frac{1}{n} \sum_{i = 1}^n \frac{\mathcal{L}^\prime(Z_i, \hz, f)}{\pi_{n,i}} \mathbbm{1}_{\{ I_{n,1} = i \}} \middle| \dn} = \dprod{\nabla\Ln(\hz)}{f}$, the previous display reads
	\begin{equation*}
		\sum_{j = 1}^{s_n} \E{ \eta_{n, j}^2 \middle| \Fnjm} = \E{ \pa*{ \frac{1}{n} \sum_{i = 1}^n \frac{\mathcal{L}^\prime(Z_i, \hz, f)}{\pi_{n,i}} \mathbbm{1}_{\{ I_{n,1} = i \}} }^2 \middle| \dn } - \dprod{\nabla\Ln(\hz)}{f}^2.
	\end{equation*}
	However, only one term of the sum in the right hand side of the previous display is non-zero.
	Thus
	\begin{align*}
		\E{ \pa*{ \frac{1}{n} \sum_{i = 1}^n \frac{\mathcal{L}^\prime(Z_i, \hz, f)}{\pi_{n,i}} \mathbbm{1}_{\{ I_{n,1} = i \}} }^2 \middle| \dn } &= \frac{1}{n^2} \sum_{i = 1}^n \frac{\mathcal{L}^\prime(Z_i, \hz, f)^2}{\pi_{n,i}^2} \P{ I_{n,1} = i \middle| \dn } \\
		&= \frac{1}{n^2} \sum_{i = 1}^n \frac{\mathcal{L}^\prime(Z_i, \hz, f)^2}{\pi_{n,i}}.
	\end{align*}
	Hence, the conditional variance reads
	\begin{equation*}
		\sum_{j = 1}^{s_n} \E{ \eta_{n, j}^2 \middle| \Fnjm} = \frac{1}{n^2} \sum_{i = 1}^n \frac{\mathcal{L}^\prime(Z_i, \hz, f)^2}{\pi_{n,i}} - \dprod{\nabla\Ln(\hz)}{f}^2.
	\end{equation*}
	If one writes $\pi_{n, i} = \psi(Z_i) / \sum_{j = 1}^n \psi(Z_j)$, then by the law of large numbers
	\begin{align*}
		\sum_{j = 1}^{s_n} \E{ \eta_{n, j}^2 \middle| \Fnjm} &= \frac{1}{n^2} \sum_{i = 1}^n \pa*{ \sum_{j = 1}^n \psi(Z_j) } \frac{\mathcal{L}^\prime(Z_i, \hz, f)^2}{\psi(Z_i)} - \dprod{\nabla\Ln(\hz)}{f}^2 \\
		&= \pa*{ \frac{1}{n} \sum_{j = 1}^n \psi(Z_j) } \pa*{ \frac{1}{n} \sum_{i = 1}^n \frac{\mathcal{L}^\prime(Z_i, \hz, f)^2}{\psi(Z_i)} } - \dprod{\nabla\Ln(\hz)}{f}^2 \\
		&\overset{\Pr}{\underset{n \rightarrow \infty}{\longrightarrow}} \E{ \psi(Z) } \E{\frac{\mathcal{L}^\prime(Z, \hz, f)^2}{\psi(Z)} } - \dprod{\nabla\L(\hz)}{f}^2 = \dprod{\Sigma^*_\mathrm{M}(\psi) f}{f},
	\end{align*}
	where $\Sigma^*_\mathrm{M}(\psi) \defeq \E{\frac{\E {\psi (Z)}}{\psi (Z)} \ell^\prime(\hz(X), Y)^2 K_X \otimes K_X} - \nabla\L(\hz) \otimes \nabla\L(\hz)$.

	Let us now show that the martingale difference array satisfies the conditional Lindeberg-Feller condition.
	Let $\varepsilon > 0$,
	\begin{align*}
		&\sum_{j = 1}^{s_n} \E{ \eta_{n,j}^2 \mathbbm{1}_{\{ | \eta_{n,j} | > \varepsilon \}}  \middle| \dn } \\
		&\quad = s_n \E{ \eta_{n,1}^2 \mathbbm{1}_{\{ | \eta_{n,1} | > \varepsilon \}}  \middle| \dn } \\
		&\quad \leq \frac{s_n}{\varepsilon^2} \E{ \eta_{n,1}^4 \middle| \dn } \\
		&\quad = \frac{1}{s_n \varepsilon^2} \E{ \pa*{ \frac{1}{n} \sum_{i = 1}^n \frac{\mathcal{L}^\prime(Z_i, \hz, f)}{\pi_{n,i}} \mathbbm{1}_{\{ I_{n,j} = i \}} - \dprod{\nabla\Ln(\hz)}{f} }^4 \middle| \dn }.
	\end{align*}
	Furthermore expanding the forth order moment yields
	\begin{align*}
		&\sum_{j = 1}^{s_n} \E{ \eta_{n,j}^2 \mathbbm{1}_{\{ | \eta_{n,j} | > \varepsilon \}} \middle| \dn } \\
		&\quad \leq \frac{1}{s_n \varepsilon^2} \Bigg( \frac{1}{n^4} \sum_{i = 1}^n \frac{\mathcal{L}^\prime(Z_i, \hz, f)^4}{\pi_{n,i}^3} - 4 \frac{\dprod{\nabla\Ln(\hz)}{f}}{n^3} \sum_{i = 1}^n \frac{\mathcal{L}^\prime(Z_i, \hz, f)^3}{\pi_{n,i}^2} \\
		&\qquad\qquad + 6\frac{\dprod{\nabla\Ln(\hz)}{f}^2}{n^2} \sum_{i = 1}^n \frac{\mathcal{L}^\prime(Z_i, \hz, f)^2}{\pi_{n,i}} - 3\dprod{\nabla\Ln(\hz)}{f}^4 \Bigg).
	\end{align*}
	However by \Cref{hyp:ssp}, for all $\delta > 0$,
	\begin{align*}
		\P{ \max_{1 \le i \le n} \cb*{ \frac{1}{n \pi_{n, i}} } > \delta } \leq \frac{1}{\delta} \E{ \max_{1 \le i \le n} \cb*{ \frac{1}{n \pi_{n, i}} } } \leq \frac{1}{\delta} \sqrt{\E{ \max_{1 \le i \le n} \cb*{ \frac{1}{n \pi_{n, i}} }^2 }} = O\pa*{ \frac{1}{\delta} }.
	\end{align*}
	Whence, $\max_{1 \le i \le n} \cb*{ n \pi_{n, i} }^{-1} = O_\Pr(1)$ and more specifically $\max_{1 \le i \le n} \cb*{ n \pi_{n, i} }^{-p} = O_\Pr(1)$ for any $p \in \{ 1 , 2, 3 \}$.
	Hence, using the fact that $\dprod{\nabla\Ln(\hz)}{f}^4 = O_\Pr(1)$,
	\begin{align*}
		\frac{1}{n^{p+1}} \sum_{i = 1}^n \frac{ | \mathcal{L}^\prime(Z_i, \hz, f) |^{p+1}}{\pi_{n, i}^p} &\leq \max_{1 \le i \le n} \cb*{ \frac{1}{n^p \pi_{n, i}^p} } \frac{1}{n} \sum_{i = 1}^n | \mathcal{L}^\prime(Z_i, \hz, f) |^{p+1} \\
		&= \max_{1 \le i \le n} \cb*{ \frac{1}{n \pi_{n, i}} }^p \frac{1}{n} \sum_{i = 1}^n | \mathcal{L}^\prime(Z_i, \hz, f) |^{p+1} = O_\Pr(1).
	\end{align*}
	Thus $\sum_{j = 1}^{s_n} \E{ \eta_{n,j}^2 \mathbbm{1}_{\{ | \eta_{n,j} | > \varepsilon \}} \middle| \dn } \pconv 0$.
	Whence, from \Cref{thm:technical_conditional_dconv},
	\begin{equation*}
		\prod_{j = 1}^{s_n} \E{e^{it \eta_{n,j}} \middle| \Fnjm} \pconv \exp\pa*{-\frac{1}{2}t^2 \dprod{\Sigma^*_\mathrm{M}(\psi) f}{f}}.
	\end{equation*}
	Finally, $\E{ e^{it \eta_{n, j}} \middle| \Fnjm} = \E{ e^{it \eta_{n, j}} \middle| \dn}$ and conditionally to $\dn$, $\eta_{n, 1}, \dots, \eta_{n,s_n}$ are independent hence
	\begin{equation*}
		\E{ e^{it S_{s_n}^*} \middle| \dn} = \prod_{j = 1}^{s_n} \E{e^{it \eta_{n,j}} \middle| \Fnjm} \pconv \exp\pa*{-\frac{1}{2}t^2 \dprod{\Sigma^*_\mathrm{M}(\psi) f}{f}}. \qedhere
	\end{equation*}
\end{proof}

We can now prove the main theorem using the above intermediate results.
\begin{proof}[Proof of \Cref{thm:srm_normality_R}]
	The fundamental theorem of calculus yields that
	\begin{equation*}
		0 = \nabla\Rsn(\hsn) = \nabla\Rsn(\hz) + \int_0^1 \nabla^2\Rsn(\bar h_{n,t})(\hsn - \hz) \,\d t,
	\end{equation*}
	with $\bar h_{n,t} = \hz + t(\hsn - \hz)$.
	By \Cref{lem:srm_normality_2}, one has
	\begin{align*}
		-\sqrt{s_n} \nabla\Rsn(\hz) = \sqrt{s_n} \nabla^2\risk(\hz) (\hsn - \hz) + o_\Pr(1).
	\end{align*}
	By definition of $\Rsn(\hz)$, the fact that $\L(\hz) = - \lz \hz$ and adding and retrieving $\L(\hz)$ and $\Ln(\hz)$, one has
	\begin{align*}
		\nabla\Rsn(\hz) &= \nabla\Lsn(\hz) + \lsn \hz \\
		&= (\Lsn(\hz) - \nabla\Ln(\hz)) + (\nabla\Ln(\hz) - \nabla\L(\hz)) + (\lsn - \lz) \hz.
	\end{align*}
	Whence
	\begin{align*}
		\sqrt{s_n} \nabla^2\risk(\hz) (\hsn - \hz) + o_\Pr(1) &= - \sqrt{s_n} (\nabla\Lsn(\hz) - \nabla\Ln(\hz)) \\
		&\qquad - \sqrt{\frac{s_n}{n}} \sqrt{n} (\nabla\Ln(\hz) - \nabla\L(\hz)) \\
		&\qquad - \sqrt{s_n} (\lsn - \lz) \hz.
	\end{align*}
	First, $\sqrt{s_n} (\lsn - \ln) \hz \pconv 0$ by \Cref{hyp:lsn_speed}.
	Applying \ref{lem:srm_normality_final} with $\Fn = \sigma(\dn)$ and $H_n = \sqrt{s_n} (\nabla\Lsn(\hz) - \nabla\Ln(\hz))$, together with \Cref{lem:srm_normality_1,lem:srm_normality_1_bis} yields
	\begin{equation*}
		\sqrt{s_n} \nabla^2\risk(\hz) (\hsn - \hz) \dconv \GP(0, \Sigma^*_\mathrm{M}(\psi) + c \Sigma). \qedhere
	\end{equation*}
\end{proof}

\subsection{Proof of \Cref{thm:srm_normality_P}}

\begin{lemma} \label{lem:poisson_srm_normality_1}
	Under \Cref{hyp:kernel,hyp:loss,hyp:distribution,hyp:ssp}, furthermore assume that $s_n / n \conv c \in [0, 1]$ and $\pi_{n, i} \leq 1 /s_n$ almost surely for all $n > 0$ and $i \in \{ 1, \dots, n \}$.
	Then for all $t \in \R$ and $f \in \H$,
	\begin{equation*}
		\E{\exp\pa*{it S_{s_n}^*} \middle| \dn} \pconv \exp\pa*{-\frac{1}{2} t^2 \dprod{\Sigma^*_{\mathrm{P},c}(\psi) f}{f}},
	\end{equation*}
	with
	\begin{equation*}
	S_{s_n}^* \defeq \sqrt{s_n} \dprod{(\nabla\Lsn(\hz) - \nabla\Ln(\hz))}{f},
	\end{equation*}
	and
	\begin{equation*}
		\Sigma^*_{\mathrm{P},c}(\psi) \defeq \E{\frac{\E {\psi (Z)}}{\psi (Z)} \ell^\prime(\hz(X), Y)^2 K_X \otimes K_X} - c\E{ \ell^\prime(\hz(X), Y)^2 K_X \otimes K_X}.
	\end{equation*}
\end{lemma}

\begin{proof}
	The goal is to apply \Cref{thm:technical_conditional_dconv}.
	Recall that
	\begin{align*}
		\nabla\Lsn(\hz) = \frac{1}{n} \sum_{i = 1}^{n} \frac{\ell^\prime (\hz(X_i), Y_i)}{s_n\pi_{n,i}} K_{X_i} \delta_{n,i},
	\end{align*}
	Let $f \in \H$ and define $S_{s_n}^* \defeq \dprod{\sqrt{s_n} (\nabla\Lsn - \nabla\Ln)(\hz)}{f}$ a real-valued random variable and $\mathcal{L}^\prime(Z, h, f) \defeq \ell^\prime(h(X), Y)f(X)$, then
	\begin{align*}
		S_{s_n}^* &= \frac{\sqrt{s_n}}{n} \sum_{i = 1}^{n} \frac{\mathcal{L}^\prime(Z_i, \hz, f)}{s_n \pi_{n, i}} \delta_{n,i} - \frac{\sqrt{s_n}}{n} \sum_{i = 1}^{n} \mathcal{L}^\prime(Z_i, \hz, f) \\
		&= \sum_{i = 1}^{n} \frac{\mathcal{L}^\prime(Z_i, \hz, f)}{n \sqrt{s_n} \pi_{n,i}}  (\delta_{n,i} - s_n \pi_{n,i} ) \\
		&= \sum_{i = 1}^{n} \eta_{n,i},
	\end{align*}
	where
	\begin{equation*}
		\eta_{n,i} \defeq \frac{\mathcal{L}^\prime(Z_i, \hz, f)}{n \sqrt{s_n} \pi_{n,i}}  (\delta_{n,i} - s_n \pi_{n,i} ),
	\end{equation*}
	for all $n \in \N$ and $i \in \{ 1, \dots, n \}$.
	Moreover, define the filtration array $(\Fni)_{n \in \N, 1 \leq i \leq n}$ such that $\Fni$ is the $\sigma$-algebra generated by $\dn$ and $\delta_{n,1}, \dots, \delta_{n,i}$ for $1 \leq i \leq n$ and $n \in \N$.
	Notice that for all $n \in \N$ and $i \leq n$ the conditional distributions of $\eta_{n,i}$ with respect to $\dn$ and $\Fnim$ are the same.
	Remark as well that, conditionally to $\dn$, $S_{s_n}^*$ is a sum of independent random variables such that $\E{ \eta_{n, i} \middle| \dn } = 0$ since $\E{ \delta_{n,i} \middle| \dn } = s_n \pi_{n, i}$.
	Hence $\{ \eta_{n,i} : n \in \N, 1 \leq i \leq n \}$ is a martingale difference array for the filtration $(\Fni)_{n \in \N, 1 \leq i \leq n}$.
	Moreover, its conditional variance reads
	\begin{align*}
		\sum_{i = 1}^{n} \E{ \eta_{n, i}^2 \middle| \Fnim} &= \sum_{i = 1}^{n} \frac{\mathcal{L}^\prime(Z_i, \hz, f)^2}{n^2 s_n \pi_{n, i}^2} \E{ (\delta_{n,i} - s_n \pi_{n, i})^2 \middle| \dn}.
	\end{align*}
	Since, conditionally to $\dn$, $\delta_{n,i}$ follows a Bernoulli distribution of parameter $s_n \pi_{n, i}$, the previous display reads
	\begin{align*}
		\sum_{i = 1}^{n} \E{ \eta_{n, i}^2 \middle| \Fnim} &= \sum_{i = 1}^{n} \frac{s_n \pi_{n, i} (1 - s_n \pi_{n, i})}{n^2 s_n \pi_{n, i}^2} \mathcal{L}^\prime(Z_i, \hz, f)^2 \\
		&= \frac{1}{n^2} \sum_{i = 1}^{n} \frac{\mathcal{L}^\prime(Z_i, \hz, f)^2}{\pi_{n, i}} - \frac{s_n}{n^2} \sum_{i = 1}^{n} \mathcal{L}^\prime(Z_i, \hz, f)^2.
	\end{align*}
	On the one hand
	\begin{equation*}
		\frac{s_n}{n^2} \sum_{i = 1}^{n} \mathcal{L}^\prime(Z_i, \hz, f)^2 \pconv c\E{ \mathcal{L}^\prime(Z, \hz, f)^2 },
	\end{equation*}
	as it is assumed that $s_n / n \conv c \in [0,1]$.
	On the other hand, if one writes $\pi_{n, i} = \psi(Z_i) / \sum_{j = 1}^n \psi(Z_j)$, by the law of large numbers
	\begin{align*}
		\frac{1}{n^2} \sum_{i = 1}^{n} \frac{\mathcal{L}^\prime(Z_i, \hz, f)^2}{\pi_{n, i}} &= \frac{1}{n^2} \sum_{i = 1}^n \pa*{ \sum_{j = 1}^n \psi(Z_j) } \frac{\mathcal{L}^\prime(Z_i, \hz, f)^2}{\psi(Z_i)} \\
		&= \pa*{ \frac{1}{n} \sum_{j = 1}^n \psi(Z_j) } \pa*{ \frac{1}{n} \sum_{i = 1}^n \frac{\mathcal{L}^\prime(Z_i, \hz, f)^2}{\psi(Z_i)} } \\
		&\pconv \E{ \psi(Z) } \E{\frac{\mathcal{L}^\prime(Z, \hz, f)^2}{\psi(Z)} }.
	\end{align*}
	Therefore
	\begin{align*}
		\sum_{i = 1}^{n} \E{ \eta_{n, i}^2 \middle| \Fnim} &\pconv \E{\frac{\Ex \psi(Z)}{\psi(Z)} \mathcal{L}^\prime(Z, \hz, f)^2 } - c\E{ \mathcal{L}^\prime(Z, \hz, f)^2 } \\
		&\qquad \defeq \dprod{\Sigma^*_{\mathrm{P},c}(\psi) f}{f},
	\end{align*}
	where
	\begin{equation*}
		\Sigma^*_{\mathrm{P},c}(\psi) \defeq \E{\frac{\Ex \psi (Z)}{\psi (Z)} \ell^\prime(\hz(X), Y)^2 K_X \otimes K_X} - c \E{ \ell^\prime(\hz(X), Y)^2 K_X \otimes K_X}.
	\end{equation*}
	Moreover, let us now show that the martingale difference array satisfies the conditional Lindeberg-Feller condition.
	Let $\varepsilon > 0$,
	\begin{align*}
		\sum_{i = 1}^{n} \E{ \eta_{n,i}^2 \mathbbm{1}_{\{ | \eta_{n,i} | > \varepsilon \}} \middle| \dn } &\leq \frac{1}{\varepsilon^2} \sum_{i = 1}^n \E{ \eta_{n,i}^4 \mathbbm{1}_{\{ | \eta_{n,i} | > \varepsilon \}} \middle| \dn } \\
		&\leq \frac{1}{\varepsilon^2} \sum_{i = 1}^n \E{ \eta_{n,i}^4 \middle| \dn } \\
		&= \frac{1}{\varepsilon^2} \sum_{i=1}^{n} \frac{\mathcal{L}^\prime(Z_i, \hz, f)^4}{n^4 s_n^2 \pi_{n, i}^4} \E{ (\delta_{n,i} - s_n \pi_{n, i})^4 \middle| \dn}.
	\end{align*}
		Furthermore
	\begin{equation*}
		\E{ (\delta_{n,i} - s_n \pi_{n, i})^4 \middle| \dn} = p(1-p) (1 - 3 p(1-p)) \leq 4p,
	\end{equation*}
	where $p \defeq s_n \pi_{n, i} \in [0, 1]$, since $1 - p \leq 1$ and $3p^2 - 3p \leq 3$.
	Hence bounding the forth order moment yields
	\begin{align*}
		\sum_{i = 1}^{n} \E{ \eta_{n,i}^2 \mathbbm{1}_{\{ | \eta_{n,i} | > \varepsilon \}}  \middle| \dn } &\leq \frac{1}{n^4 \varepsilon^2} \sum_{i=1}^{n} \frac{p(1-p) (1 - 3 p(1-p))}{p^2 \pi_{n, i}^2} \mathcal{L}^\prime(Z_i, \hz, f)^4 \\
		&\leq \frac{4}{n^4 \varepsilon^2} \sum_{i=1}^{n} \frac{1}{p \pi_{n, i}^2} \mathcal{L}^\prime(Z_i, \hz, f)^4 \\
		&= \frac{4}{s_n \varepsilon^2} \max_{1 \le i \le n} \cb*{ \frac{1}{n  \pi_{n,i}} }^3 \pa*{ \frac{1}{n} \sum_{i=1}^{n} \mathcal{L}^\prime(Z_i, \hz, f)^4 }.
	\end{align*}
	However, by \Cref{hyp:ssp} there exists $M > 0$ and $N \in \N$ such that for all $n \geq N$,
	\begin{equation*}
		\E{ \max_{1 \le i \le n} \cb*{ \frac{1}{n  \pi_{n,i}} }^2 } \leq M^2,
	\end{equation*}
	hence for all $\varepsilon > 0$,
	\begin{align*}
		\P{ \max_{1 \le i \le n} \cb*{ \frac{1}{n  \pi_{n,i}} }^3 > M^{3/2} \varepsilon^{-3/2} } &= \P{ \max_{1 \le i \le n} \cb*{ \frac{1}{n  \pi_{n,i}} }^2 > M \varepsilon^{-1} } \\
		&\leq \frac{\varepsilon}{M^2} \E{ \max_{1 \le i \le n} \cb*{ \frac{1}{n  \pi_{n,i}} }^2 } \\
		&\leq  \varepsilon.
	\end{align*}
	Therefore $\max_{1 \le i \le n} \cb*{ n \pi_{n,i} }^{-3} = O_\Pr(1)$.
	Moreover, by \Cref{hyp:distribution}
	\begin{equation*}
		\frac{1}{n} \sum_{i = 1}^n \mathcal{L}^\prime(Z_i, \hz, f)^4 = O_\Pr(1),
	\end{equation*}
	thus
	\begin{equation*}
		\sum_{i = 1}^{n} \E{ \eta_{n,i}^2 \mathbbm{1}_{\{ | \eta_{n,i} | > \varepsilon \}} \middle| \dn } \leq \frac{4}{s_n \varepsilon} \max_{1 \le i \le n} \cb*{\frac{1}{n \pi_{n, i}}}^3 \pa*{\frac{1}{n} \sum_{i = 1}^n \mathcal{L}^\prime(Z_i, \hz, f)^4} \pconv 0,
	\end{equation*}
	and from \Cref{thm:technical_conditional_dconv}
	\begin{equation*}
		\prod_{i = 1}^{n} \E{e^{it \eta_{n,i}} \middle| \Fnim} \pconv e^{-\frac{1}{2}t^2 \dprod{\Sigma^*_{\mathrm{P},c}(\psi) f}{f}}.
	\end{equation*}
	Finally, $\E{ e^{it \eta_{n, i}} \middle| \Fnim} = \E{ e^{it \eta_{n, i}} \middle| \dn}$ and conditionally to $\dn$, $\eta_{n, 1}, \dots, \eta_{n,n}$ are independent hence
	\begin{equation*}
		\E{ \exp\pa*{it S_{s_n}^*} \middle| \dn} = \prod_{i = 1}^{n} \E{\exp\pa*{it \eta_{n,i}} \middle| \Fnim} \pconv \exp\pa*{-\frac{1}{2}t^2 \dprod{\Sigma^*_{\mathrm{P},c}(\psi) f}{f}}. \qedhere
	\end{equation*}
\end{proof}

We can now prove the main theorem using the above intermediate results.

\begin{proof}[Proof of \Cref{thm:srm_normality_P}]
	Following the same steps as in the proof of \Cref{thm:srm_normality_R}, one has
	\begin{align*}
		\sqrt{s_n} \nabla^2\risk(\hz) (\hsn - \hz) + o_\Pr(1) &= - \sqrt{s_n} (\nabla\Lsn(\hz) - \nabla\Ln(\hz)) \\
		&\qquad - \sqrt{\frac{s_n}{n}} \sqrt{n} (\nabla\Ln(\hz) - \nabla\L(\hz)).
	\end{align*}
	Applying \ref{lem:srm_normality_final} with $\Fn = \sigma(\dn)$ and $H_n = \sqrt{s_n} (\nabla\Lsn(\hz) - \nabla\Ln(\hz))$, together with \Cref{lem:poisson_srm_normality_1,lem:srm_normality_1_bis} yields
	\begin{equation*}
		\sqrt{s_n} \nabla^2\risk(\hz) (\hsn - \hz) \dconv \GP(0, \Sigma^*_{\mathrm{P},c}(\psi) + c \Sigma). \qedhere
	\end{equation*}
\end{proof}

\subsection{Proof of \Cref{cor:uniform_ssp_R} and \Cref{cor:uniform_ssp_P}}

\begin{proof}
	First notice that taking a constant importance \( \psi \) yields uniform probabilities and that both respectively verify \Cref{hyp:ssp,hyp:nonexplosive_cov}.
	Thus, apply \Cref{thm:srm_normality_R,thm:srm_normality_P} with
	constant $\psi \equiv 1$.
	This yields
	\begin{equation*}
		\Sigma^*_\mathrm{M}(\psi) = \E{\ell^\prime(\hz(X), Y)^2 K_X \otimes K_X } - \nabla\L(\hz) \otimes \nabla\L(\hz) = \Sigma,
	\end{equation*}
	and
	\begin{equation*}
		\Sigma^*_{\mathrm{P},c}(\psi) = (1 - c)\E{\ell^\prime(\hz(X), Y)^2 K_X \otimes K_X }.
	\end{equation*}
	Which concludes the proof.
\end{proof}

\subsection{Proof of \Cref{prop:optimal_ssp_R}}

\begin{proof}
	Let $\psi : \Z \rightarrow \R^*_+$.
	By Cauchy-Schwarz, one has
	\begin{align*}
		\Tr(\Sigma^*_\mathrm{M}(\psi)) &= \E{ \psi(Z) }\E{ \frac{1}{\psi(Z)} \ell^\prime(\hz(X), Y)^2 k(X,X)} - \| \nabla\L(\hz)\|_\H^2 \\
		&\geq \E{ \sqrt{\frac{\psi(Z)}{\psi(Z)} \ell^\prime(\hz(X), Y)^2 k(X,X)} }^2 - \| \nabla\L(\hz)\|_\H^2 \\
		&= \E{ | \ell^\prime(\hz(X), Y) | \sqrt{k(X,X)} }^2 - \| \nabla\L(\hz)\|_\H^2,
	\end{align*}
	with equality if and only if $\psi(Z) \propto \psi(Z)^{-1} \ell^\prime(\hz(X), Y)^2 k(X,X)$ or equivalently $\psi(Z) \propto | \ell^\prime(\hz(X), Y) | k(X,X)^{1/2}$.
	Thus, $\psi^\star(Z) \defeq | \ell^\prime(\hz(X), Y) | \sqrt{k(X, X)}$ is a minimizer of $\psi \mapsto \Tr(\Sigma^*(\psi))$.
\end{proof}

\subsection{Proof of \Cref{thm:srm_alpha_pilot_normality_R}}

In this section, denote for all $n > 0$ and $i \in \{ 1, \dots, n \}$,
\begin{gather*}
	\psi_i = |\ell^\prime(\hz(X_i), Y_i)| \sqrt{k(X_i, X_i)}\quad \text{and} \quad \psi_{n,i} = |\ell^\prime(\hp(X_i), Y_i)| \sqrt{k(X_i, X_i)},
\end{gather*}
such that $\tilde \pi_{n, i} \propto \psi_{n,i}$ and
\begin{equation*}
	\tilde \pi_{n,i}^\alpha = (1 - \alpha) \tilde\pi_{n,i} + \alpha\frac{1}{n}.
\end{equation*}
As $\tilde \pi_{n,i}^\alpha \geq \alpha n^{-1}$ for all $i \in \{ 1, \dots, n \}$ then
\begin{equation*}
	\E{ \max_{1 \le i \le n} \cb*{\frac{1}{n \tilde\pi_{n,i}^\alpha}}^2 } \leq \frac{1}{\alpha^2},
\end{equation*}
hence the sequence of probability vectors $(\tilde\ppi_n^\alpha)_{n > 0}$ verifies \Cref{hyp:ssp}, thus we can deduce the consistency of $\hasn$ by application of \Cref{thm:srm_consistency_core}.
Therefore
\begin{equation}
	\| \hasn - \hz \|_\H \pconv 0.
\end{equation}
Moreover, the results of \Cref{lem:srm_normality_2,lem:srm_normality_1_bis} still apply, if replacing $\hsn$ by $\hasn$ and $\nabla\Lsn$ by $\nabla\Lsna$, because $(\tilde\ppi_n^\alpha)_{n > 0}$ verifies \Cref{hyp:ssp}.

\begin{lemma} \label{lem:srm_alpha_pilot_normality_1}
	Under \Cref{hyp:kernel,hyp:loss,hyp:distribution,hyp:lsn_speed}, if $\hp \pconv \hz$, one has for all $t \in \R$ and $f \in \H$,
	\begin{equation*}
		\E{\exp\pa*{it S_{s_n}^{*\alpha}} \middle| \dn, \hp} \pconv \exp\pa*{-\frac{1}{2} t^2 \dprod{\Sigma^*_\mathrm{M}(\alpha) f}{f}},
	\end{equation*}
	with
	\begin{gather*}
		S_{s_n}^{*\alpha} \defeq \sqrt{s_n} \dprod{\nabla\Lsna(\hz) - \nabla\Ln(\hz)}{f}, \\
		\nabla\Lsna(\hz) = \frac{1}{s_n} \sum_{j = 1}^{s_n} \sum_{i = 1}^n \frac{\ell^\prime(\hz(X_i), Y_i))}{n \tilde \pi_{n,i}^\alpha} K_{X_i} \mathbbm{1}_{\{ I_{n,j} = i \}},
	\end{gather*}
	and
	\begin{equation*}
		\Sigma^*_\mathrm{M}(\alpha) \defeq \E{ \frac{\E{ \psi^\star(Z) } \ell^\prime(\hz(X), Y)^2}{(1 - \alpha) \psi^\star(Z) + \alpha \E{ \psi^\star(Z) }} K_X \otimes K_X } - \nabla\L(\hz) \otimes \nabla\L(\hz).
	\end{equation*}
\end{lemma}

\begin{proof}
	As per the proof of \Cref{lem:srm_normality_1}, the goal is to apply \Cref{thm:technical_conditional_dconv}.
	The only differences lie in the sequence of filtration to handle a martingale difference array, the conditional variance and its limit.
	As the rest remains unchanged, this is what we will focus on in this proof.
	Let $f \in \H$.
	Define $S_{s_n}^{*\alpha} \defeq \dprod{\sqrt{s_n} (\nabla\Lsna - \nabla\Ln)(\hz)}{f}$ a real-valued random variable and $\mathcal{L}^\prime(Z, h, f) \defeq \ell^\prime(h(X), Y)f(X)$, then
	\begin{align*}
		S_{s_n}^{*\alpha} &= \sum_{j = 1}^{s_n} \frac{1}{\sqrt{s_n}} \pa*{ \frac{1}{n} \sum_{i = 1}^n \frac{\mathcal{L}^\prime(Z_i, \hz, f)}{\tilde \pi_{n,i}^\alpha} \mathbbm{1}_{\{ I_{n,j} = i \}} - \dprod{\nabla\Ln(\hz)}{f} } \\
		&= \sum_{j = 1}^{s_n} \tilde \eta_{n,j},
	\end{align*}
	where
	\begin{equation*}
		\tilde \eta_{n,j} \defeq \frac{1}{\sqrt{s_n}} \pa*{ \frac{1}{n} \sum_{i = 1}^n \frac{\mathcal{L}^\prime(Z_i, \hz, f)}{\tilde \pi_{n,i}^\alpha} \mathbbm{1}_{\{ I_{n,j} = i \}} - \dprod{\nabla\Ln(\hz)}{f} },
	\end{equation*}
	for all $n \in \N$ and $j \in \{ 1, \dots, s_n \}$.
	Moreover, define the filtrations $(\tFnj)_{n \in \N, 0 \leq j \leq s_n}$ such that $\tFnj$ is the $\sigma$-algebra generated by $\dn$, $\hp$ and $I_{n,1}, \dots, I_{n,j}$ for $0 \leq j \leq s_n$ and $n \in \N$.
	Notice that for all $n \in \N$ and $j \in \{ 1, \dots, s_n \}$ the conditional distributions of $\tilde \eta_{n,j}$ with respect to $(\dn, \hp)$ and $\tFnjm$ are the same.
	Remark as well that, conditionally to $\dn$ and $\hp$, $S_{s_n}^*$ is a sum of \iid\ random variables such that $\E{ \tilde \eta_{n, j} \middle| \dn, \hp } = 0$ since $\P{ I_{n,j} = i \middle| \dn, \hp } = \tilde \pi_{n, i}^\alpha$.
	Hence $\{ \tilde \eta_{n,j} : n \in \N, 1 \leq j \leq s_n \}$ is a martingale difference array for the filtration $(\tFnj)_{n \in \N, 1 \leq j \leq s_n}$.
	Moreover, its conditional variance reads
	\begin{align*}
		\sum_{j = 1}^{s_n} \E{ \tilde \eta_{n, j}^2 \middle| \tFnjm} &= s_n \E{ \tilde \eta_{n, 1}^2 \middle| \dn, \hp } \\
		&= \E{ \pa*{ \frac{1}{n} \sum_{i = 1}^n \frac{\mathcal{L}^\prime(Z_i, \hz, f)}{\tilde \pi_{n,i}^\alpha} \mathbbm{1}_{\{ I_{n,j} = i \}} - \dprod{\nabla\Ln(\hz)}{f} }^2 \middle| \dn, \hp}.
	\end{align*}
	Since $\E{ \frac{1}{n} \sum_{i = 1}^n \frac{\mathcal{L}^\prime(Z_i, \hz, f)}{\tilde \pi_{n,i}^\alpha} \mathbbm{1}_{\{ I_{n,j} = i \}} \middle| \dn, \hp} = \dprod{\nabla\Ln(\hz)}{f}$, the previous display reads
	\begin{equation*}
		\sum_{j = 1}^{s_n} \E{ \tilde \eta_{n, j}^2 \middle| \tFnjm} = \E{ \pa*{ \frac{1}{n} \sum_{i = 1}^n \frac{\mathcal{L}^\prime(Z_i, \hz, f)}{\tilde \pi_{n,i}^\alpha} \mathbbm{1}_{\{ I_{n,j} = i \}} }^2 \middle| \dn, \hp } - \dprod{\nabla\Ln(\hz)}{f}^2.
	\end{equation*}
	However only one term in the sum is non-zero.
	Thus
	\begin{align*}
		\sum_{j = 1}^{s_n} \E{ \tilde \eta_{n, j}^2 \middle| \tFnjm} &= \frac{1}{n^2} \sum_{i = 1}^n \pa*{ \frac{\mathcal{L}^\prime(Z_i, \hz, f)}{\tilde \pi_{n,i}^\alpha} }^2 \P{I_{n,j} = i \middle| \dn, \hp} - \dprod{\nabla\Ln(\hz)}{f}^2\\
		&= \frac{1}{n^2} \sum_{i = 1}^n\frac{\mathcal{L}^\prime(Z_i, \hz, f)^2}{\tilde \pi_{n,i}^\alpha} - \dprod{\nabla\Ln(\hz)}{f}^2.
	\end{align*}
	Furthermore
	\begin{align*}
		\tilde \pi_{n, i}^\alpha = (1 - \alpha) \tilde \pi_{n, i} + \frac{\alpha}{n} = \frac{(1 - \alpha) \psi_{n,i} + \frac{\alpha}{n} \sum_{j = 1}^{n} \psi_{n,j}}{\sum_{j = 1}^n \psi_{n,j}},
	\end{align*}
	hence
	\begin{align*}
		\frac{1}{n^2} \sum_{i = 1}^n \frac{\mathcal{L}^\prime(Z_i, \hz, f)^2}{\tilde \pi_{n,i}^\alpha} = \pa*{ \frac{1}{n} \sum_{j = 1}^{n} \psi_{n,j} } \pa*{ \frac{1}{n} \sum_{i=1}^n \frac{\mathcal{L}^\prime(Z_i, \hz, f)^2}{(1 - \alpha) \psi_{n,i} + \frac{\alpha}{n} \sum_{j = 1}^{n} \psi_{n,j}} }.
	\end{align*}
	Denote $\psi_i^\alpha \defeq (1 - \alpha) \psi_i + \alpha \Ex\psi^\star(Z)$ and $\psi_{n,i}^\alpha \defeq (1 - \alpha) \psi_{n,i} + \frac{\alpha}{n} \sum_{j= 1}^n \psi_{n,j}$.
	Hence for all $n > 0$ and $i \in \{ 1, \dots, n \}$,
	\begin{equation*}
		 \psi_i^\alpha \geq \alpha \Ex\psi^\star(Z) \quad \text{and} \quad \psi_{n,i}^\alpha \geq \frac{\alpha}{n} \sum_{j= 1}^n \psi_{n,j}.
	\end{equation*}
	Therefore by the triangle inequality
	\begin{align*}
		\av*{ \frac{1}{\psi_{n,i}^\alpha} - \frac{1}{\psi_i^\alpha} } &= \frac{| \psi_{n,i}^\alpha - \psi_i^\alpha |}{\psi_{n,i}^\alpha \psi_i^\alpha} \\
		&\leq \frac{n}{\alpha^2 (\sum_{j = 1}^n \psi_{n,j}) \Ex \psi} \pa*{ (1 - \alpha) | \psi_{n,i} - \psi_i | + \alpha \av*{ \frac{1}{n} \sum_{j = 1}^n \psi_{n,j} - \Ex\psi^\star(Z) }},
	\end{align*}
	and consequently
	\begin{align*}
		&\av*{ \frac{1}{n} \sum_{i=1}^n \frac{\mathcal{L}^\prime(Z_i, \hz, f)^2}{\psi_{n,i}^\alpha} - \frac{1}{n} \sum_{i=1}^n \frac{\mathcal{L}^\prime(Z_i, \hz, f)^2}{\psi_i^\alpha} } \\
		&\quad \leq \frac{n}{\alpha^2 (\sum_{j = 1}^n \psi_{n,j}) \Ex \psi} \Bigg( (1 - \alpha) \frac{1}{n} \sum_{i = 1}^n \mathcal{L}^\prime(Z_i, \hz, f)^2 |\psi_{n,i} - \psi_i| \\
		&\hspace{5cm} + \alpha \av*{ \frac{1}{n} \sum_{j = 1}^n \psi_{n,j} - \Ex\psi^\star(Z) } \bigg( \frac{1}{n} \sum_{i=1}^n \mathcal{L}^\prime(Z_i, \hz, f)^2 \bigg)\Bigg).
	\end{align*}
	Moreover, for all $n \in \N$, and $i \in \{ 1, \dots, n \}$,
	\begin{align*}
		&| \psi_{n,i} - \psi_i | \mathbbm{1}_{\{ \| \hp - \hz \|_\H < 1 \}} \\
		&\quad \leq \sqrt{k(X_i, X_i)} | \ell^\prime(\hp(X_i), Y_i) - \ell^\prime(\hz(X_i), Y_i) | \mathbbm{1}_{\{ \| \hp - \hz \|_\H < 1 \}} \\
		&\quad \leq \kappa \phi(Y_i) | \hp(X_i) - \hz(X_i) | \mathbbm{1}_{\{ \| \hp - \hz \|_\H < 1 \}} \\
		&\quad \leq \kappa^2 \| \hp - \hz \|_\H \phi(Y_i),
	\end{align*}
	using \Cref{hyp:loss}.
	Thus, together with \Cref{hyp:distribution} and $\hp$'s consistency, the previous inequality yields
	\begin{align*}
		\frac{1}{n} \sum_{i = 1}^n \mathcal{L}^\prime(Z_i, \hz, f)^2 |\psi_{n,i} - \psi_i| \mathbbm{1}_{\{ \| \hp - \hz \|_\H < 1 \}} \pconv 0.
	\end{align*}
	Whence applying \Cref{lem:technical_pconv_2} with $E_n = \{ \| \hp - \hz \| \geq 1 \}$ along with the previous display yields
	\begin{align*}
		\frac{1}{n} \sum_{i = 1}^n \mathcal{L}^\prime(Z_i, \hz, f)^2 |\psi_{n,i} - \psi_i| \pconv 0.
	\end{align*}
	Similarly $n^{-1} \sum_{i = 1}^n | \psi_{n,i} - \psi_i | \pconv 0$, which implies that $n^{-1} \sum_{i = 1}^n \psi_{n,i} \pconv \Ex \psi^\star(Z)$.
	Hence
	\begin{equation*}
		\av*{ \frac{1}{n} \sum_{i=1}^n \frac{\mathcal{L}^\prime(Z_i, \hz, f)^2}{\psi_{n,i}^\alpha} - \frac{1}{n} \sum_{i=1}^n \frac{\mathcal{L}^\prime(Z_i, \hz, f)^2}{\psi_i^\alpha} } \pconv 0.
	\end{equation*}
	However
	\begin{equation*}
		\frac{1}{n} \sum_{i=1}^n \frac{\mathcal{L}^\prime(Z_i, \hz, f)^2}{\psi_i^\alpha} \pconv \E{ \frac{\mathcal{L}^\prime(Z, \hz, f)^2}{(1 - \alpha) \psi^\star(Z) + \alpha \E{ \psi^\star(Z) }} },
	\end{equation*}
	which yields
	\begin{align*}
		\sum_{j = 1}^{s_n} \E{ \tilde \eta_{n, j}^2 \middle| \tFnjm} &= \frac{1}{n^2} \sum_{i = 1}^n \frac{\mathcal{L}^\prime(Z_i, \hz, f)^2}{\tilde \pi_{n,i}^\alpha} - \dprod{\nabla\Ln(\hz)}{f}^2 \\
		&\quad \pconv \E{ \frac{\E{\psi^\star(Z)} \mathcal{L}^\prime(Z, \hz, f)^2}{(1 - \alpha) \psi^\star(Z) + \alpha \E{ \psi^\star(Z) }} } - \dprod{\nabla\L(\hz)}{f}^2.
	\end{align*}
	Which concludes the proof.
\end{proof}

\begin{proof}[Proof of \Cref{thm:srm_alpha_pilot_normality_R}]
	Since \Cref{hyp:ssp} is verified, by application of \ref{lem:srm_normality_2} and following the proof of \Cref{thm:srm_normality_R}, one gets
	\begin{align*}
		\sqrt{s_n} \nabla^2\risk(\hz) (\hasn - \hz) + o_\Pr(1) &= - \sqrt{s_n} (\nabla\Lsna(\hz) - \nabla\Ln(\hz)) \\
		&\qquad - \sqrt{\frac{s_n}{n}} \sqrt{n} (\nabla\Ln(\hz) - \nabla\L(\hz)).
	\end{align*}
	Applying \ref{lem:srm_normality_final} with $\Fn = \sigma(\dn, \hp)$ and $H_n = \sqrt{s_n} (\nabla\Lsna(\hz) - \nabla\Ln(\hz))$, together with \Cref{lem:srm_normality_1_bis,lem:srm_alpha_pilot_normality_1} yields
	\begin{equation*}
		\sqrt{s_n} \nabla^2\risk(\hz) (\hasn - \hz) \dconv \GP(0, \Sigma^*_\mathrm{M}(\alpha) + c \Sigma). \qedhere
	\end{equation*}
\end{proof}

\subsection{Proof of \Cref{thm:srm_alpha_pilot_normality_P}}

In this section, as of the previous one, denote for all $n > 0$ and $i \in \{ 1, \dots, n \}$,
\begin{gather*}
	\psi_i = |\ell^\prime(\hz(X_i), Y_i)| \sqrt{k(X_i, X_i)}\quad \text{and} \quad \psi_{n,i} = |\ell^\prime(\hp(X_i), Y_i)| \sqrt{k(X_i, X_i)},
\end{gather*}
such that $\tilde \pi_{n, i} \propto \psi_{n,i}$ and
\begin{equation*}
	\tilde \pi_{n,i}^\alpha = (1 - \alpha) \tilde\pi_{n,i} + \alpha\frac{1}{n}.
\end{equation*}
As $\tilde \pi_{n,i}^\alpha \geq \alpha n^{-1}$ for all $i \in \{ 1, \dots, n \}$ then $\tilde\pi_{n, i}^\alpha \wedge s_n^{-1} \geq (\alpha n^{-1}) \wedge s_n^{-1} \geq \alpha n^{-1}$ since $\alpha \in (0, 1)$ and $n \geq s_n$.
Therefore
\begin{equation*}
	\E{ \max_{1 \le i \le n} \cb*{\frac{1}{n (\tilde\pi_{n, i}^\alpha \wedge s_n^{-1})}}^2 } \leq \frac{1}{\alpha^2},
\end{equation*}
hence the sequence of vectors $(\tilde\ppi_n^\alpha \wedge s_n^{-1})_{n > 0}$ verifies \Cref{hyp:ssp}, thus we can deduce the consistency of $\hasn$ by application of \Cref{thm:srm_consistency_core}.
Therefore
\begin{equation}
	\| \hasn - \hz \|_\H \pconv 0.
\end{equation}
Moreover, the results of \Cref{lem:srm_normality_2,lem:srm_normality_1_bis} still apply, if replacing $\hsn$ by $\hasn$ and $\nabla\Lsn$ by $\nabla\Lsna$, because $(\tilde\ppi_n^\alpha)_{n > 0}$ verifies \Cref{hyp:ssp}.

\begin{lemma} \label{lem:poisson_srm_alpha_pilot_normality_1}
	Under \Cref{hyp:kernel,hyp:loss,hyp:distribution,hyp:lsn_speed}, if $\hp \pconv \hz$, one has for all $t \in \R$ and $f \in \H$,
	\begin{equation*}
		\E{\exp\pa*{it S_{s_n}^{*\alpha}} \middle| \dn, \hp} \pconv \exp\pa*{-\frac{1}{2} t^2 \dprod{\Sigma^*_{\mathrm{P},c}(\alpha) f}{f}},
	\end{equation*}
	with
	\begin{gather*}
		S_{s_n}^{*\alpha} \defeq \sqrt{s_n} \dprod{\nabla\Lsna(\hz) - \nabla\Ln(\hz)}{f}, \\
		\nabla\Lsna(\hz) = \frac{1}{n} \sum_{i = 1}^{n} \frac{\ell^\prime(\hz(X_i), Y_i)}{(s_n \tilde \pi_{n,i}^\alpha) \wedge 1} K_{X_i} \delta_{n,i},
	\end{gather*}
	and
	\begin{equation*}
		\Sigma^*_{\mathrm{P},c}(\alpha) \defeq \E{ \frac{\E{ \psi^\star(Z) } \ell^\prime(\hz(X), Y)^2}{ \Psi^\star_\alpha(Z) \wedge (c^{-1} \E{\psi^\star(Z)})} K_X \otimes K_X } - c \E{ \ell^\prime(\hz,(X), Y)^2 K_X \otimes K_X },
	\end{equation*}
	where $\Psi^\star_\alpha(Z) \defeq (1 - \alpha) \psi^\star(Z) + \alpha \Ex\psi^\star(Z)$.
\end{lemma}

\begin{proof}
	As in the proof of \Cref{lem:poisson_srm_normality_1}, the goal is to apply \Cref{thm:technical_conditional_dconv}.
	The only differences lie in the sequence of filtration to handle a martingale difference array, the conditional variance and its limit.
	As the rest remains unchanged, this is what we will focus on in this proof.
	Let $f \in \H$.
	Define $S_{s_n}^{*\alpha} \defeq \dprod{\sqrt{s_n} (\nabla\Lsna - \nabla\Ln)(\hz)}{f}$ a real-valued random variable and $\mathcal{L}^\prime(Z, h, f) \defeq \ell^\prime(h(X), Y)f(X)$, then
	\begin{align*}
		S_{s_n}^{*\alpha} &= \frac{\sqrt{s_n}}{n} \sum_{i = 1}^n \frac{\mathcal{L}^\prime(Z_i, \hz, f)}{(s_n \tilde \pi_{n,i}^\alpha) \wedge 1} \delta_{n,i} - \sqrt{s_n} \dprod{\nabla\Ln(\hz)}{f} \\
		&= \sum_{j = 1}^{s_n} \tilde \eta_{n,i},
	\end{align*}
	where
	\begin{equation*}
		\tilde \eta_{n,i} \defeq \frac{\mathcal{L}^\prime(Z_i, \hz, f)}{n \sqrt{s_n} \check\pi_{n,i}^\alpha} (\delta_{n,i} - s_n \check\pi_{n, i}^\alpha),
	\end{equation*}
	for all $n \in \N$ and $i \in \{ 1, \dots, n \}$, with $\check\pi_{n,i}^\alpha = \tilde\pi_{n, i}^\alpha \wedge s_n^{-1}$.
	Moreover, define the filtrations $(\tFni)_{n \in \N, 0 \leq i \leq n}$ such that $\tFni$ is the $\sigma$-algebra generated by $\dn$, $\hp$ and $\delta_{n,1}, \dots, \delta_{n,i}$ for $0 \leq i \leq n$ and $n \in \N$.
	Notice that for all $n \in \N$ and $i \in \{ 1, \dots, n \}$ the conditional distributions of $\tilde \eta_{n,i}$ with respect to $(\dn, \hp)$ and $\tFnim$ are the same.
	Remark as well that, conditionally to $\dn$ and $\hp$, $S_{s_n}^*$ is a sum of independent random variables such that $\Ex\br{ \tilde \eta_{n, j} | \dn, \hp } = 0$ since $\Pr\pa{ \delta_{n,j} = i | \dn, \hp } = s_n \check \pi_{n, i}^\alpha$.
	Hence $\{ \tilde \eta_{n,i} : n \in \N, 1 \leq i \leq n \}$ is a martingale difference array for the filtration $(\tFni)_{n \in \N, 0 \leq i \leq n}$.
	Moreover, its conditional variance reads
	\begin{align*}
		\sum_{i = 1}^{n} \E{ \tilde \eta_{n, i}^2 \middle| \tFnim} &= \sum_{i = 1}^{n} \E{ \pa*{ \frac{\mathcal{L}^\prime(Z_i, \hz, f)}{n \sqrt{s_n} \check \pi_{n,i}^\alpha} (\delta_{n,i} - s_n \check\pi_{n, i}^\alpha) }^2 \middle| \tFnim} \\
		&= \frac{1}{n^2} \sum_{i = 1}^{n} \frac{(1 - s_n \check\pi_{n, i}^\alpha)}{\check{\pi}_{n, i}^\alpha} \mathcal{L}^\prime(Z_i, \hz, f)^2 \\
		&= \frac{1}{n^2} \sum_{i = 1}^{n} \frac{\mathcal{L}^\prime(Z_i, \hz, f)^2}{\check\pi_{n, i}^\alpha} - \frac{s_n}{n^2} \sum_{i = 1}^n \mathcal{L}^\prime(Z_i, \hz, f)^2.
	\end{align*}
	By the law of large number and the assumption that $s_n / n \conv c \in [0, 1]$, the second term of the right hand side converges in probability to $c \E{ \mathcal{L}^\prime(Z, \hz, f)^2 }$.
	For the first term, denote
	\begin{equation*}
		\pi_i^\alpha \defeq (1 - \alpha) \frac{\psi_i}{\Ex \psi^\star(Z)} + \alpha,
	\end{equation*}
	and bound the following quantity
	\begin{align*}
		\av*{\frac{1}{n \check\pi_{n, i}^\alpha} - \frac{1}{\pi^\alpha_i \wedge (n/s_n)}} &= \frac{| n \check\pi_{n, i}^\alpha - \pi^\alpha_i \wedge (n/s_n) |}{n \check\pi_{n, i}^\alpha (\pi^\alpha_i \wedge (n/s_n))} \\
		&=\frac{| (n \tilde\pi_{n, i}^\alpha) \wedge (n/s_n) - \pi^\alpha_i \wedge (n/s_n) |}{n \check\pi_{n, i}^\alpha (\pi^\alpha_i \wedge (n/s_n))} \\
		&\leq \frac{| n \tilde\pi_{n,i}^\alpha - \pi^\alpha_i |}{n \check\pi_{n, i}^\alpha (\pi^\alpha_i \wedge (n/s_n))},
	\end{align*}
	Using the inequality $| a \wedge c - b \wedge c | \leq | a - b |$ for $a, b, c \geq 0$.
	However, $\tilde \pi_{n, i}^\alpha \geq \alpha / n$, $\pi_i^\alpha \geq \alpha$ and $\alpha \leq n / s_n$ hence
	\begin{equation*}
		\av*{\frac{1}{n \check\pi_{n, i}^\alpha} - \frac{1}{\pi^\alpha_i \wedge (n/s_n)}} \leq \frac{1}{\alpha^2} | n \tilde\pi_{n,i}^\alpha - \pi^\alpha_i |,
	\end{equation*}
	and
	\begin{equation*}
		\av*{ \frac{1}{n^2} \sum_{i = 1}^{n} \frac{\mathcal{L}^\prime(Z_i, \hz, f)^2}{\check\pi_{n, i}^\alpha} - \frac{1}{n} \sum_{i = 1}^{n} \frac{\mathcal{L}^\prime(Z_i, \hz, f)^2}{\pi_i^\alpha \wedge (n/s_n)} } \leq \frac{1}{\alpha^2} \sum_{i = 1}^{n} \mathcal{L}^\prime(Z_i, \hz, f)^2 | n \tilde\pi_{n,i}^\alpha - \pi^\alpha_i |.
	\end{equation*}
	Furthermore
	\begin{align*}
		| n \tilde\pi_{n,i}^\alpha - \pi^\alpha_i | &= (1 - \alpha) \av*{ \frac{n}{\sum_{j=1}^n \psi_{n,j}} \psi_{n,i} - \frac{\psi_i^\star}{\Ex \psi^\star} } \\
		&\leq (1 - \alpha) \frac{n}{\sum_{j=1}^n \psi_{n,j}} \av{ \psi_{n,i} - \psi_i^\star } + (1 - \alpha) \psi_i^\star \av*{ \frac{n}{\sum_{j=1}^n \psi_{n,j}} - \frac{1}{\Ex \psi^\star} }.
	\end{align*}
	We showed in the proof of \Cref{lem:srm_alpha_pilot_normality_1} that $\frac{1}{n} \sum_{i = 1}^n | \psi_{n,i} - \psi_i^\star | \pconv 0$,
	\begin{equation*}
		\frac{1}{n} \sum_{j = 1}^n \psi_{n,j} \pconv \Ex\psi^\star(Z) \quad \text{and} \quad \frac{1}{n} \sum_{i = 1}^n \mathcal{L}^\prime(Z_i, \hz, f)^2 |\psi_{n,i} - \psi_i^\star | \pconv 0,
	\end{equation*}
	by \Cref{hyp:kernel,hyp:loss,hyp:distribution} and the assumption that $\hp$ is consistent.
	Hence
	\begin{equation*}
		(1 - \alpha) \frac{n}{\sum_{j = 1}^n \psi_{n,j}} \pa*{ \frac{1}{n} \sum_{i = 1}^{n} \mathcal{L}^\prime(Z_i, \hz, f)^2 | \psi_{n,i} - \psi_i^\star | } \pconv 0,
	\end{equation*}
	and
	\begin{equation*}
		(1 - \alpha) \av*{ \frac{n}{\sum_{j = 1}^{n} \psi_{n,j}} - \frac{1}{\Ex \psi^\star(Z)} } \pa*{ \frac{1}{n} \sum_{i = 1}^{n} \mathcal{L}^\prime(Z_i, \hz, f)^2 \psi_i^\star } \pconv 0.
	\end{equation*}
	Therefore
	\begin{equation*}
		\av*{ \frac{1}{n} \sum_{i = 1}^{n} \frac{\mathcal{L}^\prime(Z_i, \hz, f)^2}{n \check\pi_{n, i}^\alpha} - \frac{1}{n} \sum_{i = 1}^{n} \frac{\mathcal{L}^\prime(Z_i, \hz, f)^2}{\pi_i^\alpha \wedge (n/s_n)} } \pconv 0.
	\end{equation*}
	Now define the functions
	\begin{equation*}
		g_n: \bm z \mapsto \frac{\mathcal{L}^\prime(\bm z, \hz, f)^2}{\pi^\alpha(\bm z) \wedge (n/s_n)} \quad \text{and} \quad g : \bm z \mapsto \frac{\mathcal{L}^\prime(\bm z, \hz, f)^2}{\pi^\alpha(\bm z) \wedge c^{-1}},
	\end{equation*}
	for $n > 0$ with $\pi^\alpha(\bm z) \defeq (1 - \alpha) \psi^\star(\bm z) / \E{ \psi^\star(Z) } + \alpha$.
	Note that even though $c$ may be null here, the function $g$ is well defined as $\pi^\alpha(\bm z) \wedge c^{-1}$ would simply be equal to $\pi^\alpha(\bm z)$.
	One has that the sequence of functions $(g_n)_{n > 0}$ is dominated by the integrable function $\bm z \mapsto \alpha^{-1} \mathcal{L}^\prime(\bm z, \hz, f)^2$ and converges pointwise towards $g$.
	Thus
	\begin{align*}
		\E{ \av*{ \frac{1}{n} \sum_{i = 1}^n g_n(Z_i) - \frac{1}{n} \sum_{i = 1}^n g(Z_i) } } &\leq \frac{1}{n} \sum_{i=1}^n \E{\av{ g_n(Z_i) - g(Z_i) }} \\
		&= \E{\av*{ g_n(Z) - g(Z) }} \conv 0,
	\end{align*}
	by the dominated convergence theorem.
	Therefore
	\begin{equation*}
		\av*{ \frac{1}{n} \sum_{i = 1}^{n} \frac{\mathcal{L}^\prime(Z_i, \hz, f)^2}{\pi_i^\alpha \wedge (n/s_n)} - \frac{1}{n} \sum_{i = 1}^{n} \frac{\mathcal{L}^\prime(Z_i, \hz, f)^2}{\pi_i^\alpha \wedge c^{-1}} } \pconv 0.
	\end{equation*}
	and
	\begin{equation*}
		\frac{1}{n^2} \sum_{i = 1}^{n} \frac{\mathcal{L}^\prime(Z_i, \hz, f)^2}{\check\pi_{n,i}^\alpha} \pconv \E{ \frac{\Ex \psi^\star(Z)}{ \Psi^\star_\alpha(Z) \wedge (c^{-1} \Ex \psi^\star(Z))} \mathcal{L}^\prime(Z, \hz, f)^2 }.
	\end{equation*}
	Which finally yields
	\begin{align*}
		\sum_{i = 1}^{n} \E{ \tilde \eta_{n, i}^2 \middle| \tFnim} &= \frac{s_n}{n^2} \sum_{i = 1}^{n} \frac{\mathcal{L}^\prime(Z_i, \hz, f)^2}{s_n \check\pi_{n, i}^\alpha} - \frac{s_n}{n^2} \sum_{i = 1}^n \mathcal{L}^\prime(Z_i, \hz, f)^2 \\
		& \pconv \E{ \frac{\Ex \psi^\star(Z)}{ \Psi^\star_\alpha(Z) \wedge (c^{-1} \E{\psi^\star(Z)})} \mathcal{L}^\prime(Z, \hz, f)^2 } - c \E{ \mathcal{L}^\prime(Z, \hz, f)^2 },
	\end{align*}
	the wanted result.
\end{proof}

\begin{proof}[Proof of \Cref{thm:srm_alpha_pilot_normality_P}]
	Since \Cref{hyp:ssp} is verified, by application of \ref{lem:srm_normality_2} and following the proof of \Cref{thm:srm_normality_R}, one gets
	\begin{align*}
		\sqrt{s_n} \nabla^2\risk(\hz) (\hasn - \hz) + o_\Pr(1) &= - \sqrt{s_n} (\nabla\Lsna(\hz) - \nabla\Ln(\hz)) \\
		&\qquad - \sqrt{\frac{s_n}{n}} \sqrt{n} (\nabla\Ln(\hz) - \nabla\L(\hz)).
	\end{align*}
	Applying \ref{lem:srm_normality_final} with $\Fn = \sigma(\dn, \hp)$ and $H_n = \sqrt{s_n} (\nabla\Lsna(\hz) - \nabla\Ln(\hz))$, together with \Cref{lem:poisson_srm_alpha_pilot_normality_1,lem:srm_normality_1_bis} yields
	\begin{equation*}
		\sqrt{s_n} \nabla^2\risk(\hz) (\hasn - \hz) \dconv \GP(0, \Sigma^*_{\mathrm{P},c}(\alpha) + c \Sigma). \qedhere
	\end{equation*}
\end{proof}


\section{Additional numerical results} \label{sec:additional}

In the first paragraph of \Cref{sec:numerical_results}, we present a study of the impact of the smoothing hyperparameter \(\alpha \in (0, 1]\) on the generalization error.
Since our framework is based on the consistency of \(\hasn\), it is also of interest to analyze \( L^2 \), \( L^\infty \) or \( \H \)-distances between \(\hasn\) and \(\hz\).
These distances are represented (with respect to \(\alpha\)) in \Cref{fig:regression_alpha_distances,fig:classification_alpha_distances}, respectively for the synthetic classification and regression tasks.
Let us remark that, in practice, since \(\hz\) is unreachable, it is approximated by the \erm \(\hn\) computed on the full dataset (\(n = 10^5\)) with a very good Nyström approximation.
The results are similar than those depicted in \Cref{fig:varying_alpha} and show that a value of \(\alpha\) between \(0.1\) and \(0.4\) leads to a better \srm estimator than uniform subsampling.
This means that, in this setting, statistical and learning performances are in line with each other.

\begin{figure}[ht]
	\centering
	\begin{subfigure}{0.48\linewidth}
		\includegraphics[width=\linewidth]{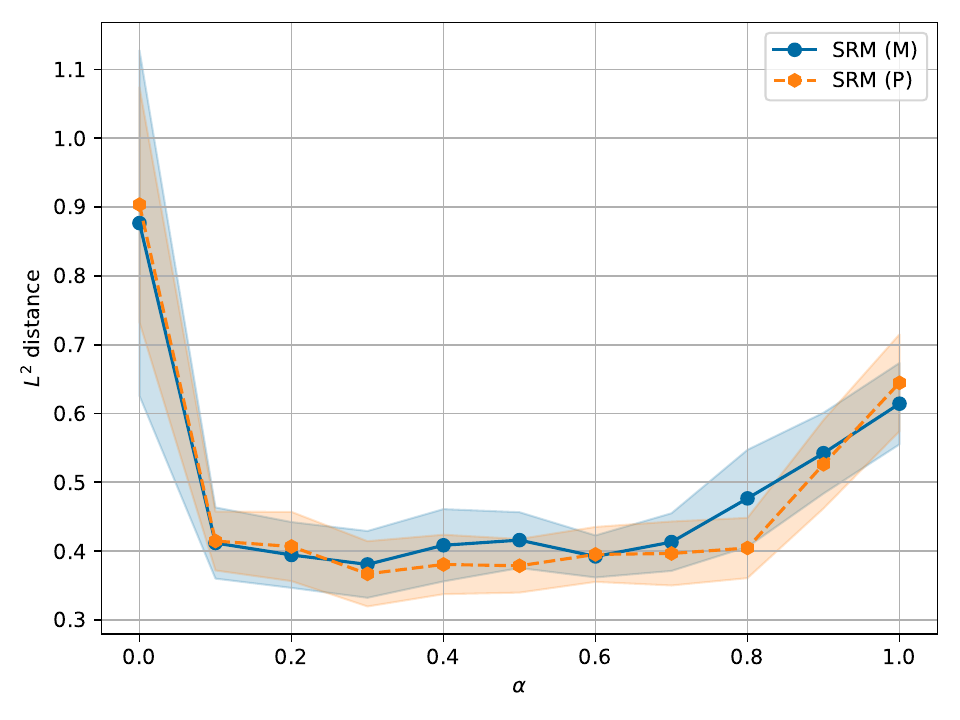}
		\caption{$L^2$-distance.}
		\label{subfig:classification_alpha_L2_distance}
	\end{subfigure}
	\hfill
	\begin{subfigure}{0.48\linewidth}
		\includegraphics[width=\linewidth]{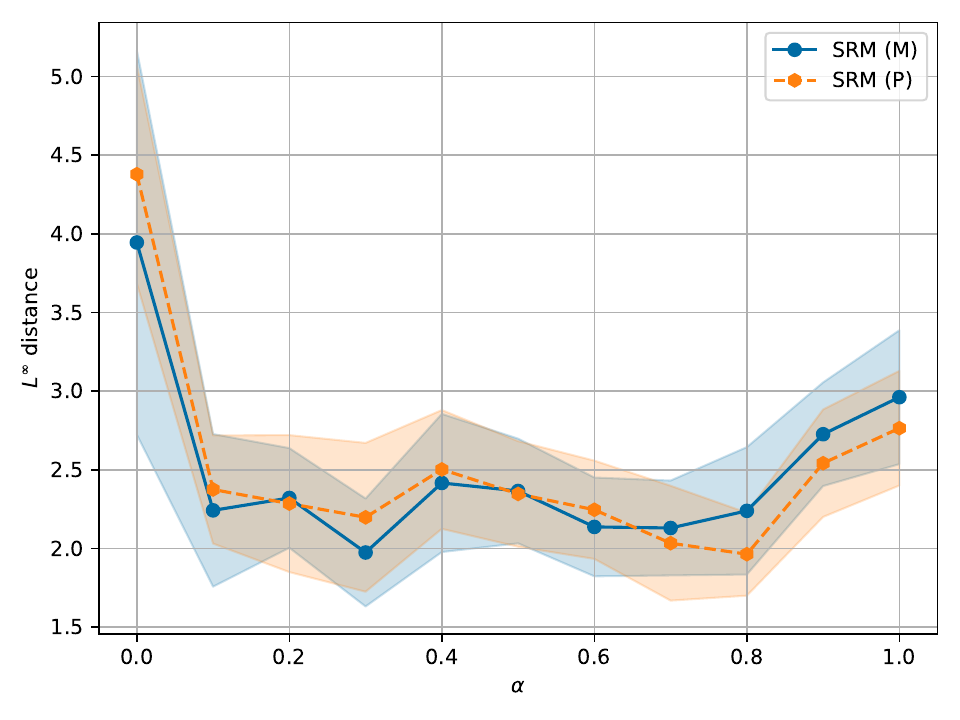}
		\caption{$L^\infty$-distance.}
		\label{subfig:classification_alpha_Linf_distance}
	\end{subfigure} \\
	\begin{subfigure}{0.48\linewidth}
		\includegraphics[width=\linewidth]{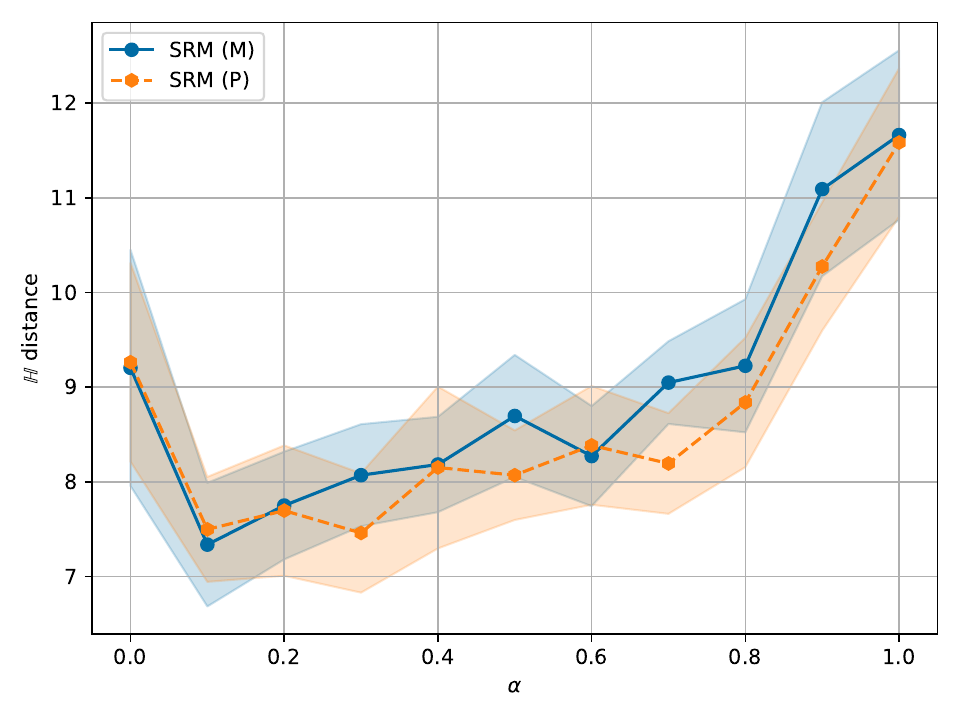}
		\caption{$\H$-distance.}
		\label{subfig:classification_alpha_RKHS_distance}
	\end{subfigure}
	\caption{
		Approximated distances between the SRM \(\hasn\) and the target \(\hz\) for the synthetic classification task (\( n=10^5 \) and \( b_n / n = 0.02\)).}
	\label{fig:classification_alpha_distances}
\end{figure}

\begin{figure}[ht]
	\centering
	\begin{subfigure}{0.48\linewidth}
		\includegraphics[width=\linewidth]{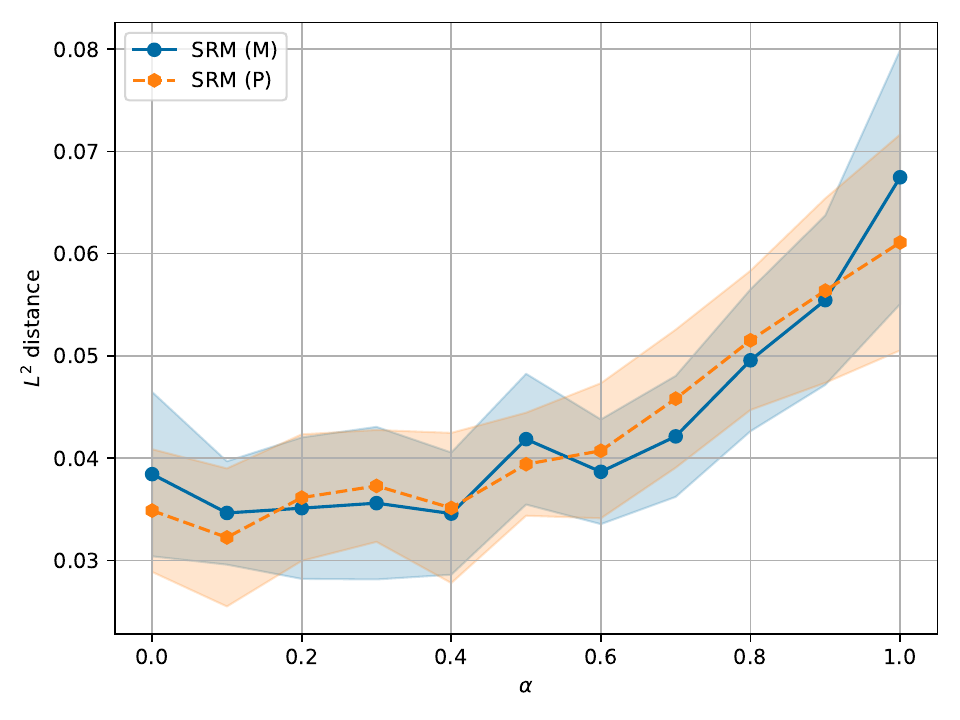}
		\caption{$L^2$-distance.}
		\label{subfig:regression_alpha_L2_distance}
	\end{subfigure}
	\hfill
	\begin{subfigure}{0.48\linewidth}
		\includegraphics[width=\linewidth]{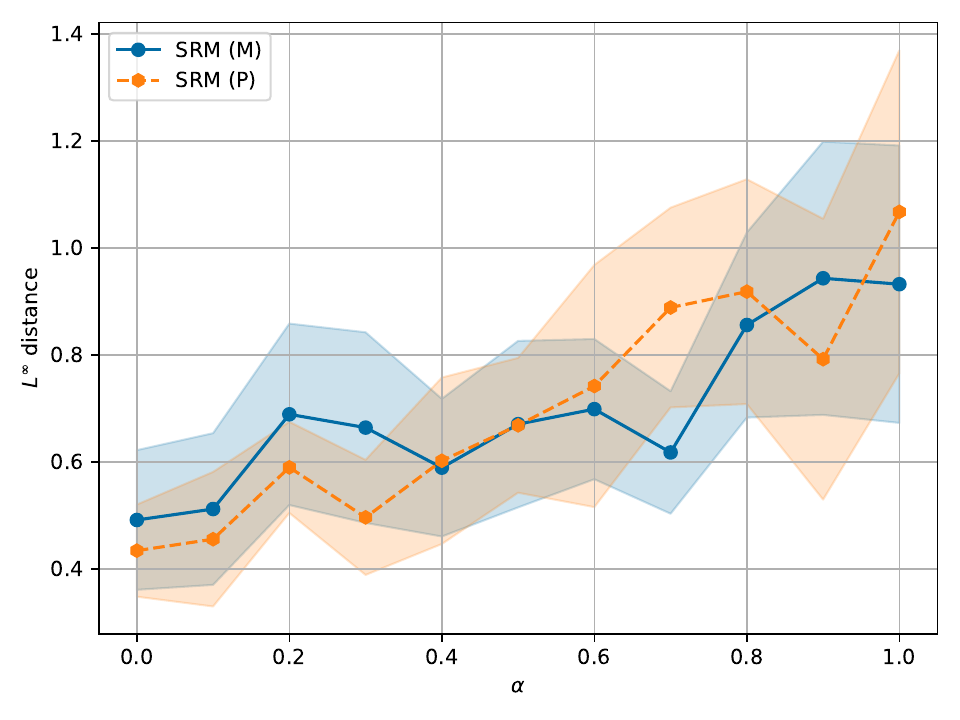}
		\caption{$L^\infty$-distance.}
		\label{subfig:regression_alpha_Linf_distance}
	\end{subfigure} \\
	\begin{subfigure}{0.48\linewidth}
		\includegraphics[width=\linewidth]{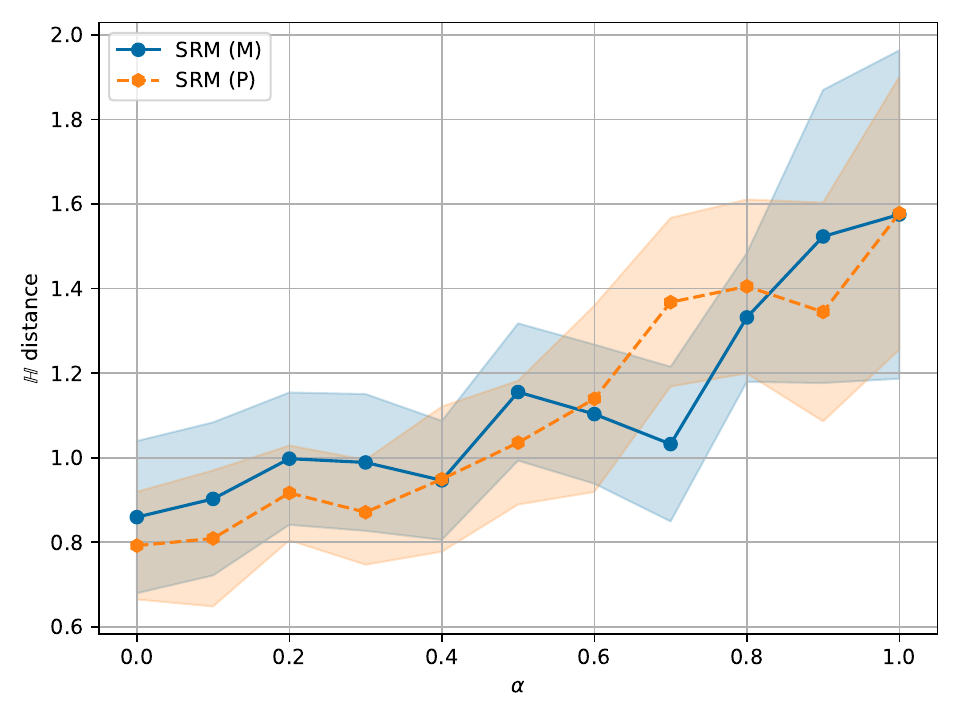}
		\caption{$\H$-distance.}
		\label{subfig:regression_alpha_RKHS_distance}
	\end{subfigure}
	\caption{
		Approximated distances between the SRM \(\hasn\) and the target \(\hz\) for the synthetic regression task (\( n=10^5 \) and \( b_n / n = 0.01\)).}
	\label{fig:regression_alpha_distances}
\end{figure}


\end{document}